\documentclass[review]{elsarticle}
\usepackage{amsmath,amsfonts}
\usepackage{lineno,hyperref}
\modulolinenumbers[5]
\usepackage{array}
\usepackage{textcomp}
\usepackage{stfloats}
\usepackage{url}
\usepackage{verbatim}
\biboptions{sort&compress}
\usepackage{amsthm}
\usepackage{color}
\usepackage{multirow}
\usepackage[para,online,flushleft]{threeparttable}
\usepackage{booktabs}
\usepackage{xcolor}
\newtheorem{theorem}{Theorem}
\newtheorem{definition}{Definition}
\newtheorem{lemma}{Lemma}

\makeatletter
\def\ps@pprintTitle{%
   \let\@oddhead\@empty
   \let\@evenhead\@empty
   \let\@oddfoot\@empty
   \let\@evenfoot\@oddfoot}
\makeatother
\begin{document}

\begin{frontmatter}

\title{\large Deeper Insights into Deep Graph Convolutional
Networks: Stability and Generalization}

\author[add1]{Guangrui Yang}
\ead{yanggrui@mail2.sysu.edu.cn}
\author[add2]{Ming Li\corref{corresponding}}
\ead{mingli@zjnu.edu.cn}
\author[add3]{Han Feng\corref{corresponding}}
\ead{hanfeng@cityu.edu.hk}
\author[add3]{Xiaosheng Zhuang}
\ead{xzhuang7@cityu.edu.hk}

\cortext[corresponding]{Corresponding authors.}
\address[add1]{Department of Mathematics, College of Mathematics and Informatics, South China Agricultural University, Guangzhou, China.}
\address[add2]{Zhejiang Key Laboratory of Intelligent Education Technology and Application,
Zhejiang Normal University, Jinhua, China}
\address[add3]{Department of Mathematics, City University of Hong Kong, Hong Kong, China}

\begin{abstract}
Graph convolutional networks (GCNs) have emerged as powerful models for graph learning tasks, exhibiting promising performance in various domains. While their empirical success is evident, there is a growing need to understand their essential ability from a theoretical perspective. Existing theoretical research has primarily focused on the analysis of single-layer GCNs, while a comprehensive theoretical exploration of the stability and generalization of deep GCNs remains limited. In this paper, we bridge this gap by delving into the stability and generalization properties of deep GCNs, aiming to provide valuable insights by characterizing rigorously the associated upper bounds. Our theoretical results reveal that the stability and generalization of deep GCNs are influenced by certain key factors, such as the maximum absolute eigenvalue of the graph filter operators and the depth of the network. Our theoretical studies contribute to a deeper understanding of the stability and generalization properties of deep GCNs, potentially paving the way for developing more reliable and well-performing models.

\end{abstract}

\end{frontmatter}

\section{Introduction}
Graph-structured data is pervasive across diverse domains, including knowledge graphs, traffic networks, and social networks to name a few \cite{LIANG2022103714,ma2021deep}.
Several pioneering works \cite{gori2005new,scarselli2008graph} introduced the initial concept of graph neural networks (GNNs), incorporating recurrent mechanisms and necessitating neural network parameters to define contraction mappings. Concurrently, Micheli \cite{micheli2009neural} introduced the neural network for graphs, commonly referred to as NN4G, over a comparable timeframe. It is worth noting that the NN4G diverges from recurrent mechanisms and instead employs a feed-forward architecture, exhibiting similarities to contemporary GNNs. In recent years, (contemporary) GNNs have gained significant attention as an effective methodology for modeling graph data \cite{yao2022multi,hamilton2020graph, GNNBook2022,bianchi2021graph,jiang2023graph,zhang2024decouple}. To obtain a comprehensive understanding of GNNs and deep learning for graphs, we refer the readers to relevant survey papers for an extensive overview \cite{bacciu2020gentle,Survey_SunMS, Survey_ZhangCQ, Survey_ZhuWW}.

Among the various GNN variants, one of the most powerful and frequently used GNNs is graph convolutional networks (GCNs).
A widely accepted perspective posits that GCNs can be regarded as an extension or generalization of traditional spatial filters, which are commonly employed in Euclidean data analysis, to the realm of non-Euclidean data. Due to its success on non-Euclidean data, GCN has attracted widespread attention on its theoretical exploration.  Recent works on GCNs includes understanding over-smoothing \cite{2018Oversmoothing,2020Oversmoothing,2019GCN_expressive_power,2020DropEdge}, interpretability and explainability\cite{2020interpretability,2021explainability,yuan2022explainability,schnake2021higher,bouritsas2022improving}, expressiveness \cite{2018how_powerful,2019Approximation_GCNs,2019understanding}, and generalization \cite{2019Du,2020fastlearning_generalization,2018VC-dim,2020Rademacher,2020RademacherOptimization,2021Rademacher,2021Rademacher_learning,2023Generalization_GNNs,2019Stable_GCN,
 2023Stable_GCN,2021PAC-generalization,2023PAC-generalization,2018_Neural-Tangent-Kernel,2019_Neural-Tangent-Kernel}. 
 In this paper, we specifically address the generalization of GCNs to provide a bound on their generalization gap.

Investigating the generalization of GCNs is essential in understanding its underlying working principles and capabilities from a theoretical perspective. However, the theoretical establishment in this area is still in its infancy. In recent work \cite{2019Stable_GCN}, Verma and Zhang provided a novel technique based on algorithmic stability to investigate the generalization capability of single-layer GCNs in semi-supervised learning tasks.
 Their results indicate that the stability of a single-layer GCN trained with the stochastic gradient descent (SGD) algorithm is dependent on the largest absolute eigenvalue of graph filter operators. This finding highlights the crucial role of graph filters in determining the generalization capability of single-layer GCNs, providing guidance for designing effective graph filters for these networks. On the other hand, a number of prior studies have shown that deep GCNs possess greater expressive power than their single-layer counterparts. Consequently, it is essential to extend the generalization results of single-layer GCNs to their multi-layer counterparts. This will help us understand the effect of factors (e.g., graph filters, number of layers) on the generalization capability of deep GCNs.

 In this paper, we investigate the generalization properties of deep GCNs. Building on the stability framework of \cite{2019Stable_GCN}, we analyze the uniform stability of deep GCNs in semi-supervised learning, while developing a more refined theoretical treatment suited to deep architectures. Our analysis reveals a strong connection between the generalization gap of deep GCNs and the characteristics of the graph filter, particularly the number of layers. In particular, we show that when the maximum absolute eigenvalue (or the largest singular value) of the graph filter operator remains invariant with respect to graph size, the generalization gap diminishes asymptotically at a rate of $O(1/\sqrt{m})$ as the training sample size $m$ grows. This result explains why normalized graph filters generally outperform non-normalized ones in deep GCNs. Furthermore, our findings indicate that increasing depth can enlarge the generalization gap and consequently degrade performance, thereby offering theoretical guidance for selecting an appropriate number of layers when designing deep GCNs. We then empirically validate our theoretical results through experiments on three benchmark datasets: Cora, Citeseer, and Pubmed, demonstrating strong consistency between theory and practice. \textcolor{black}{In addition, we further discuss how our theoretical framework extends to advanced architectures, including GCNII \cite{chen2020simple} and Graph Transformer \cite{GT}, thereby highlighting its broader applicability and its potential to inspire future theoretical studies on more complex GNN variants.}

The key contributions of our paper are as follows:
\begin{itemize}
    \item We establish the uniform stability of deep GCNs trained with SGD, thereby extending the earlier results on single-layer GCNs presented in \cite{2019Stable_GCN}.
    \item We provide a rigorous upper bound for the generalization gap of deep GCNs and highlight the key factors that govern their generalization ability. \textcolor{black}{Moreover, we further discuss how our theoretical framework extends naturally to advanced GNN architectures, including GCNII and Graph Transformer models.}
    \item We conduct empirical studies on three benchmark datasets for node classification, which strongly validate our theoretical findings regarding the influence of graph filters, as well as the depth and width of deep GCNs.
\end{itemize}


The remainder of this paper is organized as follows. In Section \ref{sec:Related_work}, an overview of prior studies on the generalization of GCNs (or generic GNNs) is presented, along with a comparative analysis highlighting the similarities and distinctions between our work and previous research. Section \ref{sec:Preliminaries} offers an exposition of the essential concepts. The primary findings of this paper are given in Section \ref{sec:Main_result}. Experimental studies designed to validate our theoretical findings are presented in Section \ref{Sec:Exp}.  \textcolor{black}{In Section~\ref{sec:Implications}, we discuss how our findings extend to advanced GNN architectures, including GCNII and Graph Transformer models.} Section \ref{sec:Conclusion} concludes the paper with additional remarks. The detailed proofs of our theoretical results are deferred to the \textbf{Appendix} section.

\section{Related Work}\label{sec:Related_work}
Theoretical studies on the generalization capability of GCNs mainly employ three methodologies: Vapnik–Chervonenkis (VC) dimension \cite{2018VC-dim,2021Rademacher_learning}, Rademacher complexity \cite{2020Rademacher,2020RademacherOptimization,2021Rademacher,2021Rademacher_learning,2023Generalization_GNNs}, and algorithmic stability \cite{2019Stable_GCN,2023Stable_GCN,2023Stable_Lp,Litpami2023}. Other approaches include PAC-Bayesian theory \cite{2021PAC-generalization,2023PAC-generalization}, neural tangent kernels (NTKs) \cite{2018_Neural-Tangent-Kernel,2019_Neural-Tangent-Kernel}, algorithm alignment \cite{xu2019can,xu2020neural}, and methods from statistical physics and random matrix theory \cite{shi2024homophily}. \textcolor{black}{For a broader perspective, we refer readers to the recent survey \cite{vasileiou2025survey}, which provides a comprehensive overview of generalization theory for message-passing GNNs.}

\textbf{VC-Dimension and Rademacher Complexity.} Scarselli et al. \cite{2018VC-dim} study the generalization capability of GNNs by deriving upper bounds on the growth order of their VC-dimension. While VC-dimension is a classical tool for establishing learning bounds, it does not capture the structure of the underlying graph. Similarly, \cite{2021Rademacher_learning} provides VC-dimension–based error bounds for GNNs, but the results are trivial and fail to reflect the benefits of degree normalization. To address graph-specific effects, Esser et al. \cite{2021Rademacher_learning} analyze upper bounds using transductive Rademacher complexity (TRC), highlighting how graph convolutions and network architectures influence generalization. Tang et al. \cite{2023Generalization_GNNs} establish high-probability generalization bounds for popular GNNs via TRC-based analysis of transductive SGD. However, their bounds scale with the parameter dimension, limiting tightness for large models.

\textbf{Algorithmic Stability.} Beyond capacity-based measures, algorithmic stability serves as an important framework for understanding GNN generalization. Building on the work of Hardt et al. \cite{hardt2016train}, Verma and Zhang \cite{2019Stable_GCN} show that one-layer GCNs exhibit uniform stability and provide generalization bounds that scale with the largest absolute eigenvalue of the graph filter operator. Extending this line, Liu et al. \cite{2023Stable_Lp} analyze the stability of single-layer GCNs trained with an SGD-proximal algorithm under $\ell_{p}$-regularization, yielding a more refined theoretical understanding. These studies, however, remain restricted to single-layer architectures. Cong et al. \cite{cong2021provable} examine GNNs under uniform transductive stability, showing that deeper models improve stability and reduce generalization error, whereas our work adopts a different stability formulation. Ng and Yip \cite{2023Stable_GCN} investigate stability and generalization in two-layer GCNs under an eigen-domain formulation, relying on spectral graph convolution \cite{bruna2013spectral}. Because this formulation requires computationally expensive eigendecomposition of the graph Laplacian, it does not scale to large node-classification tasks. Within this methodological line, the closest studies to ours are \cite{2019Stable_GCN} and \cite{2023Stable_GCN}, but our analysis focuses on deep GCNs without assuming a spectral-based formulation.

\textbf{Other Methodologies.} Alternative perspectives on GNN generalization also exist. The pioneering work of \cite{2021PAC-generalization} introduces PAC-Bayesian analysis for GCNs and message-passing neural networks, later extended in \cite{2023PAC-generalization} to provide tighter bounds linked to the graph diffusion matrix. The NTK framework introduced by \cite{2018_Neural-Tangent-Kernel} enables analysis of infinitely wide GNNs trained by gradient descent, with \cite{2019_Neural-Tangent-Kernel} extending this framework to multi-layer settings. However, NTK-based analyses typically focus on graph classification rather than the more challenging transductive node-classification setting. Additional work explores distinct theoretical frameworks, including topology-sampling techniques \cite{li2022generalization}, analysis on large random graphs \cite{keriven2020convergence}, and NTK-based loss landscape analysis of wide GCNs \cite{zhou2025tighter}. For further perspectives, we refer readers to the survey \cite{jegelka2022theory}, which synthesizes emerging theoretical approaches to characterizing GNN capabilities.

\section{Preliminaries and Notations}\label{sec:Preliminaries}
In this section, we describe the problem setup considered in this paper and review fundamental concepts of uniform stability for training algorithms, which form the basis of our subsequent analysis. \textcolor{black}{For clarity, we first summarize the main symbols used in this paper in the table below.}
\begin{table}[htbp!]
    \footnotesize
    \centering    \renewcommand{\arraystretch}{1.0}
	\setlength{\tabcolsep}{1.0pt}
    \textcolor{black}{\caption{Frequently used notations.}\label{Tab:Notations}}
    \setlength{\tabcolsep}{0.3em} 
    \renewcommand{\arraystretch}{1.2}
    \renewcommand{\aboverulesep}{0pt}
    \renewcommand{\belowrulesep}{0pt}
\begin{tabular}{@{}cl@{}}
\toprule
\textcolor{black}{\textbf{Notation}} & \multicolumn{1}{c}{\textcolor{black}{\textbf{Description}}} \\ \midrule
  $\textcolor{black}{g({\mathbf{L}})}$ & \textcolor{black}{graph filter operator used in the considered deep GCNs}\\
\textcolor{black}{$C_{g}$} & \textcolor{black}{the 2-norm of $g(\mathbf{L})$, i.e., $C_{g}:=\|g(\mathbf{L})\|_{2}$} \\
\textcolor{black}{$C_{\mathbf{X}}$} & \textcolor{black}{Frobenius norm of the input feature $\mathbf{\mathbf{X}}$, i.e., $C_{\mathbf{X}}:=\|\mathbf{X}\|_{F}$} \\
\textcolor{black}{$K$} & \textcolor{black}{number of hidden layers of the considered deep GCNs} \\
\textcolor{black}{$\alpha_{\sigma}$, $\upsilon_{\sigma}$} & \textcolor{black}{parameters w.r.t the continuity of activation function $\sigma(\cdot)$} \\
\textcolor{black}{$\nabla\sigma$} & \textcolor{black}{the derivative of activation function $\sigma(\cdot)$} \\
\textcolor{black}{$\alpha_{\ell}$, $\upsilon_{\ell}$} & \textcolor{black}{parameters w.r.t the continuity of the loss function $\ell(\cdot,\cdot)$} \\
\textcolor{black}{$M$} & \textcolor{black}{the upper bound of loss function $\ell(\cdot,\cdot)$} \\
\textcolor{black}{$\mathcal{A}_{\mathcal{S}}$} & \textcolor{black}{the learning algorithm for deep GCNs trained on dataset $\mathcal{S}$} \\
\textcolor{black}{$m$} & \textcolor{black}{the number of samples  in the trained dataset $\mathcal{S}$} \\
\textcolor{black}{$\eta$} & \textcolor{black}{the learning rate of $\mathcal{A}_{\mathcal{S}}$} \\
\textcolor{black}{$T$} & \textcolor{black}{number of iterations for training $\mathcal{A}_{\mathcal{S}}$ using SGD } \\
\textcolor{black}{$\mu_{m}$} & \textcolor{black}{the uniform stability of a learning algorithm $\mathcal{A}_{\mathcal{S}}$} \\
\textcolor{black}{$\boldsymbol{\delta}_{\mathbf{x}}$} & \textcolor{black}{the indicator vector with respect to node $\mathbf{x}$} \\
\textcolor{black}{$\boldsymbol{\delta}_{i}$} & \textcolor{black}{the indicator vector with respect to index $i$} \\
\textcolor{black}{$\mathbf{X}^{(k)}$} & \textcolor{black}{the output feature matrix of the $k$-th layer} \\
\textcolor{black}{$\triangle\mathbf{X}^{(k)}$} & \textcolor{black}{the variation of $\mathbf{X}^{(k)}$ in two GCNs} \\
\textcolor{black}{$\mathbf{W}^{(k)}$} & \textcolor{black}{the parameter matrix specific to the $k$-th layer} \\
\textcolor{black}{$B$} & \textcolor{black}{upper bound for 2-norm of $\{\mathbf{W}^{(1)},\dots,\mathbf{W}^{(K)},\mathbf{w}\}$} \\
\textcolor{black}{$\triangle\mathbf{W}^{(k)}$} & \textcolor{black}{the variation of $\mathbf{W}^{(k)}$ in two GCNs} \\
\textcolor{black}{$\triangle\theta$} & \textcolor{black}{ $\triangle\theta=\{\triangle\mathbf{W}^{(1)},\dots,\triangle\mathbf{W}^{(K)},\triangle\mathbf{w}\}$} \\
\textcolor{black}{$\mathbf{W}^{(k)}_{t}$} & \textcolor{black}{the learnt $\mathbf{W}^{(k)}$ trained after $t$ iterations}\\
\textcolor{black}{$\triangle\mathbf{W}^{(k)}_{t}$} & \textcolor{black}{the variation of $\mathbf{W}^{(k)}_{t}$ of two GCNs trained after $t$ iterations} \\
\textcolor{black}{$\triangle\theta_{t}$} & \textcolor{black}{ $\triangle\theta_{t}=\{\triangle\mathbf{W}_{t}^{(1)},\dots,\triangle\mathbf{W}_{t}^{(K)},\triangle\mathbf{w}_{t}\}$} \\
    \bottomrule
\end{tabular}
\end{table}

\subsection{Deep Graph Convolutional Networks}

Let $\mathcal{G}=(\mathcal{V},\mathcal{E},\mathbf{A})$ denote an undirected graph with a node set $\mathcal{V}$ of size $N$,  an edge set $\mathcal{E}$ and the adjacency matrix $\mathbf{A}\in\mathbb{R}^{N\times N}$. As usual, $\mathbf{L}:=\mathbf{D}-\mathbf{A}$ is denoted as its conventional graph Laplacian, where $\mathbf{D}\in\mathbb{R}^{N\times N}$ signifies the degree diagonal matrix. Furthermore, $g(\mathbf{L})\in\mathbb{R}^{N\times N}$ represents a graph filter  and is defined as a function of $\mathbf{L}$ (or its normalized versions). We denote by $C_{g}=\|g(\mathbf{L})\|_{2}$ the maximum absolute eigenvalue of a symmetric filter $g(\mathbf{L})$ or the maximum singular value of an asymmetric $g(\mathbf{L})$.

We denote by $\mathbf{X}=(\mathbf{x}_{1},\mathbf{x}_{2},\dots,\mathbf{x}_{N})^{\top}\in\mathbb{R}^{N\times d_{0}}$ the input features ($d_{0}$ stands for input dimension) and $\mathbf{x}_{j}\in\mathbb{R}^{d_{0}}$ the node feature of node $j$, while $C_{\mathbf{X}}=\|\mathbf{X}\|_{F}$ represents the Frobenius norm of $\mathbf{X}$. For the input feature $\mathbf{X}$, a deep GCN with $g(\mathbf{L})$ updates the representation  as follows:
\begin{align*}
\mathbf{X}^{(k)}=\sigma(g(\mathbf{L})\mathbf{X}^{(k-1)}\mathbf{W}^{(k)}),~~k=1,2,\dots,K,
\end{align*}
where $\mathbf{X}^{(k)}\in\mathbb{R}^{N\times d_{k}}$ is the output feature matrix of the $k$-th layer with $\mathbf{X}^{(0)}=\mathbf{X}$, the matrix $\mathbf{W}^{(k)}\in\mathbb{R}^{d_{k-1}\times d_{k}}$ represents the trained parameter matrix specific to the $k$-th layer. The function $\sigma(\cdot)$ denotes a nonlinear activation function applied within the GCN model. For simplicity, we set a final output in a single dimension, that is, the final output label of $N$ nodes is given by
\begin{equation}\label{equ:GCN_model}
\mathbf{y}=\sigma\big(g(\mathbf{L})\mathbf{X}^{(K)}\mathbf{w}\big), \\
\end{equation}
where $\mathbf{y}\in\mathbb{R}^{N}$ and $\mathbf{w}\in\mathbb{R}^{d_{K}}$.

As defined above, the deep GCN \eqref{equ:GCN_model} with learnable parameters
$$\theta=\{\mathbf{W}^{(1)},\mathbf{W}^{(2)},\dots,\mathbf{W}^{(K)},\mathbf{w}\}$$
is a $K+1$ layers GCN with $K$ hidden layers and a final output layer, and in the case of $K=0$, it degenerates into the single-layer GCN studied in \cite{2019Stable_GCN}.
\subsection{The SGD Algorithm}\label{sec:3.2}

We denote by $\mathcal{D}$ the unknown joint distribution of input features and output labels. 
Let $$\mathcal{S}:=\big\{(\mathbf{x}_{j},y_{j})\big\}_{j=1}^{m}$$
be the training set i.i.d sampled from $\mathcal{D}$
and $\mathcal{A}_{\mathcal{S}}$ be a learning algorithm for a deep GCN trained on $\mathcal{S}$.
 For a deep GCN model \eqref{equ:GCN_model} with parameters $\theta=\{\mathbf{W}^{(1)},\dots,\mathbf{W}^{(K)},\mathbf{w}\}$, denote $\mathcal{A}_{\mathcal{S}}(\mathbf{x})=f(\mathbf{x}|\theta_S)=\sigma\big(\boldsymbol{\delta}_{\mathbf{x}}^{\top}g(\mathbf{L})\mathbf{X}^{(K)}\mathbf{w}\big)$ as the output of node $\mathbf{x}$, where $\theta_S$ is the corresponding learned parameter and $\boldsymbol{\delta}_{\mathbf{x}}$ is the indicator vector with respect to node $\mathbf{x}$.
 For a loss function $\ell: \mathbb{R}\times \mathbb{R}\to \mathbb{R}_+$,
 the generalization error or risk $R(\mathcal{A}_{\mathcal{S}})$ is defined by
$$R(\mathcal{A}_{\mathcal{S}}):=\mathbb{E}_{\mathbf{z}}\Big[\ell(f(\mathbf{x}|\theta_S),y)\Big],$$
where the expectation is taken over $\mathbf{z}=(\mathbf{x},y)\sim\mathcal{D}$, and the empirical error or risk $R_{emp}(\mathcal{A}_{\mathcal{S}})$ is
$$R_{emp}(\mathcal{A}_{\mathcal{S}}):=\frac{1}{m}\sum_{j=1}^{m}\ell(f(\mathbf{x}_{j}|\theta_S),y_{j}).$$
When considering a randomized algorithm $\mathcal{A}_{\mathcal{S}}$,
\begin{equation}\label{equ:G_gap}
\epsilon_{gen}(\mathcal{A}_{\mathcal{S}}):=\mathbb{E}_{\mathcal{A}}\Big[R(\mathcal{A}_{\mathcal{S}})-R_{emp}(\mathcal{A}_{\mathcal{S}})\Big]
\end{equation}
gives the generalization gap between the generalization error and the empirical error,
where the expectation $\mathbb{E}_{\mathcal{A}}$ corresponds to the inherent randomness of $\mathcal{A}_{\mathcal{S}}$.

In this paper, $\mathcal{A}_{\mathcal{S}}$ is  considered to be the algorithm given by the SGD algorithm. Following the approach employed in \cite{2019Stable_GCN}, our analysis focuses solely on the randomness inherent in $\mathcal{A}_{\mathcal{S}}$ arising from the SGD algorithm, while disregarding the stochasticity introduced by parameter initialization. The SGD algorithm for a deep GCN \eqref{equ:GCN_model}  aims to optimize its empirical error on a dataset $\mathcal{S}$ by updating parameters iteratively. For $t\in\mathbb{N}$ and considering the parameters $\theta_{t-1}$ obtained after $t-1$ iterations, the $t$-th iteration of SGD involves randomly drawing a sample $(\mathbf{x}_{t},y_{t})$ from the dataset $\mathcal{S}$. Subsequently, parameters $\theta$ are iteratively updated as follows:
\begin{equation}\label{equ:update_W_and_w}
\theta_{t}=\theta_{t-1}-\eta\nabla_{\theta}\ell
(f(\mathbf{x}_{t}|\theta_{t-1}),y_{t}),
\end{equation}
with the learning rate $\eta>0$.

\subsection{Uniform Stability}
For the sake of estimating the generalization gap $\epsilon_{gen}(\mathcal{A}_{\mathcal{S}})$ of $\mathcal{A}_{\mathcal{S}}$, we invoke the notion of uniform stability of $\mathcal{A}_{\mathcal{S}}$ as adopted in \cite{2005stability_random,2019Stable_GCN}.

Let
$$S^{\setminus i}=\big\{(\mathbf{x}_{j},y_{j})\big\}_{j=1}^{i-1}\cup\big\{(\mathbf{x}_{j},y_{j})\big\}_{j=i+1}^{m}$$
be the dataset obtained by removing the $i$-th data point in $\mathcal{S}$,
and
$$S^{i}=\big\{(\mathbf{x}_{j},y_{j})\big\}_{j=1}^{i-1}\cup
\{(\mathbf{x}'_{i},y'_{i})\}\cup\big\{(\mathbf{x}_{j},y_{j})\big\}_{j=i+1}^{m}$$
 the dataset obtained by replacing the $i$-th data point in $\mathcal{S}$. Then, the formal definition of uniform stability of a randomized algorithm $\mathcal{A}_{\mathcal{S}}$ is given in the following.
\begin{definition}
[Uniform Stability \cite{2019Stable_GCN}] \label{def:uni_stability}
A randomized algorithm $\mathcal{A}_{\mathcal{S}}=f(\mathbf{x}|\theta_S)$ is considered to be $\mu_{m}$-uniformly stable in relation to a loss function $\ell$ when it fulfills the following condition:
\begin{equation}\label{equ:Uni_stability}
  \sup_{\mathcal{S},\mathbf{z}}\Big|\mathbb{E}_{\mathcal{A}}\big[\ell(\hat{y}, y)\big]-\mathbb{E}_{\mathcal{A}}\big[\ell(\hat{y}',y)\big]\Big|\leq\mu_{m},
  \end{equation}
  where $\mathbf{z}=(\mathbf{x},y)\sim \mathcal{D}$, $\hat{y}=f(\mathbf{x}|\theta_S)$ and $\hat{y}'=f(\mathbf{x}|\theta_{\mathcal{S}^{\setminus i}})$.
\end{definition}
As shown in Definition \ref{def:uni_stability}, $\mu_{m}$ indicates a bound on how much the variation of the training set $\mathcal{S}$ can influence the output of $\mathcal{A}_{\mathcal{S}}$. It further implies the following property:
\begin{equation}\label{equ:Uni_stability1}
\sup_{\mathcal{S},\mathbf{z}}\Big|\mathbb{E}_{\mathcal{A}}\big[\ell(\hat{y}, y)\big]-\mathbb{E}_{\mathcal{A}}\big[\ell(\hat{y}',y)\big]\Big|\leq2\mu_{m},
  \end{equation}
  where $\mathbf{z}=(\mathbf{x},y)\sim \mathcal{D}$, $\hat{y}=f(\mathbf{x}|\theta_S)$ and $\hat{y}'=f(\mathbf{x}|\theta_{\mathcal{S}^{i}})$.

Moreover, it is shown that the uniform stability of a learning algorithm $\mathcal{A}_{\mathcal{S}}$ can yield the following upper bound on the generalization gap $\epsilon_{gen}(\mathcal{A}_{\mathcal{S}})$.
\begin{lemma}
[Stability Guarantees \cite{2019Stable_GCN}]\label{lem:Sta_Guarantees}
Suppose that a randomized algorithm $\mathcal{A}_{\mathcal{S}}$ is $\mu_{m}$-uniformly stable with a bounded loss function $\ell$. Then, with a probability of at least $1-\delta$, considering the random draw of $\mathcal{S},\mathbf{z}$ with $\delta\in(0,1)$, the following inequality holds for the expected value of the generalization gap:
$$ 
\epsilon_{gen}(\mathcal{A}_{\mathcal{S}})
     \leq2\mu_{m}+ \bigg(4m\mu_{m}+M\bigg)\sqrt{\frac{\log\frac{1}{\delta}}{2m}},
$$
\textcolor{black}{where $M$ is an upper bound of the loss function $\ell$, i.e., $0\leq\ell(\cdot,\cdot)\leq M$.}
\end{lemma}



\section{Main Results} \label{sec:Main_result}
This section presents  
an established upper bound on the generalization gap $\epsilon_{gen}(\mathcal{A}_{\mathcal{S}})$ as defined in \eqref{equ:G_gap} for deep GCNs trained using the SGD algorithm.
 Notably, this generalization bound, derived from a meticulous analysis of the comprehensive back-propagation algorithm, demonstrates the enhanced insight gained through the utilization of SGD.

\subsection{Assumptions} \label{Sec:Assumption}
First, we make some assumptions about the considered deep GCN model \eqref{equ:GCN_model}, which are necessary to derive our results.

\textbf{Assumption 1.} The activation function $\sigma:\mathbb{R}\rightarrow\mathbb{R}$ is assumed to satisfy the following:
  \begin{enumerate}
    \item $\alpha_{\sigma}$-Lipschitz:
    $$|\sigma(x)-\sigma(y)|\leq\alpha_{\sigma}|x-y|,~~\forall~x,y\in\mathbb{R}.$$
    \item $\nu_{\sigma}$-smooth:
    $$|\nabla\sigma(x)-\nabla\sigma(y)|\leq\nu_{\sigma}|x-y|,~~\forall~x,y\in\mathbb{R}.$$
    \item $\sigma(0)=0$.
  \end{enumerate}
   With these assumptions, the derivative of $\sigma$, denoted by $\nabla\sigma$, is bounded, i.e.,  $|\nabla\sigma(\cdot)|\leq\alpha_{\sigma}$,
   and $\|\sigma(\mathbf{X})\|_{F}\leq\alpha_{\sigma}\|\mathbf{X}\|_{F}$ holds for any matrix $\mathbf{X}$. It can be easily verified that activation functions such as \textrm{ELU} and \textrm{tanh} satisfy the above assumptions.

  \textbf{Assumption 2.} Let $\hat{y}$ and $y$ be the predicted and true labels, respectively. We denote the loss function $\ell: [y_{\min},y_{\max}]\times[y_{\min},y_{\max}]\rightarrow\mathbb{R}$ by $\ell(\hat{y},y)$. Similar to \cite{2023Stable_GCN}, we adopt
  the following assumptions for $\ell$.
\begin{enumerate}
  \item The loss function $\ell$ exhibits continuity with respect to the variables $(\hat{y},y)$
 and possesses continuous differentiability with respect to  $\hat{y}$.
  \item The loss function $\ell$ satisfies $\alpha_{\ell}$-Lipschitz with respect to $\hat{y}$:
  $$|\ell(\hat{y},y)-\ell(\hat{y}',y)|\leq\alpha_{\ell}|\hat{y}-\hat{y}'|,~~\forall~\hat{y},\hat{y}',y\in[y_{\min},y_{\max}].$$
  \item The loss function $\ell$ meets $\nu_{\ell}$-smooth with respect to $\hat{y}$:
  $$\bigg|\frac{\partial\ell}{\partial\hat{y}}(\hat{y},y)-\frac{\partial\ell}{\partial\hat{y}}(\hat{y}',y)\bigg|
  \leq\nu_{\ell}|\hat{y}-\hat{y}'|,~~\forall~\hat{y},\hat{y}',y\in[y_{\min},y_{\max}].$$
\end{enumerate}
 With these assumptions, $|\frac{\partial\ell}{\partial\hat{y}}(\hat{y},y)|\leq\alpha_{\ell}$, and $\ell$ is bounded, i.e., $0\leq\ell(\hat{y},y)\leq M$.

\textbf{Assumption 3.} The learned parameters  $\{\mathbf{W}^{(1)},\dots,\mathbf{W}^{(K)},\mathbf{w}\}$ during the training procedure with limited iterations satisfies

    $$\max\Big\{\|\mathbf{W}^{(1)}\|_{2},\dots,\|\mathbf{W}^{(K)}\|_{2},~\|\mathbf{w}\|_{2}\Big\}\leq B.$$

\subsection{Generalization Gap}\label{sec:general}
This section presents the main results of this paper. Under the assumptions made in Section \ref{Sec:Assumption}, the bound on the generalization gap of deep GCNs is provided in the following theorem.
\begin{theorem}[Generalization gap for deep GCNs]\label{thm:G_gap}
Consider the deep GCN model, defined in equation \eqref{equ:GCN_model},  which comprises $K$ hidden layers and utilizes $g(\mathbf{L})$ as the graph filter operator. The model is trained on $\mathcal{S}$ using SGD for $T$ iterations. Under Assumptions {\bf 1, 2 and 3} stated in Section \ref{Sec:Assumption}, the following expected generalization gap is valid with a probability of at least $1-\delta$, where $\delta\in(0,1)$:
\begin{align}
\epsilon_{gen}(\mathcal{A}_{\mathcal{S}})
\leq &\frac{1}{\sqrt{m}}\Bigg\{O\bigg(\Big((K+1)\eta\kappa_{1}+\eta\kappa_{2}\Big)^{T}\bigg)+M\sqrt{\frac{\log\frac{1}{\delta}}{2}}\Bigg\},\label{equ:model_G_gap}
\end{align}
  where
  \begin{align}
  \kappa_{1}:=& (\upsilon_{\ell}\alpha_{\sigma}^{2}
   +\alpha_{\ell}\nu_{\sigma})(B\alpha_{\sigma}C_{g})^{2K}C^{2}_{g}C^{2}_{\mathbf{X}} +\alpha_{\ell}(B\alpha_{\sigma}C_{g})^{K-1}\alpha_{\sigma}^{2}C_{g}^{2}C_{\mathbf{X}}, \label{equ:kappa1}
  \end{align}
  and
  \begin{equation}\label{equ:kappa2}
\kappa_{2}:=\nu_{\sigma}\big(B\alpha_{\sigma}C_{g}\big)^{K}C_{g}^{2}C^{2}_{\mathbf{X}}
\Big(\sum_{j=0}^{K-1}(j+1)(B\alpha_{\sigma}C_{g})^{j}\Big).
  \end{equation}
\end{theorem}

A fundamental correlation between the generalization gap and the parameters governing deep GCNs is induced by Theorem \ref{thm:G_gap}. This correlation implies that the uniform stability of deep GCNs, trained using the SGD algorithm, exhibits an increase with the number of samples when the upper bound approaches zero as the sample size $m$ tends to infinity. Specifically,
it is observed that if the value of $C_{g}$ (presenting the largest absolute eigenvalue of a symmetric $g(\mathbf{L})$ or the maximum singular value of an asymmetric $g(\mathbf{L})$) remains unaffected by the size $N$, a generalization gap decaying at the order of $O(1/\sqrt{m})$ is obtained.
To compare with the result in \cite{2019Stable_GCN}, let us discuss at length the role of  $g(\mathbf{L})$ and the hidden layer number $K$ on the generalization gap.

According to \eqref{equ:kappa1} and \eqref{equ:kappa2}, $\kappa_{1}=O\Big(C_{g}^{2K+2}\Big)$ and $\kappa_{2}=O\Big(C_{g}^{2K+1}\Big)$. Therefore,
  the bound on the generalization gap of deep GCNs in Theorem \ref{thm:G_gap} is
  \begin{equation}\label{equ:G_gap_multi_layer}
\epsilon_{gen}(\mathcal{A}_{\mathcal{S}})
\leq\frac{1}{\sqrt{m}}\left(O\Big(C_g^{2T(K+1)}\Big)+M\sqrt{\frac{\log\frac{1}{\delta}}{2}}\right).
  \end{equation}
 When $K=0$, the GCN model \eqref{equ:GCN_model} degenerates into the single-layer GCN model considered in \cite{2019Stable_GCN}. At this point, according to \eqref{equ:G_gap_multi_layer}, we have
 \begin{equation}\label{equ:G_gap_single_layer}
\epsilon_{gen}(\mathcal{A}_{\mathcal{S}})
\leq\frac{1}{\sqrt{m}}\left(O\Big(C_g^{2T}\Big)+M\sqrt{\frac{\log\frac{1}{\delta}}{2}}\right),
  \end{equation}
 which is the same as the result of \cite{2019Stable_GCN}.

\textbf{Remarks.} Based on \eqref{equ:G_gap_multi_layer}, we present certain observations regarding the impact of filter $g(\mathbf{L})$ and the hidden layer number $K$ on the generalization capacity of deep GCNs in \eqref{equ:GCN_model}.
\begin{itemize}
    \item \textbf{Normalized vs. Unnormalized Graph Filters:} We examine the three most commonly utilized filters: 1) $g_{1}(\mathbf{L})=\mathbf{A}+\mathbf{I}$, 2) $g_{2}(\mathbf{L})=\mathbf{D}^{-1/2}\mathbf{A}\mathbf{D}^{-1/2}
    +\mathbf{I}$, and 3) $g_{3}(\mathbf{L})=\mathbf{D}^{-1}\mathbf{A}
    +\mathbf{I}$. For the unnormalized filter $g_{1}$, its maximum absolute eigenvalue is bounded by $O(N)$. Consequently, as the value of $m$ approaches the magnitude to $N$, the upper bound indicated by \eqref{equ:G_gap_multi_layer} tends towards $O(N^p)$ for some $p>0$, leading to an impractical upper bound when $N$ become infinitely large. On the contrary,  for two normalized filters $g_{2}$ and $g_{3}$, their largest absolute eigenvalues are bounded and independent of graph size $N$. Therefore, both filters yield a diminishing generalization gap at a rate of $O(\frac{1}{\sqrt{m}})$ as $m$ goes to infinity. This discovery underscores the superior performance of normalized filters over unnormalized counterparts in deep GCNs. This observation is consistent with the findings in \cite{2019Stable_GCN,2023Stable_GCN}.
\item \textcolor{black}{\textbf{Low-pass vs. High-pass Graph Filters:} Our theoretical results are not restricted to the choice of $g(\mathbf{L})$ as either a low-pass or a high-pass filter. To illustrate, consider two exponential filters with symmetric $\mathbf{L}$: i) a low-pass filter $g_{\mathrm{low}}(\lambda)=e^{-b\lambda^{2}}$ and ii) a high-pass filter $g_{\mathrm{high}}(\lambda)=1-e^{-a\lambda^{2}}$, where $a,b>0$. In this setting, it is straightforward to verify that
    $$\|g_{\textrm{high}}(\mathbf{L})\|_{2}<\|g_{\textrm{low}}(\mathbf{L})\|_{2}=1.$$ Consequently, both filters lead to a vanishing generalization gap at the rate of $O\left(\tfrac{1}{\sqrt{m}}\right)$ as $m \to \infty$.}
 \item \textbf{The Role of Parameter $K$:} It is evident that, when the values of $C_{g}$ and $T$ are fixed, the upper bound \eqref{equ:G_gap_multi_layer} exhibits an exponential dependence on parameter $K$. This observation implies that a larger value $K$ leads to an increase in the upper bound of the generalization gap, thereby offering valuable insights for the architectural design of deep GCNs. This finding diverges from the ones presented in \cite{2019Stable_GCN,2023Stable_GCN}, as these studies do not account for generic deep GCNs and overlook the significance of the parameter $K$.
\end{itemize}

\textcolor{black}{Furthermore, based on Theorem \ref{thm:G_gap}, we give a brief analysis of the impact of $d_{k}$ (width of the $k$-th layer) on the generalization. Actually, the impact of $d_{k}$ on the generalization is reflected in its impact on $B$. More specifically, let us consider the case where parameters
$\{\mathbf{W}^{(1)},\dots,\mathbf{W}^{(K)},\mathbf{w}\}$ belong to the set $\mathcal{X}_{\xi}$, where
 $$\mathcal{X}_{\xi}:=\{\mathbf{W}:~~\|\mathbf{W}\|_{\infty}\leq\xi\},$$
 i.e., $\mathcal{X}_{\xi}$ is the collection of all matrices whose elements' absolute values are all less than $\xi$. At this point, for $\mathbf{W}^{(k)}\in\mathbb{R}^{d_{k-1}\times d_{k}}$, we have
 $$\sup_{\mathbf{W}^{(k)}\in\mathcal{X}_{\xi}}\|\mathbf{W}^{(k)}\|_{2}\leq\sup_{\mathbf{W}^{(k)}\in\mathcal{X}_{\xi}}\|\mathbf{W}^{(k)}\|_{F}
 \leq\xi\sqrt{d_{k-1}d_{k}}.$$
 Therefore, a larger $d_{k}$ (i.e., width of the $k$-th layer) results in a larger upper bound of $\|\mathbf{W}^{(k)}\|_{2}$, which implies that a larger $d_{k}$ results in a larger $B$ (see Assumption 3 in Section \ref{Sec:Assumption}). Finally, Theorem \ref{thm:G_gap} indicates that a larger $B$ leads to a larger bound on the generalization gap, thus we conclude that a larger $d_{k}$ leads to a larger bound on the generalization gap. To justify this argument, we add some experimental studies in Section \ref{Sec:Exp}. The empirical results are consistent with our analysis.
 }

\begin{table*}[htbp!]
    \footnotesize
    \centering    
    \textcolor{black}{\caption{Comparison of the generalization gap estimated based on uniform stability.\label{Tab:Com}}} 
    \setlength{\tabcolsep}{2.2em} 
    \renewcommand{\arraystretch}{1.2}
    \renewcommand{\aboverulesep}{0pt}
    \renewcommand{\belowrulesep}{0pt}
\begin{threeparttable}
\begin{tabular}{ccc}
\toprule

         \textcolor{black}{Reference}& \textcolor{black}{Model Architecture} & \textcolor{black}{{Estimated Upper Bound of the Generalization Gap}} \\
\hline
  \textcolor{black}{\cite{2019Stable_GCN}}   &\textcolor{black}{shallow}
  & \textcolor{black}{$\frac{1}{\sqrt{m}}\bigg(O\Big((1+\eta\upsilon_{\ell}\upsilon_{\sigma}C_g^{2})^{T}\Big)+M\sqrt{\frac{\log\frac{1}{\delta}}{2}}\bigg)$}
  \\
\hline
  \textcolor{black}{\cite{2023Stable_GCN}}    & \textcolor{black}{{shallow}}
  & \textcolor{black}{$\frac{1}{\sqrt{m}}\bigg(O\Big(\eta\alpha_{\ell}\alpha_{\sigma}c_{2,T}\sum\limits_{t=0}^{T-1}c_{6,t}\prod\limits_{s=t+1}^{T-1}
 (1+\eta c_{5,s})\Big)+M\sqrt{\frac{\log\frac{1}{\delta}}{2}}\bigg)$} \\
\hline
\textcolor{black}{\cite{2023Stable_Lp}} &\textcolor{black}{shallow} & \textcolor{black}{$\frac{1}{\sqrt{m}}\Bigg\{O\bigg(C^2_{g}\eta C_{p,\lambda}\sum\limits_{t=1}^{T}(C_{p,\lambda}(1+(\alpha_{\sigma}^{2}+\alpha_{\ell})\eta C_{g}^{2}))^{t-1}\bigg)
+M\sqrt{\frac{\log\frac{1}{\delta}}{2}}\Bigg\}$}\\
\hline
  \textcolor{black}{Ours}  &\textcolor{black}{ deep} & \textcolor{black}{$\frac{1}{\sqrt{m}}\Bigg\{O\bigg(\Big((K+1)\eta\kappa_{1}+\eta\kappa_{2}\Big)^{T}\bigg)
+M\sqrt{\frac{\log\frac{1}{\delta}}{2}}\Bigg\}$} \\
\hline
	\end{tabular}
\begin{tablenotes}
\scriptsize
\item \textcolor{black}{Note: $m$ is the number of samples in the trained dataset; $M$ is the upper bound of loss function $\ell(\cdot,\cdot)$; $\eta>0$ is the learning rate;  $\delta\in(0,1)$; $T$ is the number of iterations for training $\mathcal{A}_{\mathcal{S}}$ using SGD; $C_{g}$ represents the 2-norm of filter $g(\mathbf{L})$; $\alpha_{\sigma}$, $\upsilon_{\sigma}$ are two parameters w.r.t the continuity of activation function $\sigma(\cdot)$; $\alpha_{\ell}$, $\upsilon_{\ell}$ are two parameters w.r.t the continuity of the loss function $\ell(\cdot,\cdot)$.  $c_{2,t}, c_{6,t}, c_{5,t}>0$ ($t=0,1,\dots,T$) represent some specific parameters defined in \cite{2023Stable_GCN}. $C_{p,\lambda}=\frac{28}{p(p-1)\lambda_{t}}(B_{\ell}/\lambda)^{(3-p)/p}$, where $B_{\ell}>0$ is a parameter related to loss function $\ell(\cdot,\cdot)$, $1<p\leq2$, $\lambda>0$ is the regularization parameter and $\lambda_{t}>0$ is another regularization parameter dependent on $\lambda$ and $t$, as detailed in \cite{2023Stable_Lp}. $K$ is number of hidden layers of the considered deep GCNs; $\kappa_{1}$ and $\kappa_{2}$ are two parameters as defined in \eqref{equ:kappa1} and \eqref{equ:kappa2}.}
\end{tablenotes}
\end{threeparttable}
{\footnotesize }
\end{table*}

\textcolor{black}{
Table \ref{Tab:Com} offers a concise summary of various upper bounds on the generalization gap, derived through the application of uniform stability. From Table \ref{Tab:Com}, we can see that all the works derive a generalization gap decaying at the order of $O(1/\sqrt{m})$. However, compared to the other three works which only consider shallow GCNs, our work explores the case of deep GCNs.  We should point out that the generalization of single-layer GCNs into deep GCNs is not trivial. To derive the results for deep GCNs, we tackle two significant challenges that arise specifically in the context of deep GCNs, which are unique to deep GCNs and are non-existent in single-layer models. {\bf The first challenge} is the derivation of the gradient of the final output with respect to the learnable parameters across multiple layers, which requires determining how the gradient of the overall error of a GCN is shared among neurons in different hidden layers. In particular, in Appendix \textcolor{red}{A}, we provide a recursive formula to compute the related gradients.
  {\bf The second challenge} is the evaluation of gradient variations between GCNs trained on different datasets.  In the single layer case, since the input feature is the same, the variation of the related gradient is only dependent on the variations of learnable parameters. While, in the case of deep GCNs, the variation of the related gradients is also dependent on the variations of the gradients of the final output with respect to the hidden layer outputs. Please see Lemma \ref{lem:bound_variation_df_dW} and its proof for details (see Appendix \textcolor{red}{C}).
}

\subsection{Stability Upper Bound}
In this subsection, we establish the uniform stability of SGD for deep GCNs, which is the key to further proving Theorem \ref{thm:G_gap}.

\begin{theorem}[Uniform stability of deep GCNs]\label{thm:Uni_stability}
Consider the deep GCNs defined by equation \eqref{equ:GCN_model}, which are trained on a dataset $\mathcal{S}$ using the SGD algorithm for a total of $T$ iterations and denoted as $\mathcal{A}_{\mathcal{S}}$. Assume that Assumptions {\bf 1, 2 and 3} stated in Section \ref{Sec:Assumption} are satisfied. Then, $\mathcal{A}_{\mathcal{S}}$ is $\mu_{m}$-uniformly stable, with $\mu_{m}$ satisfying the following condition:
\begin{equation}\label{equ:thm_uni_stability}
  \mu_{m}\leq \frac{C}{m}
 \sum_{t=1}^{T}\Big(1+(K+1)\eta\kappa_{1}+\eta\kappa_{2}\Big)^{t-1},
\end{equation}
 where
 $$C:=(K+1)\eta\alpha^2_{\ell}(B\alpha_{\sigma}C_{g})^{2K}\alpha_{\sigma}^{2}C^{2}_{g}C^{2}_{\mathbf{X}},$$
 $\kappa_{1}$ and $\kappa_{2}$ are defined by \eqref{equ:kappa1} and \eqref{equ:kappa2}, respectively.
  \end{theorem}

With a straightforward calculation,
one can see that
$$\mu_{m}\leq\frac{1}{m}O\bigg(\Big((K+1)\eta\kappa_{1}+\eta\kappa_{2}\Big)^{T}\bigg),$$
which decays at the rate of $\frac{1}{m}$ as $m$ tends to infinity.
Together with Lemma \ref{lem:Sta_Guarantees}, it yields the result of Theorem \ref{thm:G_gap}.



{\bf Proof Sketch for Theorem \ref{thm:Uni_stability}}. We prove Theorem \ref{thm:Uni_stability} in the following two steps.

\begin{itemize}

\item \textbf{Step 1:} We begin by bounding the stability of deep GCNs with respect to perturbations in the learned parameters caused by changes in the training set. The result is given in Lemma~\ref{lemma:Step_1}.

\item \textbf{Step 2:} Next, we provide a bound for the perturbation of the learned parameters. The result is presented in Theorem \ref{thm:bound_variation_theta_T}.
\end{itemize}

Consider $\mathcal{A}_{\mathcal{S}}$, a set of deepGCNs defined by \eqref{equ:GCN_model}, trained on the dataset $\mathcal{S}$ using SGD for $T$ iterations. Let $\theta_{t}=\{\mathbf{W}^{(1)}_{t},\dots,\mathbf{W}^{(K)}_{t},\mathbf{w}_{t}\}$ and
 $\theta'_{t}=\{\mathbf{W}^{(1)'}_{t},\dots,\mathbf{W}^{(K)'}_{t},\mathbf{w}'_{t}\}$ (with $\theta_{0}=\theta'_{0}$) denote the parameters of two GCNs trained on $\mathcal{S}$ and $\mathcal{S}^{i}$ after $t$ iterations, respectively. We set $\triangle \mathbf{w}_t=\mathbf{w}_t-\mathbf{w}_t'$ and $\triangle \mathbf{W}^{(k)}_t=\mathbf{W}^{(k)}_t-\mathbf{W}^{(k)'}_t$ to be the perturbation of learning parameters and define \begin{equation}\label{equ:norm_theta}
  \|\triangle \theta_t\|_{*}=\|\triangle\mathbf{w}_t\|_{2}+\sum_{k=1}^{K}\|\triangle\mathbf{W}_t^{(k)}\|_{2}.
\end{equation}


 In the following lemma, it is shown that the stability of $\mathcal{A}_{\mathcal{S}}$ can be bounded by $\|\triangle\theta_{T}\|_{*}$.
\begin{lemma}\label{lemma:Step_1}
  Let $\theta_{t}$ and $\theta'_{t}$ be the learnt parameters of two GCNs trained on $\mathcal{S}$ and $\mathcal{S}^{i}$ using SGD in the $t$-th iteration with $\theta_{0}=\theta_{0}'$, and $\triangle\theta_{t}:=\theta_{t}-\theta'_{t}$. Suppose that all the assumptions made in Section \ref{Sec:Assumption} hold. Then, after $T$ iterations, we have that for any $\mathbf{z}=(\mathbf{x},y)$ taken from $\mathcal{D}$,
  \begin{align}
    &\Big|\mathbb{E}_{\mathcal{A}}\big[\ell(\hat{y},y)\big]
  -\mathbb{E}_{\mathcal{A}}\big[\ell(\hat{y}',y)\big]\Big|
 \leq \alpha_{\ell}B^{K}\alpha_{\sigma}^{K+1}C_{g}^{K+1}C_{\mathbf{X}}\cdot\mathbb{E}_{\mathcal{A}}\big[\|\triangle\theta_{T}\|_{*}\big],\label{equ:Step_1}
  \end{align}
 where $\hat{y}=f(\mathbf{x}|\theta_{T})$ and $\hat{y}'=f(\mathbf{x}|\theta'_{T})$.
  \end{lemma}
  We provide the proof of Lemma \ref{lemma:Step_1} in Appendix \textcolor{red}{B}.

  Combining \eqref{equ:Uni_stability1} and \eqref{equ:Step_1}, the stability of $\mathcal{A}_{\mathcal{S}}$ has a bound

  \begin{equation}\label{equ:bound_betam_vs_theta}
  \mu_{m}\leq\frac{\alpha_{\ell}B^{K}\alpha_{\sigma}^{K+1}C_{g}^{K+1}C_{\mathbf{X}}}{2}
  \sup_{\mathcal{S}}\Big\{\mathbb{E}_{\mathcal{A}}\big[\|\triangle\theta_{T}\|_{*}\big]\Big\}.
  \end{equation}

So, to estimate the uniform stability of $\mathcal{A}_{\mathcal{S}}$, we need to bound $\mathbb{E}_{\mathcal{A}}\big[\|\triangle\theta_{T}\|_{*}\big]$. Now, let us recall \eqref{equ:update_W_and_w} for parameter updating, for training on $\mathcal{S}$,
$$\mathbf{w}_{t}=\mathbf{w}_{t-1}-\eta\nabla_{\mathbf{w}}\ell(f(\mathbf{x}_{t}|\theta_{t-1}),y_{t}),$$
$$\mathbf{W}_{t}^{(k)}=\mathbf{W}^{(k)}_{t-1}-\eta\nabla_{\mathbf{W}^{(k)}}\ell(f(\mathbf{x}_{t}|\theta_{t-1}),y_{t}),$$
$k=1,2,\dots,K$, and for training on $\mathcal{S}^{i}$,
$$\mathbf{w}'_{t}=\mathbf{w}'_{t-1}-\eta\nabla_{\mathbf{w}}\ell(f(\mathbf{x}'_{t}|\theta'_{t-1}),y'_{t}),$$
$$\mathbf{W}_{t}^{(k)'}=\mathbf{W}^{(k)'}_{t-1}-\eta\nabla_{\mathbf{W}^{(k)}}\ell(f(\mathbf{x}'_{t}|\theta'_{t-1}),y'_{t}),$$
$k=1,2,\dots,K$, where $(\mathbf{x}_{t},y_{t})\in\mathcal{S}$ and $(\mathbf{x}'_{t},y'_{t})\in\mathcal{S}^{i}$ are the samples drawn at the $t$-th SGD iteration. Therefore, $\triangle\theta_{t}=\{\triangle\mathbf{W}^{(1)}_{t},\dots,\triangle\mathbf{W}^{(K)}_{t},\triangle\mathbf{w}_{t}\}$ has the following iterations:
\begin{align*}
&\triangle\mathbf{w}_{t}=\triangle\mathbf{w}_{t-1}-\eta\Big(\nabla_{\mathbf{w}}\ell(f(\mathbf{x}_{t}|\theta_{t-1}),y_{t})
-\nabla_{\mathbf{w}}\ell(f(\mathbf{x}'_{t}|\theta_{t-1}'),y'_{t})\Big),
\end{align*}
and for $k=1,2,\dots,K$,
\begin{align*}
&\triangle\mathbf{W}_{t}^{(k)}=\triangle\mathbf{W}^{(k)}_{t-1}-\eta\Big(\nabla_{\mathbf{W}^{(k)}}\ell(f(\mathbf{x}_{t}|\theta_{t-1}),y_{t})
    -\nabla_{\mathbf{W}^{(k)}}\ell(f(\mathbf{x}'_{t}|\theta_{t-1}'),y'_{t})\Big),
\end{align*}
with $\|\triangle\theta_{0}\|_{*}=0$.

\textcolor{black}{So, we need to bound 
$$\nabla_{\mathbf{w}}\ell(f(\mathbf{x}_{t}|\theta_{t-1}),y_{t})
-\nabla_{\mathbf{w}}\ell(f(\mathbf{x}'_{t}|\theta_{t-1}'),y'_{t})$$
and
$$\nabla_{\mathbf{W}^{(k)}}\ell(f(\mathbf{x}_{t}|\theta_{t-1}),y_{t})
    -\nabla_{\mathbf{W}^{(k)}}\ell(f(\mathbf{x}'_{t}|\theta_{t-1}'),y'_{t})$$ to obtain a bound of $\|\triangle\theta_{t}\|_{*}$. There are two
     scenarios to consider: i) At step $t$, SGD picks  a sample $\mathbf{z}_{t}=(\mathbf{x}_{t},\mathbf{y}_{t})$ which is identical in $\mathcal{S}$ and $\mathcal{S}^{i}$, and occurs with probability $(m-1)/m$; and ii) At step $t$, SGD picks the only samples that $\mathcal{S}$ and $\mathcal{S}^{i}$ differ, $\mathbf{z}_{t}=(\mathbf{x}_{t},\mathbf{y}_{t})$ and $\mathbf{z}'_{t}=(\mathbf{x}'_{t},\mathbf{y}'_{t})$ which occurs with probability $1/m$. We provide the results in the following Lemma \ref{lem:Same_sample} and Lemma \ref{lem:diff_sample}.}

\begin{lemma}\label{lem:Same_sample}
  Consider two GCNs with parameters $\theta_{t}$ and $\theta'_{t}$, respectively. Then,  the following holds for any sample $\mathbf{z}_{t}=(\mathbf{x}_{t},y_{t})$:
  \begin{align}
 &\|\nabla_{\mathbf{w}}\ell(f(\mathbf{x}_{t}|\theta_{t-1}),y_{t})-\nabla_{\mathbf{w}}\ell(f(\mathbf{x}_{t}|\theta'_{t-1}),y_{t})\big\|_{F}
\leq\kappa_{1}\|\triangle\theta_{t-1}\|_{*},\label{equ:bound_variation_dl_dw_same_sample}
  \end{align}
  and for $k=1,2,\dots,K$,
  \begin{align}
 &\|\nabla_{\mathbf{W}^{(k)}}\ell(f(\mathbf{x}_{t}|\theta_{t-1}),y_{t})-\nabla_{\mathbf{W}^{(k)}}\ell(f(\mathbf{x}_{t}|\theta'_{t-1}),y_{t})\|_{F}
 \leq  (\kappa_{1}+\rho_{k})\|\triangle\theta_{t-1}\|_{*}, \label{equ:bound_variation_dl_dW_same_sample}
\end{align}
where $\kappa_{1}$ and $\rho_{k}$ are defined by \eqref{equ:kappa1} and (\ref{equ:rho_k}).
\end{lemma}

\begin{lemma}\label{lem:diff_sample}
  Consider two GCNs with parameters $\theta_{t}$ and $\theta'_{t}$, respectively. Then,  the following holds for any two samples $\mathbf{z}_{t}=(\mathbf{x}_{t},y_{t})$ and $\mathbf{z}'_{t}=(\mathbf{x}'_{t},y'_{t})$:
  \begin{align}
 &\|\nabla_{\mathbf{W}^{(k)}}\ell(f(\mathbf{x}_{t}|\theta_{t-1}),y_{t})-\nabla_{\mathbf{W}^{(k)}}\ell(f(\mathbf{x}'_{t}|\theta'_{t-1}),y'_{t})\|_{F}
 \leq  2\alpha_{\ell}B^{K}\alpha_{\sigma}^{K+1}C^{K+1}_{g}C_{\mathbf{X}}, \label{equ:bound_variation_dl_dW_diff_sample}
  \end{align}
  for $k=1,2,\dots,K+1$. Note that $\mathbf{W}^{(K+1)}=\mathbf{w}$.
\end{lemma}

The proofs of Lemma \ref{lem:Same_sample} and Lemma \ref{lem:diff_sample} are given in  Appendix \textcolor{red}{C}. We now provide a bound for $\mathbb{E}_{\mathcal{A}}\big[\|\triangle\theta_{T}\|_{*}\big]$.
\begin{theorem}\label{thm:bound_variation_theta_T}
  Let $\theta_{t}$ and $\theta'_{t}$ be the learnt parameters of two GCNs trained on $\mathcal{S}$ and $\mathcal{S}^{i}$ using SGD in the $t$-th iteration with $\theta_{0}=\theta'_{0}$. The assumptions made in Section \ref{Sec:Assumption} hold. Then, after $T$ iterations, $\triangle\theta_{T}$ satisfies
  \begin{equation}\label{equ:bound_variation_theta_T}
  \mathbb{E}_{\mathcal{A}}\Big[\|\triangle\theta_{T}\|_{*}\Big]
  \leq c
  \sum_{t=1}^{T}\Big(1+(K+1)\eta\kappa_{1}+\eta\kappa_{2}\Big)^{t-1},
  \end{equation}
where $c:=\frac{2(K+1)\eta\alpha_{\ell}B^{K}\alpha_{\sigma}^{K+1}C^{K+1}_{g}C_{\mathbf{X}}}{m},$ and $\kappa_{1}$ and $\kappa_{2}$ are defined by \eqref{equ:kappa1} and \eqref{equ:kappa2}, respectively.
\end{theorem}
The proof of Theorem \ref{thm:bound_variation_theta_T}, using  Lemma \ref{lem:Same_sample} and Lemma \ref{lem:diff_sample}, is provided in  Appendix \textcolor{red}{D}.
Combining \eqref{equ:bound_betam_vs_theta} and Theorem \ref{thm:bound_variation_theta_T}, we obtain that the uniform stability $\mu_{m}$ of $\mathcal{A}_{\mathcal{S}}$ has a bound as
  \begin{align}
\nonumber    \mu_{m}&\leq\alpha_{\ell}B^{K}\alpha_{\sigma}^{K+1}C_{g}^{K+1}C_{\mathbf{X}}
\sup_{\mathcal{S}}\Big\{\mathbb{E}_{A}\big[\|\triangle\theta_{T}\|_{*}\big]\Big\}\\
\nonumber  &\leq\frac{C}{m}
  \sum_{t=1}^{T}\Big(1+(K+1)\eta\kappa_{1}+\eta\kappa_{2}\Big)^{t-1},
  \end{align}
which completes the proof of Theorem \ref{thm:Uni_stability}.

\section{Experiments} \label{Sec:Exp}
In this section, we conduct some empirical studies using three benchmark datasets commonly utilized for the node classification task, namely Cora, Citeseer, and Pubmed \cite{sen2008collective,yang2016revisiting}. Table \ref{tab:stats:node_classification} summarizes the basic statistics of these datasets. 
\vspace{-4mm}
\begin{table}[htbp!]
\footnotesize
\centering
\textcolor{black}{\caption{Statistics of the three benchmark datasets.}\label{tab:stats:node_classification}}
\begin{tabular}{lcccc}
\toprule
 & \textcolor{black}{\textbf{Cora}} & \textcolor{black}{\textbf{Citeseer}} & \textcolor{black}{\textbf{Pubmed}} \\
\midrule
\textcolor{black}{\# Nodes} & \textcolor{black}{$2,708$} & \textcolor{black}{$3,327$} & \textcolor{black}{$19,717$} \\
\textcolor{black}{\# Edges }& \textcolor{black}{$5,429$} & \textcolor{black}{$4,732$} & \textcolor{black}{$44,338$}  \\
\textcolor{black}{\# Features} & \textcolor{black}{$1,433$} & \textcolor{black}{$3,703$} & \textcolor{black}{$500$}  \\
\textcolor{black}{\# Classes} &\textcolor{black}{$7$ }& \textcolor{black}{$6$} & \textcolor{black}{$3$ } \\
\textcolor{black}{Label Rate} & \textcolor{black}{$0.052$} & \textcolor{black}{$0.036$} & \textcolor{black}{$0.003$}  \\
\bottomrule
\end{tabular}
\end{table}

In our experiments, we follow the standard transductive learning problem formulation and the training/test setting used in \cite{gcn}. To rigorously test our theoretical insights, our experiments aim to answer the following key questions:
\begin{itemize}
  \item Q1: How does the design of graph filters (i.e., $g(\mathbf{L})$) influence the generalization gap?
  \item Q2: How does the generalization gap change with the number of hidden layers (i.e., $K$)?
  \item Q3: How does the width (i.e., the number of hidden units: $d$) affect the generalization gap?
\end{itemize}
To address each question, we empirically estimate the generalization gap by calculating the absolute difference in loss between training and test samples. We adopt the official TensorFlow implementation (\url{https://github.com/tkipf/gcn}) for GCN \cite{gcn} and the Adam optimizer with default settings. The number of iterations is fixed to $T =200$ for all the simulations.

\noindent \textbf{Results and Discussion for Q1.} We analyze two types of graph filters in our study: 1) the normalized graph filter, defined as $g(\mathbf{L})=\mathbf{\tilde{D}}^{-1/2}\mathbf{\tilde{A}}\mathbf{\tilde{D}}^{-1/2}$ with $\mathbf{\tilde{A}}=\mathbf{A}+\mathbf{I}$ and $\mathbf{\tilde{D}}_{ii}=\sum_{j}\mathbf{\tilde{A}}_{ij}$ (which was first employed in the vanilla GCN \cite{gcn} and has subsequently become widely used in follow-up works on GCNs), and 2) the random walk filter, $g(\mathbf{L})=\mathbf{D}^{-1}\mathbf{A} +\mathbf{I}$. To fit our theoretical finding, we compare the performance of two 5-layer GCN models (with width $d=32$ for each layer), each employing one of these filters. Table \ref{Tab:exp_diff_filter} presents the numerical records of $R_{emp}(\mathcal{A}_{\mathcal{S}})$, $R(\mathcal{A}_{\mathcal{S}})$, $\epsilon_{gen}(\mathcal{A}_{\mathcal{S}})$, $C_{g}$ for both filters. The results indicate clearly that the 5-layer GCN with the normalized graph filter exhibits a smaller generalization gap compared to the one with the random walk filter. Furthermore, Fig. \ref{fig:exo_diff_filter_ite} illustrates the performance of each filter across different datasets over iterations, demonstrating the superior performance of the normalized graph filter. Overall, the empirical findings in Table \ref{Tab:exp_diff_filter} and Fig. \ref{fig:exo_diff_filter_ite} align well with our theoretical finding regarding the impact of $C_{g}$ on the generalization gap. 
\begin{figure*}[t!]
  \centering   \includegraphics[width=0.325\textwidth]{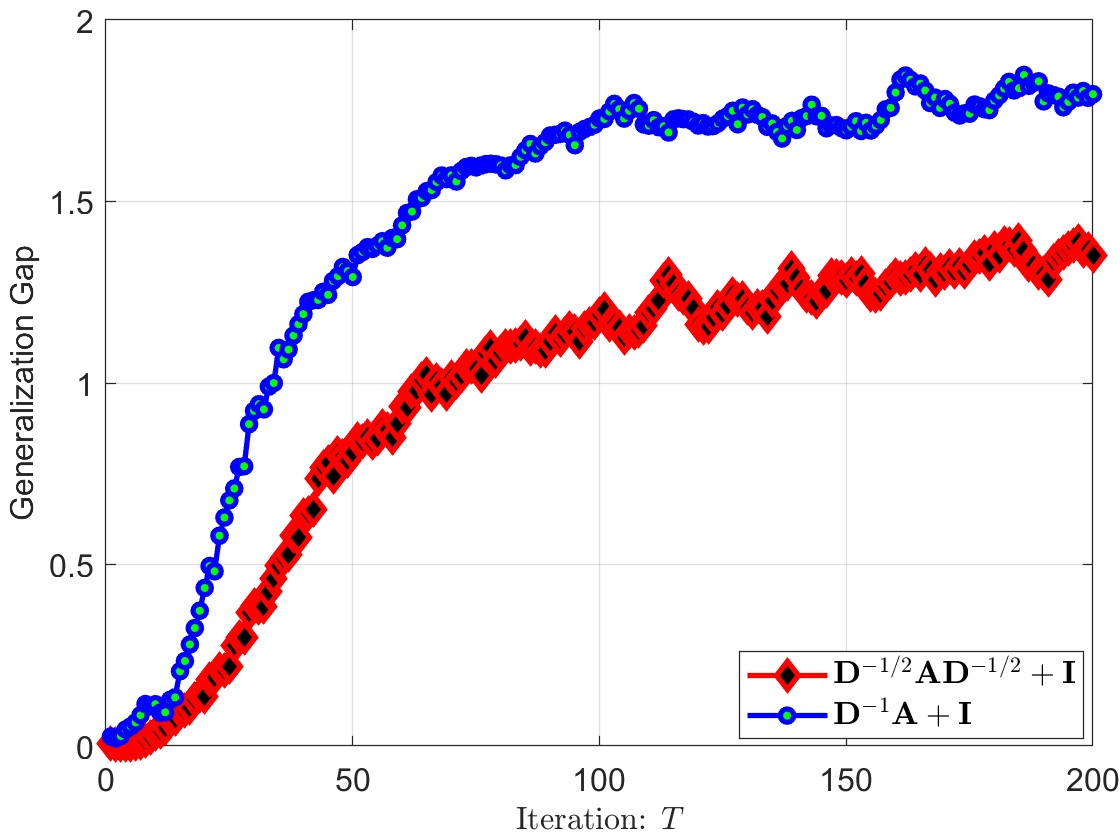}
 \includegraphics[width=0.325\textwidth]{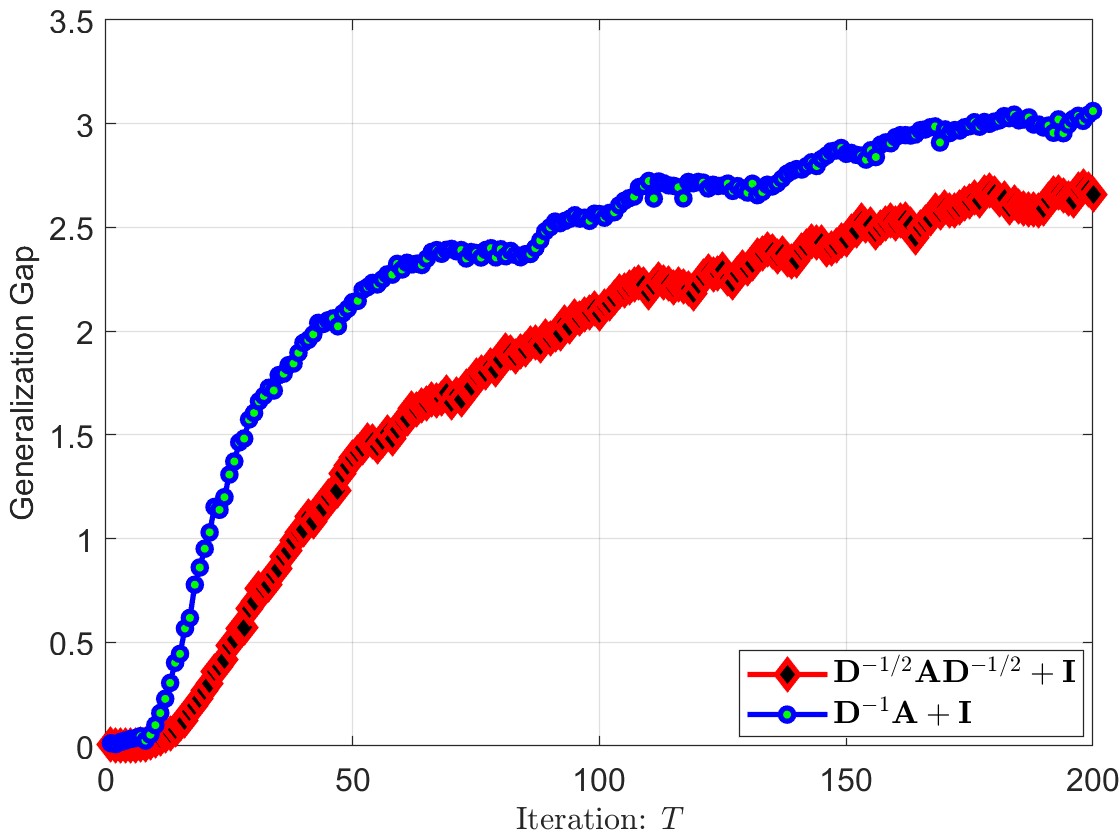}
   \includegraphics[width=0.325\textwidth]{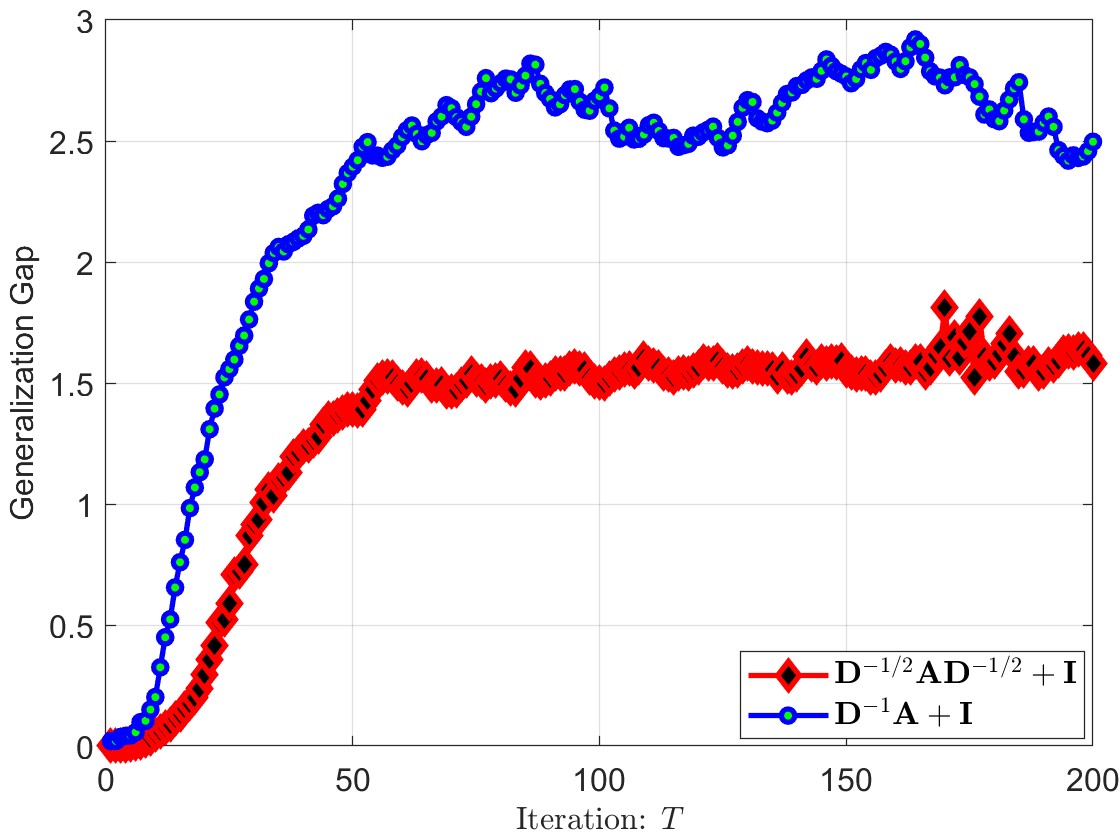}
   \vspace{-3mm}
 \textcolor{black}{ \caption{Comparison of trends in the generalization gap: Cora (left), Citeseer (middle), Pubmed (right).}\label{fig:exo_diff_filter_ite}}\vspace*{-1mm}
\end{figure*}
\begin{figure*}[htbp!]
  \centering
   \includegraphics[width=0.325\textwidth]{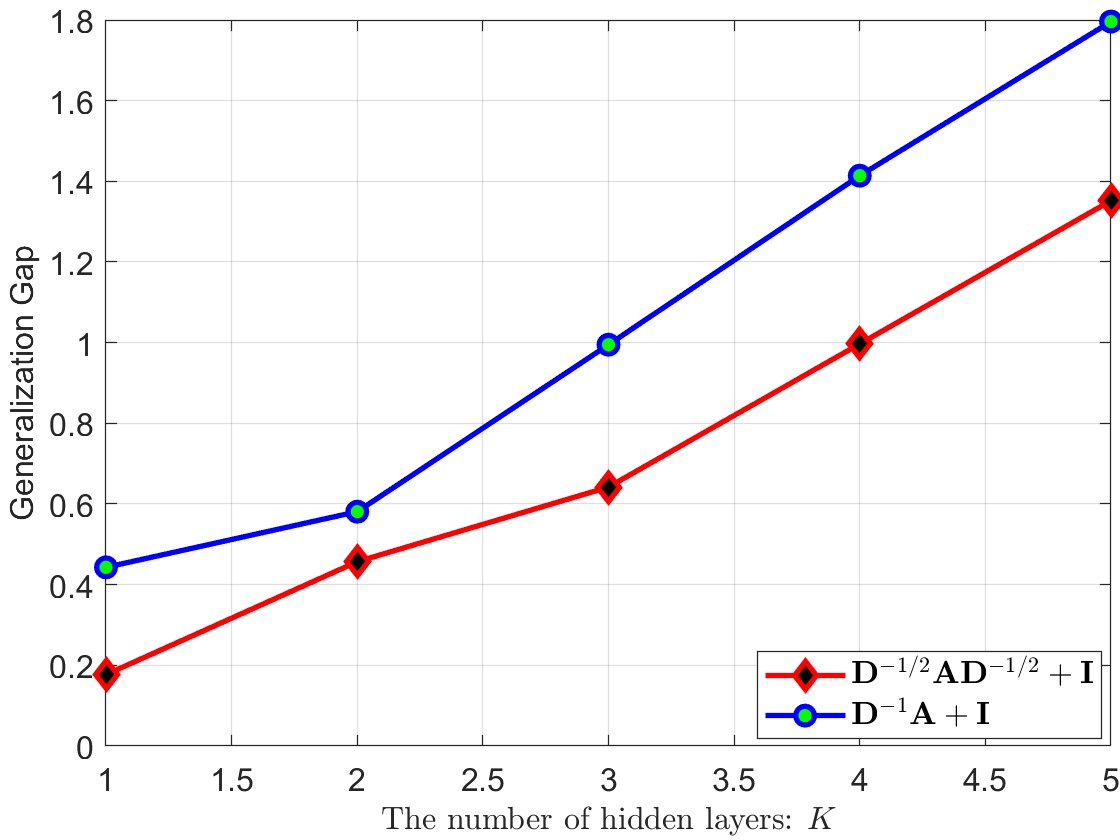}
 \includegraphics[width=0.325\textwidth]{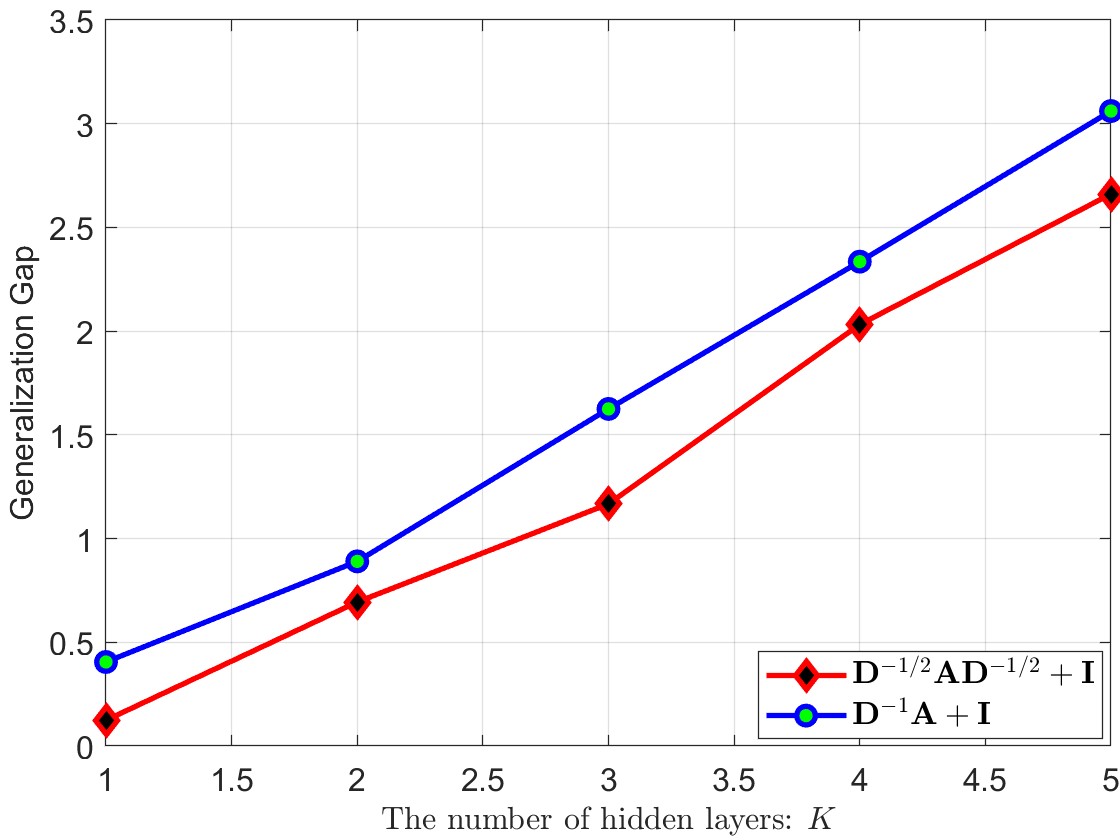}
   \includegraphics[width=0.325\textwidth]{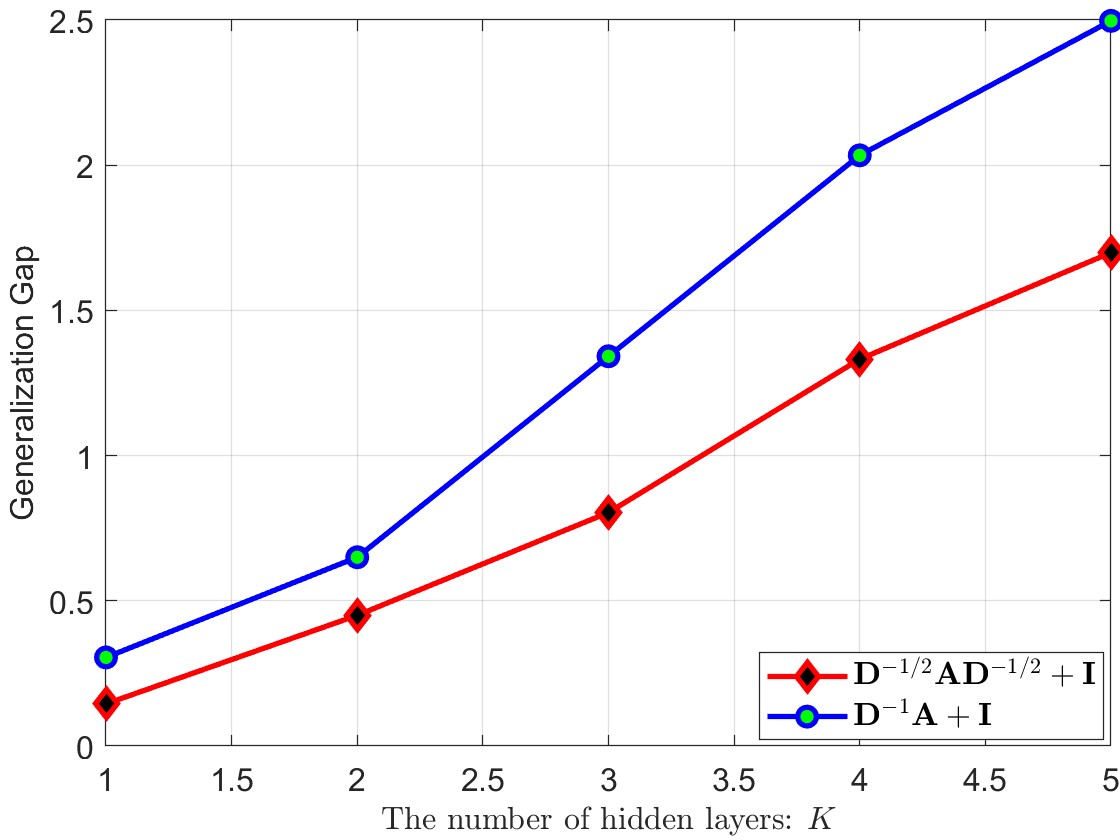}
  \vspace*{-3mm}
  \textcolor{black}{ \caption{Comparison of the generalization gap with different settings of network depth $K$: Cora (left), Citeseer (middle), Pubmed (right).}\label{fig:Role_of_K}}\vspace*{-1mm}
\end{figure*}
\begin{figure*}[t!]
  \centering   \includegraphics[width=0.326\textwidth]{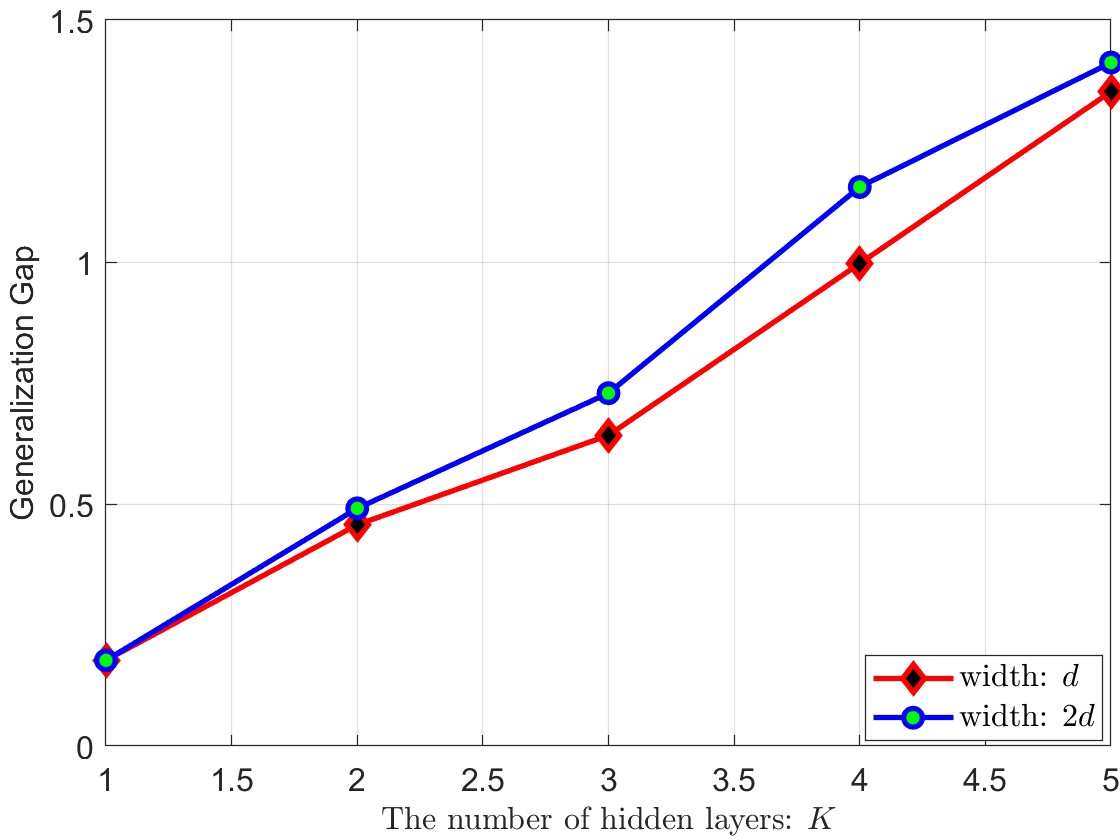}
 \includegraphics[width=0.326\textwidth]{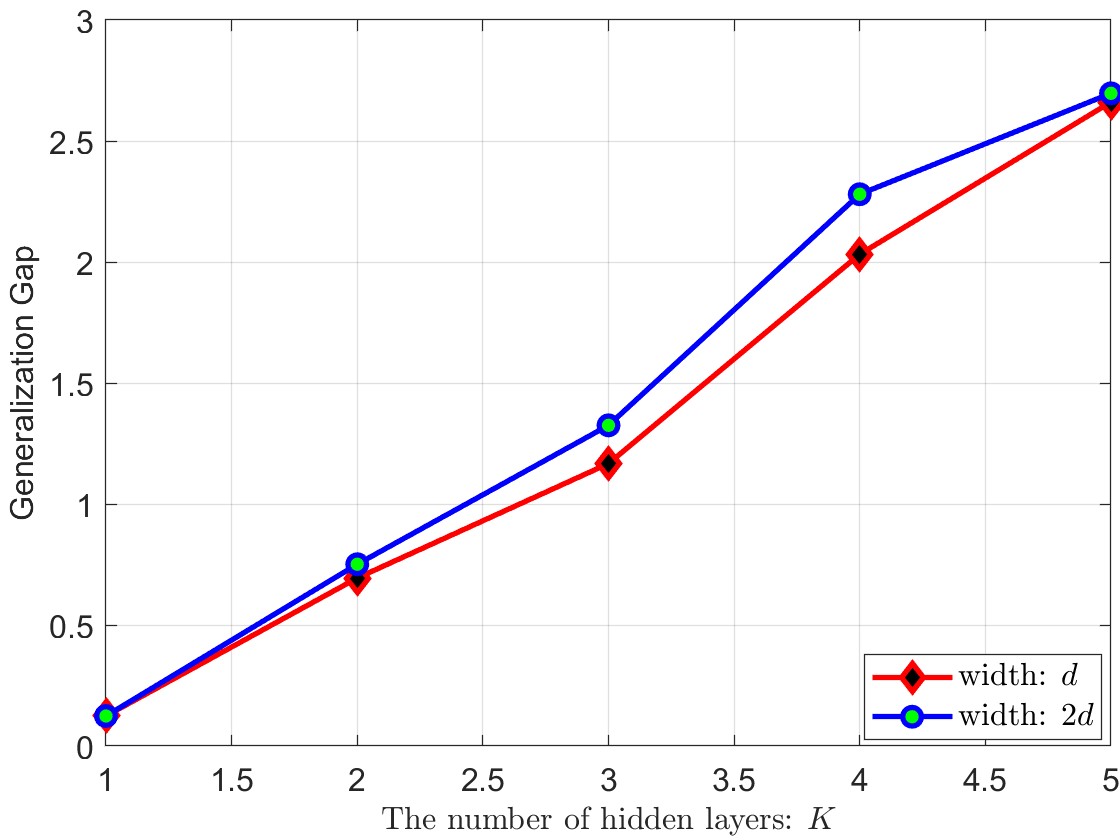}
   \includegraphics[width=0.326\textwidth]{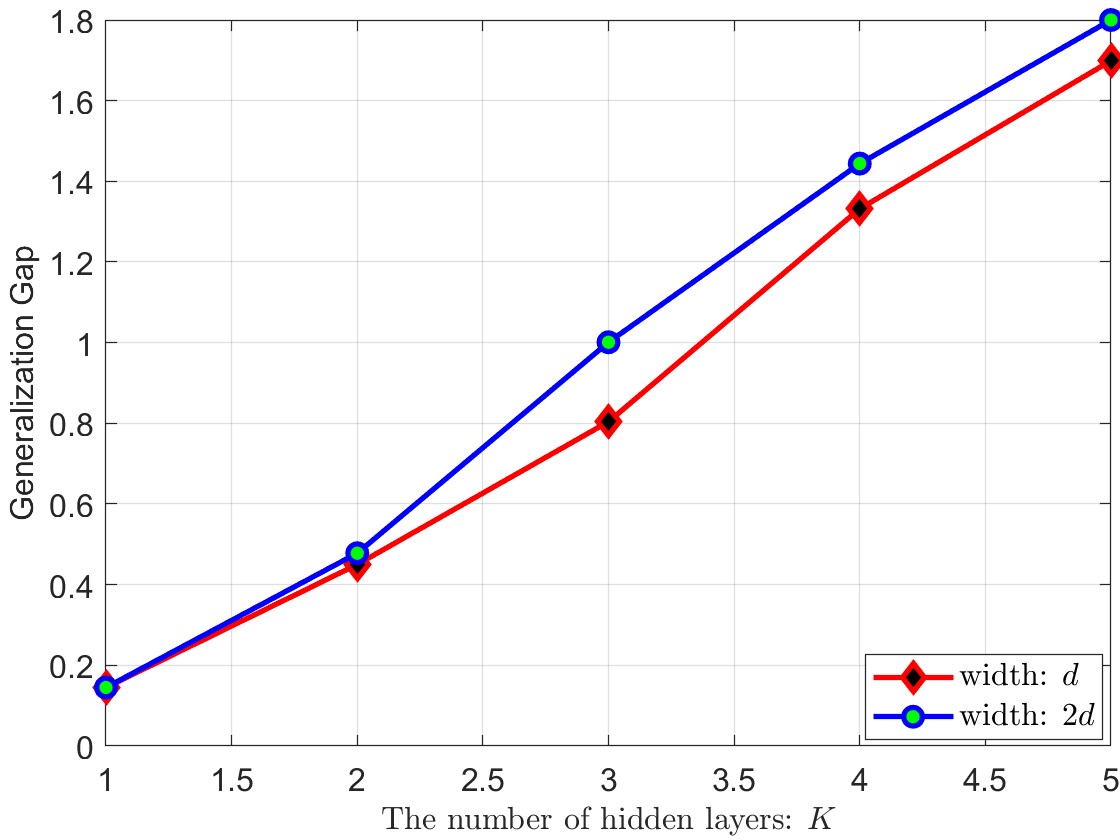}
\vspace*{-3mm}
 \textcolor{black}{ \caption{Comparison of the generalization gap with different settings of network width $d$: Cora (left), Citeseer (middle), Pubmed (right).}\label{fig:Role_of_d}}
\end{figure*}
\begin{table*}[htbp!]
\footnotesize
\centering
\textcolor{black}{\caption{The generalization gap with different graph filter for three datasets.}\label{Tab:exp_diff_filter} }
	\begin{tabular}{cccccc}
\toprule
         \textcolor{black}{Dataset} & \textcolor{black}{Graph filter $g(\mathbf{L})$} & \textcolor{black}{$R_{emp}(\mathcal{A}_{\mathcal{S}})$}
          & \textcolor{black}{$R(\mathcal{A}_{\mathcal{S}})$}  & \textcolor{black}{$\epsilon_{gen}(\mathcal{A}_{\mathcal{S}})$} & \textcolor{black}{$C_{g}$}\\
\midrule
 
 \multirow{2}{*}{\textcolor{black}{Cora}} &\textcolor{black}{$\mathbf{\tilde{D}}^{-1/2}\mathbf{\tilde{A}}\mathbf{\tilde{D}}^{-1/2}$}
    &\textcolor{black}{1.488}&\textcolor{black}{0.136}&\textcolor{black}{\textbf{1.352}}&\textcolor{black}{1}\\
&\textcolor{black}{$\mathbf{D}^{-1}\mathbf{A}
    +\mathbf{I}$}
    &\textcolor{black}{1.914}&\textcolor{black}{0.118}&\textcolor{black}{1.796}&\textcolor{black}{4.746}\\
\midrule
 
 \multirow{2}{*}{\textcolor{black}{Citeseer}}&\textcolor{black}{$\mathbf{\tilde{D}}^{-1/2}\mathbf{\tilde{A}}\mathbf{\tilde{D}}^{-1/2}$}
    &\textcolor{black}{2.896}&\textcolor{black}{0.235}&\textcolor{black}{\textbf{2.661}}&\textcolor{black}{1}\\
&\textcolor{black}{$\mathbf{D}^{-1}\mathbf{A}
    +\mathbf{I}$}
    &\textcolor{black}{3.206}&\textcolor{black}{0.145}&\textcolor{black}{3.061}&\textcolor{black}{4.690}\\
\midrule
 
 \multirow{2}{*}{\textcolor{black}{Pubmed}}&\textcolor{black}{$\mathbf{\tilde{D}}^{-1/2}\mathbf{\tilde{A}}\mathbf{\tilde{D}}^{-1/2}$}
    &\textcolor{black}{1.594}&\textcolor{black}{0.023}&\textcolor{black}{\textbf{1.571}}&\textcolor{black}{1}\\
&\textcolor{black}{$\mathbf{D}^{-1}\mathbf{A}
    +\mathbf{I}$}
    &\textcolor{black}{2.534}&\textcolor{black}{0.037}&\textcolor{black}{2.497}&\textcolor{black}{7.131}\\
\bottomrule
\end{tabular}
\end{table*}

\vspace{2mm}\noindent \textcolor{black}{\textbf{Results and Discussion for Q2.} In this experimental study, we try different settings of $K$, i.e., the number of hidden layers. Specifically, for $K=\{1,2,3,4,5\}$, we compare the performance of two $K$-layer GCNs (with width $d=32$ for each layer): one employing the normalized graph filter $g(\mathbf{L})=\mathbf{\tilde{D}}^{-1/2}\mathbf{\tilde{A}}\mathbf{\tilde{D}}^{-1/2}$, and one using the random walk filter $g(\mathbf{L})=\mathbf{D}^{-1}\mathbf{A} +\mathbf{I}$. Fig.~\ref{fig:Role_of_K} shows the performance comparison results for each $K$. It demonstrates clearly that, consistent with the aforementioned results for \textbf{Q1}, GCN with a normalized graph filter (with smaller $C_{g}$) consistently exhibits smaller generalization gaps compared to those with the random walk filter. Also, it is observed that the generalization gap becomes larger as $K$ increases, further validating our theoretical assertions regarding the influence of $K$ on the model's generalization gap.}

\noindent \textcolor{black}{\textbf{Results and Discussion for Q3.} To empirically investigate the impact of width $d$ (i.e., the number of hidden units) on the generalization gap, we conduct additional experiments using a 5-layer GCN equipped with a normalized graph filter. The experiments specifically involve a comparison between a 5-layer GCN configured with a width of $2d$ for each layer and the previously studied model with $d$ width ($d=32$), as illustrated in Fig. \ref{fig:Role_of_d}. This setup allows for a direct comparison under varying network configurations, providing insights into how changes in the number of hidden units influence the generalization gap. As demonstrated in Fig. \ref{fig:Role_of_d}, across all the datasets examined, a $d$-width GCN consistently exhibits smaller generalization gaps compared to one with a $2d$-width. This observation is in harmony with our theoretical explanation presented after Theorem 1, that is, the factor $B$ (i.e., the upper bound of 2-norm of the parameters $\{\mathbf{W}^{(1)},\dots,\mathbf{W}^{(K)},\mathbf{w}\}$) directly influences factors $\kappa_{1}$ and $\kappa_{2}$ in the upper bound of the generalization gap.}

\section{\textcolor{black}{Theoretical Implications}}\label{sec:Implications}
\textcolor{black}{Our work establishes a theoretical framework for analyzing the generalization gap of traditional deep GCNs, which further provides insights into extending the analysis to other classes of graph neural networks, including Graph Transformers. As illustrative examples, we briefly discuss how the theoretical proof methodology developed in our framework can be applied to GCNII and Graph Transformer, which are representative models of more advanced GNNs, thereby demonstrating the broader applicability of our theoretical framework.}

\subsection{\textcolor{black}{Extension to GCNII}}\label{sec:61}
\textcolor{black}{With input features $\mathbf{X}^{(0)}=\mathbf{X}\in\mathbb{R}^{N\times d}$, GCNII defines its $k$-th layer as
                  \begin{align*}
                  \mathbf{X}^{(k)}=\sigma\bigg(&\Big((1-a_{k})g(\mathbf{L})\mathbf{X}^{(k-1)}
                  +a_{k}\mathbf{X}^{(0)}\Big)\cdot\Big((1-b_{k})\mathbf{I}_{d}+b_{k}\mathbf{W}^{(k)}\Big)\bigg),
                  \end{align*}
                  for $k=1,2,\ldots,K$, where $a_{k},b_{k}\in(0,1)$ are two hyperparameters, $\mathbf{X}^{(k)}$ is the output feature matrix of the  $k$-th layer, $\mathbf{W}^{(k)}$ is the trained parameter matrix specific to  the $k$-th layer, graph filter $g(\mathbf{L})=\tilde{\mathbf{D}}^{-1/2}\tilde{\mathbf{A}}\tilde{\mathbf{D}}^{-1/2}$, and $\mathbf{I}_{d}$ is the $d\times d$ identity matrix. The output  for node $\mathbf{x}$ is
                  $$f(\mathbf{x} | \theta) = \sigma\left( \boldsymbol{\delta}_\mathbf{x}^\top
                  \left( (1 - a_{K+1}) g(\mathbf{L}) \mathbf{X}^{(K)} + a_{K+1} \mathbf{X}^{(0)} \right) \mathbf{w} \right),  $$
  where $ \theta = \{ \mathbf{W}^{(1)}, \mathbf{W}^{(2)},\dots, \mathbf{W}^{(K)}, \mathbf{w} \} $ (all trainable parameters, with $ \mathbf{w} \in \mathbb{R}^d $ the output layer parameter);
   $ \boldsymbol{\delta}_\mathbf{x} \in \mathbb{R}^N $ is the indicator vector for node $ \mathbf{x} $; $ a_{K+1} \in (0, 1) $ is a hyperparameter for the output layer residual connection.
   Let $\theta_{t}$ and $\theta'_{t}$ be the learned parameters of two GCNs trained on $\mathcal{S}$ and $\mathcal{S}^{i}$ using SGD in the $t$-th iteration with $\theta_{0}=\theta_{0}'$, and $\triangle\theta_{t}:=\theta_{t}-\theta'_{t}$.}

    \textcolor{black}{For each layer $ k $, the perturbation of layer outputs $ \|\triangle \mathbf{X}^{(k)}\|_F = \| \mathbf{X}^{(k)} - \mathbf{X}^{(k)'} \|_F $ satisfies the recursive bound:
  \begin{equation}\label{equ:GCNII_iterative_deltaXk}
  \|\triangle \mathbf{X}^{(k)}\|_F \leq c_1^{(k)} \|\triangle \mathbf{X}^{(k-1)}\|_F + c_2^{(k)} \|\triangle \mathbf{W}^{(k)}\|_2,
  \end{equation}
  where $c_{1}^{(k)}=(1-a_{k})(1-b_{k}+b_{k}B)\alpha_{\sigma}C_{g}$ and $c_{2}^{(k)}=\alpha_{\sigma}b_{k}\big((1-a_{k})C_{g}B_{\mathbf{X}}^{(k-1)}+a_{k}C_{\mathbf{X}}\big)$ with $ B_\mathbf{X}^{(k-1)} $ the bound of $ \| \mathbf{X}^{(k-1)} \|_F$ (see (\ref{equ:GCNII_iterative_bound_var_Xk}) in the Appendix \textcolor{red}{E}). The first term on the right side of the iterative formula captures propagation of perturbations from the previous layer, while the second term captures perturbation from $\mathbf{W}^{(k)}$.}

  \textcolor{black}{
   By induction, it yields that
  \begin{equation}\label{equ:GCNII_var_boundof_deltaXk}
  \|\triangle\mathbf{X}^{(k)}\|_{F}\leq e^{(k)}(\sum_{j=1}^{k}\|\triangle\mathbf{W}^{(k)}\|_{2}),
 \end{equation}
 where $e^{(k)}=\max\{c_{1}^{(k)}e^{(k-1)},c_{2}^{(k)}\}$ with $e^{(0)}=0$. We provide the proof of \eqref{equ:GCNII_iterative_deltaXk} and \eqref{equ:GCNII_var_boundof_deltaXk} in Appendix \textcolor{red}{E}. Then,  combining layer-wise bounds and using the Lipschitz property of $ \sigma $, one can have the output perturbation $ |f(\mathbf{x}| \theta) - f(\mathbf{x} | \theta')| $  bounded by the total parameter perturbation $ \|\Delta \theta\|_* = \sum\limits_{j=1}^K \| \mathbf{W}^{(j)} - \mathbf{W}^{(j)'} \|_2 + \| \mathbf{w} - \mathbf{w}' \|_2 $ (see Appendix \textcolor{red}{E} for technical details) as
   \begin{equation}\label{equ:GCNII_bound_var_finaloutput}
 |f(\mathbf{x} | \theta) - f(\mathbf{x} | \theta')|
  \leq \alpha_{\sigma}\cdot\varrho\|\triangle \theta\|_*,
  \end{equation}
  where $\varrho=\max\Big\{(1-a_{K+1})BC_g\cdot e^{(K)}, (1-a_{K+1})C_{g}B_{\mathbf{X}}^{(K)}+a_{K+1}C_{\mathbf{X}}\Big\}$.
  Then,
  \begin{align*}
 &\Big|\mathbb{E}_{\mathcal{A}}\big[\ell(\hat{y},y)\big]
  -\mathbb{E}_{\mathcal{A}}\big[\ell(\hat{y}',y)\big]\Big|  = \Big|\mathbb{E}_{\mathcal{A}}\big[\ell\big(f(\mathbf{x}|\theta_{T}),y\big)
 -\ell\big(f(\mathbf{x}|\theta'_{T}),y\big)\big]\Big|
 \leq\alpha_{\ell}\mathbb{E}_{\mathcal{A}}\Big[\big|f(\mathbf{x}|\theta_{T})-f(\mathbf{x}|\theta'_{T})\big|\Big]
\leq \varrho\alpha_{\ell}\cdot\mathbb{E}_{\mathcal{A}}\big[\|\triangle\theta_{T}\|_{*}\big].
 \end{align*}
 This implies that the stability of $\mathcal{A}_{\mathcal{S}}$ for GCNII has a bound
 $$\mu_{m}\leq\frac{\varrho\alpha_{\ell}}{2}\sup_{\mathcal{S}}\Big\{\mathbb{E}_{\mathcal{A}}[\|\triangle\theta_{T}\|_{*}]\Big\}.$$
 Note that when $a_{k}=0,b_{k}=1$ for all $k$, GCNII degenerates into the traditional GCN, we have $\varrho=B^{K}\alpha_{\sigma}^{K}C_g^{K+1}C_{\mathbf{X}}$, and thus
$$\mu_{m}\leq\frac{\alpha_{\ell}B^{K}\alpha_{\sigma}^{K}C_g^{K+1}C_{\mathbf{X}}}{2}\sup_{\mathcal{S}}\Big\{\mathbb{E}_{\mathcal{A}}[\|\triangle\theta_{T}\|_{*}]\Big\},$$
 which is consistent with \eqref{equ:bound_betam_vs_theta}.
 }

\textcolor{black}{
 To further bound $\|\triangle\theta_{T}\|_{*}$,  the  crucial step is to bound the perturbation of the gradient of $f(\mathbf{x}|\theta)$ with respect to the parameters $\theta=\{\mathbf{W}_{1},\mathbf{W}_{2},\dots,\mathbf{W}_{K},\mathbf{w}\}$ and obtain the result similar to Lemma \ref{lem:bound_variation_df_dW} in Appendix \textcolor{red}{C}, which can be achieved by following the technique in our paper. Here, we provide the result for $\|\nabla_{\mathbf{w}}f(\mathbf{x}|\theta)-\nabla_{\mathbf{w}}f(\mathbf{x}|\theta')\|_{F}$:
 \begin{align}
\|\nabla_{\mathbf{w}}f(\mathbf{x}|\theta)-\nabla_{\mathbf{w}}f(\mathbf{x}|\theta')\|_{F}\leq \Big(\nu_{\sigma}\varrho\cdot
\big((1-a_{K+1})C_{g}B_{\mathbf{X}}^{(K)} +a_{K+1}C_{\mathbf{X}}\big)+\alpha_{\sigma}\cdot(1-a_{K+1})C_ge^{(K)}\Big)
\cdot\|\triangle\theta\|_{*}, \label{equ:GCNII_bound_var_df_dw_sec6_1}
 \end{align}
where $\varrho=\max\Big\{(1-a_{K+1})BC_g\cdot e^{(K)}, (1-a_{K+1})C_{g}B_{\mathbf{X}}^{(K)}+a_{K+1}C_{\mathbf{X}}\Big\}$. Note that when $a_{k}=0,b_{k}=1$ for all $k$, GCNII degenerates into the traditional GCN, we have $\varrho=B^{K}\alpha_{\sigma}^{K}C_g^{K+1}C_{\mathbf{X}}$, $B_{\mathbf{X}}^{(K)}=B^{K}\alpha_{\sigma}^{K}C_{g}^{K}
C_{\mathbf{X}}$ and $e^{(K)}=B^{K-1}\alpha_{\sigma}^{K}C_{g}^{K}C_{\mathbf{X}}$. At this point,
  \begin{align*}
  \big\|\nabla_{\mathbf{w}}f(\mathbf{x}|\theta) -\nabla_{\mathbf{w}}f(\mathbf{x}|\theta')\big\|_{F}
 \leq \Big(\upsilon_{\sigma}B^{2K}\alpha^{2K}_{\sigma}C^{2K+2}_{g}C^{2}_{\mathbf{X}}
     +B^{K-1}\alpha^{K+1}_{\sigma}C^{K+1}_{g}C_{\mathbf{X}}\Big)\|\triangle\theta\|_{*},
  \end{align*}
  which is consistent with (\ref{equ:bound_variation_df_dw}) in Appendix \textcolor{red}{C}. For the bound of $\|\nabla_{\mathbf{W}^{(k)}} f(\mathbf{x} | \theta) - \nabla_{\mathbf{W}^{(k)}} f(\mathbf{x} | \theta')\|_{F}$, we refer the readers to  the proof process of (\ref{equ:GCNII_bound_var_df_dWk0}) in Appendix \textcolor{red}{E}.}

  \textcolor{black}{
  Finally, these structured analysis results can lead to the results corresponding Lemma \ref{lem:Same_sample} and Lemma \ref{lem:diff_sample}, and thus enable bounding the stability of GCNII.
}

\subsection{\textcolor{black}{Extension to Graph Transformer}}\label{sec:62}
\textcolor{black}{
  To extend our theoretical framework to more complex  models like Graph Transformer,  the key is to bound the generalization gap of Graph Transformer by quantifying how perturbations in the training set (e.g., removing or replacing a node) propagate to changes in model outputs. Graph Transformer introduce new learnable parameters: query $(\mathbf{W}_Q)$, key $(\mathbf{W}_K)$, and value $(\mathbf{W}_V)$ projection matrices, alongside attention scalers and feed-forward layers, for which a self-attention layer is defined \cite{GT} as
\begin{align*}
F\left(\mathbf{x}_n\right)= \mathbf{a}^{\top} \operatorname{Relu}\Big(\mathbf{W}_O \sum_{i \in \mathcal{T}^n} \mathbf{W}_V \mathbf{x}_i
\cdot \operatorname{softmax}_{n}\big( 
(\mathbf{W}_K \mathbf{x}_i)^{\top} \mathbf{W}_Q \mathbf{x}_n\big)\Big),
\end{align*}
  where $\mathbf{x}_i$ denotes features of node $i$, $\mathcal{T}^n$ is the set of nodes for the aggregation computing of node $n$, and $\operatorname{softmax}_{n}(h(i,n))=\exp(h(i,n))/\sum_{j\in\mathcal{T}^n}\exp(h(j,n)).$ Despite their architectural complexity (e.g., self-attention mechanisms, query/key/value projections), gradient decomposition still remains to be conducted via the product rule and chain rule, accounting for the propagation of attention-weight variations to the final output. Besides, a Lipschitz-type inequality for softmax may be critically needed, for which we claim that for $\mathbf{z}=(z_1,z_2,\dots, z_p)$, $\mathbf{z}'=(z'_1,z'_2,\dots, z'_p)$ with $\|\mathbf{z}-\mathbf{z}'\|_\infty\leq 1$,
\begin{equation}\label{softmax-2}
\|\operatorname{softmax}(\mathbf{z})-\operatorname{softmax}(\mathbf{z}')\|_{1}\leq 2e\|\mathbf{z}-\mathbf{z}'\|_\infty.
\end{equation}
Actually, the proof is not hard to set up by straight forward boundedness and the mean value theorem of exponential functions (see the technical details in Appendix \textcolor{red}{F}).}

\textcolor{black}{
For trainable parameters $ \mathbf{W}_Q, \mathbf{W}_K, \mathbf{W}_V $, set the attention output is:
$$ F(\mathbf{x}_n) = \mathbf{a}^{\top} \text{ReLu}\Big(\mathbf{W}_O\sum_{i\in \mathcal{T}^n} \mathbf{W}_V \mathbf{x}_i \cdot \operatorname{Attn}(\mathbf{x}_n)_i \Big),$$
where $ S_{i,n} = (\mathbf{W}_K\mathbf{x}_i)^T (\mathbf{W}_Q\mathbf{x}_n)$ is the scaled dot-product score,  $A_{i,n}= \operatorname{softmax}_{n}(S_{i,n})$ are attention weights, and $\operatorname{Attn}(\mathbf{x}_n)=\sum_{i\in \mathcal{T}^n} \mathbf{W}_V \mathbf{x}_i \cdot A_{i,n}$ the attention output. Then the gradient decomposition with respect to $\mathbf{W}_K$ is given by
\begin{align*}
\nabla_{\mathbf{W}_K} F(\mathbf{x}_n) = & \underbrace{\nabla_{\text{ReLU}(\mathbf{Z})} F(\mathbf{x}_n)}_{\textcircled{1}} \cdot \underbrace{\nabla_\mathbf{Z} \text{ReLU}(\mathbf{Z})}_{\textcircled{2}} \cdot  \underbrace{\nabla_{\text{Attn}(\mathbf{x}_n)} \mathbf{Z}}_{\textcircled{3}} \cdot \underbrace{\nabla_\mathbf{A} \text{Attn}(\mathbf{x}_n)}_{\textcircled{4}} 
\cdot \underbrace{\nabla_S \mathbf{A}}_{\textcircled{5}} \cdot \underbrace{\nabla_{\mathbf{W}_K} S}_{\textcircled{6}}
\end{align*}
where $ \mathbf{Z} = \mathbf{W}_O \cdot \text{Attn}(\mathbf{x}_n) $, $ \mathbf{A} = \{A_{i,n}\} $, and $ \mathbf{S} = \{S_{i,n}\} $. Then calculating each item gives that
\begin{align*}
\nabla_{\mathbf{W}_K} F(\mathbf{x}_n) = &\mathbf{a}^{\top} \mathbb{I}_{\geq 0}(\mathbf{W}_O \cdot \text{Attn}(\mathbf{x}_n)) \cdot \mathbf{W}_Q \cdot \mathbf{W}_V \cdot \left( \sum_{i \in \mathcal{T}^n} \mathbf{A}_{i,n} (\mathbf{x}_i - \bar{\mathbf{x}}_n) \mathbf{x}_i^{\top} \right) \cdot (\mathbf{W}_Q \mathbf{x}_n)^{\top}.
\end{align*}
}

\textcolor{black}{
 By  leveraging the Lipschitz continuity  of the gradient with respect to its trainable parameters, it can lead to  bounding the gradient perturbation  in terms of the total parameter perturbation $ \|\triangle\theta\|_{*} = \|\mathbf{W}_K - \mathbf{W}'_K\|_{2} + \|\mathbf{W}_V - \mathbf{W}'_V\|_{2} + \|\mathbf{W}_O - \mathbf{W}'_O\|_{2} + \|\mathbf{W}_Q - \mathbf{W}'_Q\|_{2} + \|\mathbf{a} - \mathbf{a}'\|_{2} $ by
 \begin{equation}\label{attention-1}
\|\nabla_{\mathbf{W}_K} F(\mathbf{x}_n|\theta)-\nabla_{\mathbf{W}_K} F(\mathbf{x}_n|\theta')\|_{2} \leq2e K_{\max} B^3 C_\mathbf{X}^3\|\Delta\theta\|_{*},
\end{equation}
where $K_{\max}\geq |\mathcal{T}^{n}| $ is the maximum neighborhood size, $B$ is the upper bound of weight matrices (technical details in Appendix \textcolor{red}{F}). It mirrors the Lemma \ref{lem:bound_variation_df_dW} in our approach for deep GCNs, where we recursively decomposed gradients across layers (see Lemma \ref{lem:bound_variation_df_dW}). For Graph Transformer, similar recursive relations can be derived for attention layers, with additional terms capturing interactions between $\mathbf{W}_Q\mathbf{X}, \mathbf{W}_K\mathbf{X}, \mathbf{W}_V \mathbf{X}$. For GCNs, we bounded gradient variations using norms of graph filters and layer parameters (e.g., $\|g(\mathbf{L})\|_2$, $\|\mathbf{W}^{(k)}\|_2$). For Graph Transformer, this will be extended to:
  singular values of $\mathbf{W}_Q, \mathbf{W}_K, \mathbf{W}_V$ (analogous to $C_g$ in GCNs), as they control the "strength" of feature projections and Lipschitz constants of softmax and feed-forward activations (replacing $\alpha_\sigma$ for GCN activations, and leads to an analogous to Theorem \ref{thm:Uni_stability} for deep GCNs.}

\section{Conclusion and Further Remarks}\label{sec:Conclusion}
This paper explores the generalization of deep GCNs by providing an upper bound on their generalization gap. Our generalization bound is obtained based on the algorithmic stability of deep GCNs trained by the SGD algorithm. Our analysis demonstrates that the algorithmic stability of deep GCNs is contingent upon two factors: the largest absolute eigenvalue (or maximum singular value) of graph filter operators and the number of layers utilized. In particular, if the aforementioned eigenvalue (or singular value) remains invariant regardless of changes in the graph size, deep GCNs exhibit robust uniform stability, resulting in an enhanced generalization capability. Additionally, our results suggest that a greater number of layers can increase the generalization  gap and subsequently degrade the performance of deep GCNs. This provides guidance for designing well-performing deep GCNs with a proper number of layers \cite{li2023deepergcn}. Most importantly, the result of single-layer GCNs in \cite{2019Stable_GCN} can be regarded as a special case of our results in deep GCNs without hidden layers.

\textcolor{black}{While our study is primarily focused on exploring the fundamental principles of generalizability and stability in the context of a simple deep GCN model framework, the theoretical insights obtained here can also offer preliminary perspectives on several research topics that have drawn increasing attention in the graph neural network community. These include, among others, the over-smoothing problem in deep architectures \cite{rusch2023survey,chen2022bag}, the design of models tailored for heterophilic graphs \cite{zheng2022graph,zhu2023heterophily}, and the emerging topic of graph out-of-distribution (OOD) generalization \cite{li2022ood,li2025out}. Our theoretical study can provide potential hints toward these directions, but more fine-grained and comprehensive work is still needed to fully address them. Below, we elaborate on these aspects in turn, aiming to clarify their conceptual connections with our work, outline possible directions for extending our theoretical framework, and highlight three open and challenging questions that can serve as seeds for future exploration.}

\textcolor{black}{\textit{How can the impact of over-smoothing in deep GCNs be mitigated?} We first note that, given a trivial deep GCN model characterized by over-smoothed node embeddings (which typically result in significant training errors), our theoretical upper bound still holds — that is, for a given graph filter, an increase in layers could potentially increase this upper bound in a probabilistic sense. This also motivates the exploration of advanced deep GCN models that incorporate mechanisms to counteract over-smoothing, such as the skip connection technique used in GCNII \cite{chen2020simple} and its follow-up works. As detailed in Section 5, our theoretical results can in fact be extended to the setting of GCNII, thereby providing analytical support for architectures that integrate skip connections. In both theory and practice, reducing the maximum absolute eigenvalue of graph filter operators is achievable through the strategic implementation of skip connections across layers, which can potentially reduce the generalization gap. From this perspective, our findings may inspire further studies into sophisticated deep GCN architectures designed to mitigate over-smoothing, offering a promising direction for both theoretical and practical advancements.}

\textcolor{black}{\textit{What is the role of heterophily in GCN generalization?} It is also valuable to consider extending our theoretical analysis to models specifically designed for heterophilic graphs, where nodes often connect to neighbors with dissimilar labels. This would require incorporating the homophily/heterophily ratio of the input graph signal into the upper bound estimation, thereby capturing how graph signal characteristics influence generalization. Although our empirical study here considers two types of low-pass filters on homophilic benchmark datasets (Cora, Citeseer, Pubmed), our theoretical framework is not restricted to low-pass scenarios alone. As remarked in Section 4.2, the analysis framework is in principle applicable to a broader range of filtering schemes; however, the derivations in our proofs do not explicitly examine the impact of specific quantities such as the homophily/heterophily ratio, leaving this as an open aspect for further refinement. To ensure a consistent and fair empirical evaluation, as demonstrated in \cite{2019Stable_GCN}, we adopt homophilic datasets that are standard in prior stability and generalization analyses of GCNs. For analyses involving high-pass filters, it would be appropriate to engage with heterophilic benchmark datasets (e.g., Texas, Wisconsin, Cornell). Relevant to this discussion is the recent work \cite{shi2024homophily}, which employs analytical tools from statistical physics and random matrix theory to precisely characterize generalization in simple GCNs on the contextual stochastic block model (CSBM). Such studies, although based on specific graph signal assumptions, could inspire refinements to our theoretical framework by jointly considering graph signal characteristics (homophily/heterophily) and model complexities (filter types, depth, and width).}

\textcolor{black}{\textit{Can insights from in-distribution generalization inform OOD generalization?} Beyond the above considerations, another relevant line of research that has recently attracted considerable attention is graph out-of-distribution (OOD) generalization \cite{li2022ood,li2025out}. It is worth clarifying that the problem setting and theoretical assumptions in OOD generalization are distinct from those in the in-distribution generalization framework considered in this work. In-distribution generalization focuses on scenarios where both training and test data are drawn from the same underlying distribution, enabling rigorous analysis under well-defined stochastic assumptions, such as those adopted in our stability-based framework. In contrast, OOD generalization addresses cases involving distribution shifts, which often require additional modeling principles (e.g., invariance to spurious correlations, causal structure modeling, or domain adaptation techniques) and seek performance guarantees that hold across domains. Despite these differences, the two areas can be mutually beneficial: in-distribution analyses, such as our characterization of bias–variance trade-offs and the influence of spectral properties of graph filters on generalization, may offer insights for developing more OOD-robust architectures; conversely, OOD-oriented approaches, such as invariant risk minimization or causal subgraph intervention, may inspire new regularization schemes or architectural components that also enhance in-distribution performance. Related to this discussion, the authors in \cite{baranwal2021graph} analyze a one-layer GCN trained on the CSBM via logistic regression, providing theoretical insights into improved linear separability and out-of-distribution generalization in semi-supervised node classification. Extending the current stability-based framework to accommodate mild forms of distribution shift thus presents an appealing research direction that could bridge these two lines of work and advance the understanding of generalization in graph neural networks. }

\textcolor{black}{Taken together, these discussions highlight that our theoretical framework, while developed under a specific in-distribution setting, has the potential to be extended and adapted to address a broader range of challenges in graph learning.}

\textcolor{black}{Building on the above open questions, which outline core challenges for future exploration, it is also important to consider more concrete research directions and methodological extensions. For example, the theoretical analysis presented in this study could be extended to encompass other commonly used learning algorithms in graph neural networks, moving beyond the scope of SGD. Our theoretical results may also inform the exploration of strategies to enhance the generalization capability of deep graph neural networks, such as investigating the efficacy of regularization techniques, conducting advanced network architecture searches, or developing adaptive graph filters. In addition, establishing the potential connection between model stability, generalization, and the issues of over-smoothing and over-squashing represents another promising avenue. Understanding these interrelationships could contribute to the development of novel techniques and algorithms that address these challenges, thereby complementing the broader problem-oriented directions discussed above and improving the overall effectiveness of deep graph neural networks in dealing with more complex tasks.}
\setcounter{lemma}{4}
\setcounter{equation}{0}
\renewcommand{\theequation}{A.\arabic{equation}}

\section*{Acknowledgment}
This work was supported in part by the National Natural Science Foundation of China (No. U21A20473, No. 62536006, No. 62172370). M. Li also acknowledged the support from the ``Pioneer'' and ``Leading Goose'' R\&D Program of Zhejiang (No. 2024C03262). G. Yang acknowledged the support from the Opening Project of Guangdong Province Key Laboratory of Computational Science at the Sun Yat-sen University (No. 2024008). H. Feng was supported in part  by the Research Grants Council of Hong Kong (Project no. CityU 11303821 ,and CityU 11315522). X. Zhuang was supported in part  by the Research Grants Council of Hong Kong (Project no. CityU 11309122, CityU 11302023, CityU 11301224, and CityU 11300825). The authors also wish to thank Dr. Yi Wang (City University of Hong Kong, Hong Kong SAR, China) and Dr. Xianchen Zhou (National University of Defense Technology, China) for their insightful discussions and dedicated assistance with the experimental studies.
\section*{Appendix: Preliminaries}\label{appendix_section}

The proofs of our main results are given in this section. We first make some statements about the notations used in the paper. $\mathbf{W}^{\top}$ denotes the transpose of a matrix $\mathbf{W}$; the $(i,j)$-entry of $\mathbf{W}$ is denoted as $\mathbf{W}_{ij}$; however when contributing to avoid confusion,
the alternative notation $\mathbf{W}(i,j)$ will be used. $\|\cdot\|_{2}$ denotes the 2-norm of a matrix or vector and $\|\cdot\|_{F}$ denotes the Frobenius norm. $\boldsymbol{\delta}_{i}$ denotes the unit pulse signal at node $i$ that all elements are 0 except the $i$-th one, which is 1. Let $f:\mathbb{R}^{m\times n}\rightarrow\mathbb{R}$ be a real-valued function of variable $\mathbf{W}\in\mathbb{R}^{m\times n}$. Then, the gradient of $f$ with respect to $\mathbf{W}$ is denoted as
 $$\nabla_{\mathbf{W}}f=\frac{\partial f}{\partial \mathbf{W}}=(\frac{\partial f}{\partial \mathbf{W}_{ij}})\in\mathbb{R}^{m\times n}.$$

 To make it easier to understand the derivation of our results, we first provide the following inequalities, which will be used frequently in the derivation.

For any matrix $\mathbf{A}_{1}$, $\mathbf{A}_{2}$, $\mathbf{A}'_{1}$ and $\mathbf{A}'_{2}$, we have:
 \begin{itemize}
 \item  $\|\mathbf{A}_{1}\mathbf{A}_{2}\|_{F}\leq\|\mathbf{A}_{1}\|_{2}\|\mathbf{A}_{2}\|_{F}$. To prove this, let $\mathbf{A}_{1}=\mathbf{U}\Sigma\mathbf{V}^{\top}$ be the SVD of $\mathbf{A}_{1}$, where $\mathbf{U}$ and $\mathbf{V}$ are both orthogonal matrix. Then,
 $$\|\mathbf{A}_{1}\mathbf{A}_{2}\|_{F}=\|\mathbf{U}\Sigma\mathbf{V}^{\top}\mathbf{A}_{2}\|_{F}
 =\|\Sigma\mathbf{V}^{\top}\mathbf{A}_{2}\|_{F}\leq\|\Sigma\|_{2}\|\mathbf{V}^{\top}\mathbf{A}_{2}\|_{F}
 =\|\mathbf{A}_{1}\|_{2}\|\mathbf{A}_{2}\|_{F}.$$
Similarly, we also have $\|\mathbf{A}_{1}\mathbf{A}_{2}\|_{F}\leq\|\mathbf{A}_{1}\|_{F}\|\mathbf{A}_{2}\|_{2}$.

 \item $\|\mathbf{A}_{1}\mathbf{A}_{2}-\mathbf{A}'_{1}\mathbf{A}'_{2}\|_{F}
       \leq\|\mathbf{A}_{1}-\mathbf{A}_{1}'\|_{F}\|\mathbf{A}_{2}\|_{2}
       +\|\mathbf{A}'_{1}\|_{F}\|\mathbf{A}_{2}-\mathbf{A}_{2}'\|_{2}.$ To show this, note that
       \begin{align*}
          \|\mathbf{A}_{1}\mathbf{A}_{2}-\mathbf{A}'_{1}\mathbf{A}'_{2}\|_{F}
         =&\|(\mathbf{A}_{1}-\mathbf{A}'_{1})\mathbf{A}_{2}+\mathbf{A}'_{1}(\mathbf{A}_{2}-\mathbf{A}'_{2})\|_{F}\\
        \leq&\|(\mathbf{A}_{1}-\mathbf{A}'_{1})\mathbf{A}_{2}\|_{F}+\|\mathbf{A}'_{1}(\mathbf{A}_{2}-\mathbf{A}'_{2})\|_{F}.
       \end{align*}
       Then, the proof is complete using the first inequality $\|\mathbf{A}_{1}\mathbf{A}_{2}\|_{F}\leq\|\mathbf{A}_{1}\|_{F}\|\mathbf{A}_{2}\|_{2}$,
  \item $\|\mathbf{A}_{1}\odot\mathbf{A}_{2}\|_{F}\leq \alpha\|\mathbf{A}_{1}\|_{F}\leq\|\mathbf{A}_{1}\|_{F}\|\mathbf{A}_{2}\|_{F}$, where $\alpha$ is the maximum absolute value of the entries of $\mathbf{A}_{2}$. Note that $\alpha\|\mathbf{A}_{1}\|_{F}\leq\|\mathbf{A}_{1}\|_{F}\|\mathbf{A}_{2}\|_{F}$ holds true because $\alpha\leq\|\mathbf{A}_{2}\|_{F}$. Furthermore,
      \begin{align*}
        &\|\mathbf{A}_{1}\odot\mathbf{A}_{2}\|_{F}=\sqrt{\sum_{ij}\Big(\mathbf{A}_{1}(i,j)\mathbf{A}_{2}(i,j)\Big)^2}\\
        \leq&\sqrt{\sum_{ij}\Big(\alpha\mathbf{A}_{1}(i,j)\Big)^2}
        \leq \alpha\sqrt{\sum_{ij}\Big(\mathbf{A}_{1}(i,j)\Big)^2}=\alpha\|\mathbf{A}_{1}\|_{F}.
      \end{align*}
 \end{itemize}

 \section*{APPENDIX A: Gradient computation for SGD} \label{Sec:appendix_gradient}
 To work with the SGD algorithm, we provide a recursive formula for the gradient of the final output $f(\mathbf{x}|\theta)$ at node $\mathbf{x}$ in the GCNs model (\ref{equ:GCN_model}) with respect to the learnable parameters.
\begin{itemize}
  \item For the final layer,
  \begin{equation}\label{equ:Gradient_w}
  \nabla_{\mathbf{w}}f(\mathbf{x}|\theta)=
  \nabla\sigma\big(\boldsymbol{\delta}_{\mathbf{x}}^{\top}g(\mathbf{L})\mathbf{X}^{(K)}\mathbf{w}\big)
  \big[\boldsymbol{\delta}^{\top}_{\mathbf{x}}g(\mathbf{L})\mathbf{X}^{(K)}\big]^{\top},
  \end{equation}
  \item For the hidden layer $k=1,2,\dots,K$,
  \begin{equation}\label{equ:Gradient_W}
     \nabla_{\mathbf{W}^{(k)}}f(\mathbf{x}|\theta)
     =\big[g(\mathbf{L})\mathbf{X}^{(k-1)}\big]^{\top}\Big(\frac{\partial
f(\mathbf{x}|\theta)}{\partial\mathbf{X}^{(k)}}\odot\mathbf{R}^{(k)}\Big),
  \end{equation}
  where $\mathbf{R}^{(k)}:=\nabla\sigma\big(g(\mathbf{L})\mathbf{X}^{(k-1)}\mathbf{W}^{(k)}\big)$ and
  \begin{equation}\label{equ:Gradient_X}
    \frac{\partial f(\mathbf{x}|\theta)}{\partial\mathbf{X}^{(k-1)}}=g(\mathbf{L})^{\top}\Big(\frac{\partial
f(\mathbf{x}|\theta)}{\partial\mathbf{X}^{(k)}}\odot\mathbf{R}^{(k)}\Big)\big[\mathbf{W}^{(k)}\big]^{\top},
  \end{equation}
  with
  \begin{equation}\label{equ:Gradient_XK}
   \frac{\partial f(\mathbf{x}|\theta)}{\partial\mathbf{X}^{(K)}} =\nabla\sigma\big(\boldsymbol{\delta}_{\mathbf{x}}^{\top}g(\mathbf{L})\mathbf{X}^{(K)}\mathbf{w}\big)
  \big[\boldsymbol{\delta}^{\top}_{\mathbf{x}}g(\mathbf{L})\big]^{\top}\mathbf{w}^{\top}.
  \end{equation}
\end{itemize}
The notation $\odot$ represents the Hadamard product of two matrices. \eqref{equ:Gradient_w} and \eqref{equ:Gradient_XK} are easy to verify, while \eqref{equ:Gradient_W} and \eqref{equ:Gradient_X} are not. In the following, a detailed procedure is provided to derive \eqref{equ:Gradient_W} and \eqref{equ:Gradient_X}.

First, since
   $\mathbf{X}^{(k)}_{ij}=\sigma\big(\boldsymbol{\delta}_{i}^{\top}g(\mathbf{L})\mathbf{X}^{(k-1)}\mathbf{W}^{(k)}\boldsymbol{\delta}_{j}\big)$,
    \begin{align*}
     \frac{\partial\mathbf{X}^{(k)}_{ij}}{\partial\mathbf{W}^{(k)}}
  & = \frac{\partial \sigma\big(\boldsymbol{\delta}_{i}^{\top}g(\mathbf{L})\mathbf{X}^{(k-1)}\mathbf{W}^{(k)}\boldsymbol{\delta}_{j}\big)}{\partial\mathbf{W}^{(k)}} \\  &=\nabla\sigma\big(\boldsymbol{\delta}_{i}^{\top}g(\mathbf{L})\mathbf{X}^{(k-1)}\mathbf{W}^{(k)}\boldsymbol{\delta}_{j}\big)
\frac{\partial\big\{\boldsymbol{\delta}_{i}^{\top}g(\mathbf{L})\mathbf{X}^{(k-1)}\mathbf{W}^{(k)}\boldsymbol{\delta}_{j}\big\}}{\partial\mathbf{W}^{(k)}}\\
  & =\nabla\sigma\big(\boldsymbol{\delta}_{i}^{\top}g(\mathbf{L})\mathbf{X}^{(k-1)}\mathbf{W}^{(k)}\boldsymbol{\delta}_{j}\big)
  \big[g(\mathbf{L})\mathbf{X}^{(k-1)}\big]^{\top}\boldsymbol{\delta}_{i}\boldsymbol{\delta}_{j}^{\top},
    \end{align*}
and

 \begin{align*}
      \frac{\partial\mathbf{X}^{(k)}_{ij}}{\partial\mathbf{X}^{(k-1)}}
  & =\frac{\partial\sigma\big(\boldsymbol{\delta}_{i}^{\top}g(\mathbf{L})\mathbf{X}^{(k-1)}\mathbf{W}^{(k)}\boldsymbol{\delta}_{j}\big)}
  {\partial\mathbf{X}^{(k-1)}} 
=\nabla\sigma\big(\boldsymbol{\delta}_{i}^{\top}g(\mathbf{L})\mathbf{X}^{(k-1)}\mathbf{W}^{(k)}\boldsymbol{\delta}_{j}\big)
  g(\mathbf{L})^{\top}\boldsymbol{\delta}_{i}\boldsymbol{\delta}_{j}^{\top}\big[\mathbf{W}^{(k)}\big]^{\top}.
\end{align*}
Let $\mathbf{R}^{(k)}=\nabla\sigma\big(g(\mathbf{L})\mathbf{X}^{(k-1)}\mathbf{W}^{(k)}\big)$. Then,
\begin{align*}
   \frac{\partial f(\mathbf{x}|\theta)}{\partial\mathbf{W}^{(k)}}
&=\sum_{i,j}\frac{\partial
f(\mathbf{x}|\theta)}{\partial\mathbf{X}^{(k)}_{ij}}\cdot\frac{\partial\mathbf{X}^{(k)}_{ij}}{\partial\mathbf{W}^{(k)}}
      =\sum_{i,j}\frac{\partial
f(\mathbf{x}|\theta)}{\partial\mathbf{X}^{(k)}}(i,j)\cdot\frac{\partial\mathbf{X}^{(k)}_{ij}}{\partial\mathbf{W}^{(k)}}\\
 &=\sum_{i,j}\frac{\partial
f(\mathbf{x}|\theta)}{\partial\mathbf{X}^{(k)}}(i,j)\cdot\mathbf{R}^{(k)}(i,j)
\big[g(\mathbf{L})\mathbf{X}^{(k-1)}\big]^{\top}\boldsymbol{\delta}_{i}\boldsymbol{\delta}_{j}^{\top}\\
 &=\big[g(\mathbf{L})\mathbf{X}^{(k-1)}\big]^{\top}\sum_{i,j}\frac{\partial
f(\mathbf{x}|\theta)}{\partial\mathbf{X}^{(k)}}(i,j)\cdot\mathbf{R}^{(k)}(i,j)\boldsymbol{\delta}_{i}\boldsymbol{\delta}_{j}^{\top}\\
 &=\big[g(\mathbf{L})\mathbf{X}^{(k-1)}\big]^{\top}\Big(\frac{\partial
f(\mathbf{x}|\theta)}{\partial\mathbf{X}^{(k)}}\odot\mathbf{R}^{(k)}\Big),
    \end{align*}
and
    \begin{align*}
   \frac{\partial f(\mathbf{x}|\theta)}{\partial\mathbf{X}^{(k-1)}}
&=\sum_{i,j}\frac{\partial
f(\mathbf{x}|\theta)}{\partial\mathbf{X}^{(k)}_{ij}}\cdot\frac{\partial\mathbf{X}^{(k)}_{ij}}{\partial\mathbf{X}^{(k-1)}}\\
  &=g(\mathbf{L})^{\top}\Big(\sum_{i,j}\frac{\partial
f(\mathbf{x}|\theta)}{\partial\mathbf{X}^{(k)}}(i,j)\cdot\mathbf{R}^{(k)}(i,j)\boldsymbol{\delta}_{i}\boldsymbol{\delta}_{j}^{\top}\Big)
\big[\mathbf{W}^{(k)}\big]^{\top}\\
  &=g(\mathbf{L})^{\top}\Big(\frac{\partial
f(\mathbf{x}|\theta)}{\partial\mathbf{X}^{(k)}}\odot\mathbf{R}^{(k)}\Big)\big[\mathbf{W}^{(k)}\big]^{\top}.
    \end{align*}
This completes the derivation of \eqref{equ:Gradient_W} and \eqref{equ:Gradient_X}.

Based on the above recursive formula, we prove the following lemma recursively.
\begin{lemma}\label{lem:bound_GCN}
  Let the assumptions made in Section \ref{Sec:Assumption} hold. Then, we have the following results for the GCNs model (\textcolor{red}{1}) during the training procedure.
  \begin{itemize}
  \item Hidden layer output $\mathbf{X}^{(k)}(k=1,2\dots,K)$ satisfies
  \begin{equation}\label{equ:bound_X}
    \|\mathbf{X}^{(k)}\|_{F}\leq B^{k}\alpha^{k}_{\sigma}C^{k}_{g}C_{\mathbf{X}}.
    \end{equation}
    \item The gradient of $f$ with respect to $\mathbf{X}^{(k)}~(k=1,2,\dots,K)$ satisfies
    \begin{equation}\label{equ:bound_df_dX}
  \|\frac{\partial f(\mathbf{x}|\theta)}{\partial\mathbf{X}^{(k)}}\|_{F}\leq B^{K+1-k}\alpha^{K+1-k}_{\sigma}C^{K+1-k}_{g}.
\end{equation}
    \item The gradient of $f$ with respect to $\mathbf{W}^{(k)}~(k=1,\dots,K+1)$ satisfies
      \begin{equation}\label{equ:bound_df_dW}
    \big\|\nabla_{\mathbf{W}^{(k)}}f(\mathbf{x}|\theta)\big\|_{F}\leq B^{K}\alpha_{\sigma}^{K+1}C^{K+1}_{g}C_{\mathbf{X}},
  \end{equation}
  where $\mathbf{W}^{(K+1)}:=\mathbf{w}$.
  \end{itemize}
\end{lemma}

{\emph{Proof}}.
Now, we give a complete proof for Lemma \ref{lem:bound_GCN}.
  \begin{itemize}
     \item Firstly, for $k=1,2,\dots,K$, since $\|\sigma(\mathbf{Z})\|_{F}\leq\alpha_{\sigma}\|\mathbf{Z}\|_{F}$ holds for any matrix $\mathbf{Z}$, we have $$\|\mathbf{X}^{(k)}\|_{F} =\|\sigma\big(g(\mathbf{L})\mathbf{X}^{(k-1)}\mathbf{W}^{(k)}\big)\|_{F}
    \leq\alpha_{\sigma}\|g(\mathbf{L})\mathbf{X}^{(k-1)}\mathbf{W}^{(k)}\|_{F}.$$
    Then, by applying the inequality $\|\mathbf{A}_{1}\mathbf{A}_{2}\|_{F}\leq\|\mathbf{A}_{1}\|_{2}\|\mathbf{A}_{2}\|_{F}$ twice, we obtain 
     $\|\mathbf{X}^{(k)}\|_{F} \leq B\alpha_{\sigma}C_{g}\|\mathbf{X}^{(k-1)}\|_{F}.$
     Note that $\|\mathbf{X}^{(1)}\|_{F} \leq B\alpha_{\sigma}C_{g}\|\mathbf{X}^{(0)}\|_{F}=B\alpha_{\sigma}C_{g}C_{\mathbf{X}}$, it further yields that $$
    \|\mathbf{X}^{(k)}\|_{F}\leq B^{k}\alpha^{k}_{\sigma}C^{k}_{g}C_{\mathbf{X}},~~~k=1,2,\dots,K,
    $$
    which completes the proof of \eqref{equ:bound_X}.
    
\item To show \eqref{equ:bound_df_dX}, note that for $k=1,2,\dots,K-1$, by applying $\|\mathbf{A}_{1}\mathbf{A}_{2}\|_{F}\leq\|\mathbf{A}_{1}\|_{2}\|\mathbf{A}_{2}\|_{F}$ twice, we obtain
    \begin{align*}
  \Big\|\frac{\partial
f(\mathbf{x}|\theta)}{\partial\mathbf{X}^{(k)}}\Big\|_{F}
 = \Big\|g(\mathbf{L})^{\top}\Big(\frac{\partial
f(\mathbf{x}|\theta)}{\partial\mathbf{X}^{(k+1)}}\odot\mathbf{R}^{(k+1)}\Big)\big[\mathbf{W}^{(k+1)}\big]^{\top}\Big\|_{F}
\leq\|g(\mathbf{L})\|_{2}\Big\|\Big(\frac{\partial
f(\mathbf{x}|\theta)}{\partial\mathbf{X}^{(k+1)}}\odot\mathbf{R}^{(k+1)}\Big)\Big\|_{F}\|\mathbf{W}^{(k+1)}\|_{2}.
  \end{align*}
  Since $C_g=\|g(\mathbf{L})\|_{2}$, $\|\mathbf{W}^{(k+1)}\|_{2}\leq B$ and the absolute value of the elements in $\mathbf{R}^{(k+1)}$ is less than $\alpha_{\sigma}$, we further have
  $\big\|\frac{\partial
f(\mathbf{x}|\theta)}{\partial\mathbf{X}^{(k)}}\big\|_{F}\leq B\alpha_{\sigma}C_{g}\|\frac{\partial
f(\mathbf{x}|\theta)}{\partial\mathbf{X}^{(k+1)}}\|_{F}$. Meanwhile, since $|\nabla\sigma\big(\boldsymbol{\delta}_{\mathbf{x}}^{\top}g(\mathbf{L})\mathbf{X}^{(K)}\mathbf{w}\big)|\leq\alpha_{\sigma}$, 
$$\big\|\frac{\partial f(\mathbf{x}|\theta)}{\partial\mathbf{X}^{(K)}}\big\|_{F}
   =\big\|\nabla\sigma\big(\boldsymbol{\delta}_{\mathbf{x}}^{\top}g(\mathbf{L})\mathbf{X}^{(K)}\mathbf{w}\big)
  \big[\boldsymbol{\delta}^{\top}_{\mathbf{x}}g(\mathbf{L})\big]^{\top}\mathbf{w}\big\|_{F}
            \leq B\alpha_{\sigma}C_{g}.$$
  Therefore, for $k=1,2,\dots,K$,
 $$
  \|\frac{\partial f(\mathbf{x}|\theta)}{\partial\mathbf{X}^{(k)}}\|_{F}\leq B^{K+1-k}\alpha^{K+1-k}_{\sigma}C^{K+1-k}_{g}.
$$
This completes the proof of \eqref{equ:bound_df_dX}.

   \item Now, let's prove \eqref{equ:bound_df_dW}. Firstly, note that $|\nabla\sigma\big(\boldsymbol{\delta}_{\mathbf{x}}^{\top}g(\mathbf{L})\mathbf{X}^{(K)}\mathbf{w}\big)|\leq\alpha_{\sigma}$, so $$\big\|\nabla_{\mathbf{w}}f(\mathbf{x}|\theta)\big\|_{F}
   =\big\|\nabla\sigma\big(\boldsymbol{\delta}_{\mathbf{x}}^{\top}g(\mathbf{L})\mathbf{X}^{(K)}\mathbf{w}\big)
  \big[g(\mathbf{L})\mathbf{X}^{(K)}\big]^{\top}\boldsymbol{\delta}_{\mathbf{x}}\big\|_{F}
  \leq\alpha_{\sigma}\|\mathbf{X}^{(K)}\|_{F}\|\boldsymbol{\delta}_{\mathbf{x}}^{\top}g(\mathbf{L})\|_{2}.$$
  Combining \eqref{equ:bound_X} and $\|\boldsymbol{\delta}_{\mathbf{x}}^{\top}g(\mathbf{L})\|_{2}\leq C_g$, we have
  $$\big\|\nabla_{\mathbf{w}}f(\mathbf{x}|\theta)\big\|_{F}\leq B^{K}\alpha_{\sigma}^{K+1}C^{K+1}_{g}C_{\mathbf{X}}.$$
  Furthermore, for $k=1,2,\dots,K$, by applying $\|\mathbf{A}_{1}\mathbf{A}_{2}\|_{F}\leq\|\mathbf{A}_{1}\|_{2}\|\mathbf{A}_{2}\|_{F}$ twice, it yields
  $$\|\nabla_{\mathbf{W}^{(k)}}f(\mathbf{x}|\theta)\|_{F}
   =\big\|\big[g(\mathbf{L})\mathbf{X}^{(k-1)}\big]^{\top}\Big(\frac{\partial
f(\mathbf{x}|\theta)}{\partial\mathbf{X}^{(k)}}\odot\mathbf{R}^{(k)}\Big)\big\|_{F}
 \leq\big\|g(\mathbf{L})\big\|_{2}\big\|\mathbf{X}^{(k-1)}\big\|_{F}\big\|\frac{\partial
f(\mathbf{x}|\theta)}{\partial\mathbf{X}^{(k)}}\odot\mathbf{R}^{(k)}\big\|_{F}.$$ 
Since the absolute value of the elements in $\mathbf{R}^{(k)}$ is less than $\alpha_{\sigma}$, we have
$$\|\nabla_{\mathbf{W}^{(k)}}f(\mathbf{x}|\theta)\|_{F}\leq \alpha_{\sigma}C_g\big\|\mathbf{X}^{(k-1)}\big\|_{F}\big\|\frac{\partial
f(\mathbf{x}|\theta)}{\partial\mathbf{X}^{(k)}}\big\|_{F}\leq B^{K}\alpha^{K+1}_{\sigma}C^{K+1}_{g}C_{\mathbf{X}},$$
which holds by combining \eqref{equ:bound_X} and \eqref{equ:bound_df_dX}. This completes the proof of \eqref{equ:bound_df_dW}.
  \end{itemize}

\section*{APPENDIX B:Proof of Lemma \ref{lemma:Step_1}} \label{proof:lem:step_1}
To prove Lemma \ref{lemma:Step_1}, we first provide the following lemma to show the variation of output in each layer for two GCNs with different learned parameters $\theta=\{\mathbf{W}^{(1)},\mathbf{W}^{(2)},\dots,\mathbf{W}^{(K)},\mathbf{w}\}$ and $\theta'=\{\mathbf{W}^{(1)'},\mathbf{W}^{(2)'},\dots,\mathbf{W}^{(K)'},\mathbf{w}'\}$. Let $\mathbf{X}^{(k)}$ and $\mathbf{X}^{(k)'}$ be their output of the hidden layer, as well as $f(\mathbf{x}|\theta)$ and $f(\mathbf{x}|\theta')$ the final output of node $\mathbf{x}$. The following lemma provides a bound of $\mathbf{X}^{(k)}-\mathbf{X}^{(k)'}$ and $f(\mathbf{x}|\theta)-f(\mathbf{x}|\theta')$ based on $\triangle\theta=\{\triangle\mathbf{W}^{(1)},\dots,\triangle\mathbf{W}^{(K)},\triangle\mathbf{w}\}$.

\begin{lemma}\label{lem:bound_variation}
  Consider two GCNs with parameters $\theta$ and $\theta'$, respectively. Then, we obtain the following results for their variations.
  \begin{itemize}
    \item Their variation of outputs in hidden layers  $\triangle\mathbf{X}^{(k)}:=\mathbf{X}^{(k)}-\mathbf{X}^{(k)'}~(k=1,2,\dots,K)$ satisfies
    \begin{equation}\label{equ:bound_variation_X}
  \|\triangle\mathbf{X}^{(k)}\|_{F}\leq B^{k-1}\alpha^{k}_{\sigma}C^{k}_{g}C_{\mathbf{X}}
  \Big(\sum_{j=1}^{k}\|\triangle\mathbf{W}^{(j)}\|_{2}\Big).
  \end{equation}
  \item Furthermore, for the final output of node $\mathbf{x}$,
  \begin{equation}\label{equ:bound_variation_f}
       |f(\mathbf{x}|\theta)-f(\mathbf{x}|\theta')|\leq B^{K}\alpha^{K+1}_{\sigma}C^{K+1}_{g}C_{\mathbf{X}}\|\triangle\theta\|_{*}.
      \end{equation}
  \end{itemize}
\end{lemma}

{\emph{Proof}}:
To prove \eqref{equ:bound_variation_X}, we first have that for $k=1,2,\dots,K$,
 $$\|\triangle\mathbf{X}^{(k)}\|_{F} =\|\mathbf{X}^{(k)}-\mathbf{X}^{(k)'}\|_{F}
 =\|\sigma\big(g(\mathbf{L})\mathbf{X}^{(k-1)}\mathbf{W}^{(k)}\big)
     -\sigma\big(g(\mathbf{L})\mathbf{X}^{(k-1)'}\mathbf{W}^{(k)'}\big)\|_{F}.$$
Since $\|\sigma(\mathbf{Z})\|_{F}\leq\alpha_{\sigma}\|\mathbf{Z}\|_{F}$ holds for any matrix $\mathbf{Z}$,  we have
$$\|\triangle\mathbf{X}^{(k)}\|_{F}\leq\alpha_{\sigma}\|g(\mathbf{L})\big(\mathbf{X}^{(k-1)}\mathbf{W}^{(k)}
     -\mathbf{X}^{(k-1)'}\mathbf{W}^{(k)'}\big)\|_{F}\leq\alpha_{\sigma}\|g(\mathbf{L})\|_{2}\cdot\|\mathbf{X}^{(k-1)}\mathbf{W}^{(k)}
     -\mathbf{X}^{(k-1)'}\mathbf{W}^{(k)'}\|_{F}.$$
Note that
\begin{align*}
\|\mathbf{X}^{(k-1)}\mathbf{W}^{(k)}
     -\mathbf{X}^{(k-1)'}\mathbf{W}^{(k)'}\|_{F}&\leq\|\mathbf{X}^{(k-1)}\|_{F}\|\mathbf{W}^{(k)}-\mathbf{W}^{(k)'}\|_{2}
     +\|\mathbf{X}^{(k-1)}-\mathbf{X}^{(k-1)'}\|_{F}\|\mathbf{W}^{(k)'}\|_{2}\\
    &=\|\mathbf{X}^{(k-1)}\|_{F}\|\triangle\mathbf{W}^{(k)}\|_{2}
     +\|\triangle\mathbf{X}^{(k-1)}\|_{F}\|\mathbf{W}^{(k)'}\|_{2}.
\end{align*}
Then, combining \eqref{equ:bound_X} and $\|\mathbf{W}^{(k)'}\|_{2}\leq B$, we obtain $$\|\mathbf{X}^{(k-1)}\mathbf{W}^{(k)}
     -\mathbf{X}^{(k-1)'}\mathbf{W}^{(k)'}\|_{F}\leq B^{k-1}\alpha^{k-1}_{\sigma}C^{k-1}_{g}C_{\mathbf{X}}\|\triangle\mathbf{W}^{(k)}\|_{2}
     +B\|\triangle\mathbf{X}^{(k-1)}\|_{F}.$$
Thus,  
$$\|\triangle\mathbf{X}^{(k)}\|_{F}\leq\alpha_{\sigma}\|g(\mathbf{L})\|_{2}\cdot\|\mathbf{X}^{(k-1)}\mathbf{W}^{(k)}
     -\mathbf{X}^{(k-1)'}\mathbf{W}^{(k)'}\|_{F}\leq B^{k-1}\alpha^{k}_{\sigma}C^{k}_{g}C_{\mathbf{X}}\|\triangle\mathbf{W}^{(k)}\|_{2}
     +B\alpha_{\sigma}C_{g}\|\triangle\mathbf{X}^{(k-1)}\|_{F}.$$
Then, since $\|\triangle\mathbf{X}^{(1)}\|_{F}\leq\alpha_{\sigma}C_{g}C_{\mathbf{X}}\|\triangle\mathbf{W}^{(1)}\|_{2}$, we have 
      $$
  \|\triangle\mathbf{X}^{(k)}\|_{F}\leq B^{k-1}\alpha^{k}_{\sigma}C^{k}_{g}C_{\mathbf{X}}
  \Big(\sum_{j=1}^{k}\|\triangle\mathbf{W}^{(j)}\|_{2}\Big),
$$ holds for any $k=1,2,\dots,K$.
This completely proves \eqref{equ:bound_variation_X}.\\
Furthermore, for the final output, using the Lipschitz property of $\sigma(\cdot)$, we have
$$ |f(\mathbf{x}|\theta)-f(\mathbf{x}|\theta')|
      =|\sigma\big(\boldsymbol{\delta}_{\mathbf{x}}^{\top}g(\mathbf{L})\mathbf{X}^{(K)}\mathbf{w}\big)
      -\sigma\big(\boldsymbol{\delta}_{\mathbf{x}}^{\top}g(\mathbf{L})\mathbf{X}^{(K)'}\mathbf{w}'\big)| 
  \leq \alpha_{\sigma}|\boldsymbol{\delta}_{\mathbf{x}}^{\top}g(\mathbf{L})\big(
      \mathbf{X}^{(K)}\mathbf{w}-\mathbf{X}^{(K)'}\mathbf{w}'\big)|.$$
Note that 
$$|\boldsymbol{\delta}_{\mathbf{x}}^{\top}g(\mathbf{L})\big(
      \mathbf{X}^{(K)}\mathbf{w}-\mathbf{X}^{(K)'}\mathbf{w}'\big)|
\leq\|\boldsymbol{\delta}_{\mathbf{x}}^{\top}g(\mathbf{L})\|_{2}\cdot\|
      \mathbf{X}^{(K)}\mathbf{w}-\mathbf{X}^{(K)'}\mathbf{w}'\|_{F}\leq C_{g}
      \big(\|\mathbf{X}^{(K)}\|_{F}\|\triangle\mathbf{w}\|_{2}+\|\triangle\mathbf{X}^{(K)}\|_{F}\|\mathbf{w}'\|_{2}\big).$$
Combining \eqref{equ:bound_X} and \eqref{equ:bound_variation_X}, we further have
$$|\boldsymbol{\delta}_{\mathbf{x}}^{\top}g(\mathbf{L})\big(
      \mathbf{X}^{(K)}\mathbf{w}-\mathbf{X}^{(K)'}\mathbf{w}'\big)|
\leq B^{K}\alpha_{\sigma}^{K}C_{g}^{K+1}C_\mathbf{X}(\|\triangle\mathbf{w}\|_{2}+\sum_{j=1}^{K}\|\triangle\mathbf{W}^{(j)}\|_{2})
= B^{K}\alpha_{\sigma}^{K}C_{g}^{K+1}C_\mathbf{X}\|\triangle\theta\|_{*}.$$
Thus, 
$$|f(\mathbf{x}|\theta)-f(\mathbf{x}|\theta')|\leq \alpha_{\sigma}|\boldsymbol{\delta}_{\mathbf{x}}^{\top}g(\mathbf{L})\big(
      \mathbf{X}^{(K)}\mathbf{w}-\mathbf{X}^{(K)'}\mathbf{w}'\big)|
\leq B^{K}\alpha_{\sigma}^{K+}C_{g}^{K+1}C_\mathbf{X}\|\triangle\theta\|_{*},$$
which completes the proof of \eqref{equ:bound_variation_f}.

{\emph{Proof of Lemma \ref{lemma:Step_1}}}: Now, we are ready to prove Lemma \ref{lemma:Step_1} based on Lemma \ref{lem:bound_variation}. For any $\mathbf{z}=(\mathbf{x},y)$ taken from $\mathcal{D}$, we denote by $\hat{y}=f(\mathbf{x}|\theta_{T})$ and $\hat{y}'=f(\mathbf{x}|\theta'_{T})$. Firstly,
using the Lipschitz property of loss function $\ell(\cdot,\cdot)$, we have $$\sup_{\mathcal{S},\mathbf{z}}\Big|\mathbb{E}_{\mathcal{A}}\big[\ell(\hat{y},y)\big]
  -\mathbb{E}_{\mathcal{A}}\big[\ell(\hat{y}',y)\big]\Big|
 = \sup_{\mathcal{S},z}\Big|\mathbb{E}_{\mathcal{A}}\big[\ell\big(f(\mathbf{x}|\theta_{T}),y\big)-\ell\big(f(\mathbf{x}|\theta'_{T}),y\big)\big]\Big| \leq\alpha_{\ell}\sup_{\mathbf{x}}\mathbb{E}_{\mathcal{A}}\Big[\big|f(\mathbf{x}|\theta_{T})-f(\mathbf{x}|\theta'_{T})\big|\Big]$$
Then, according to \eqref{equ:bound_variation_f},
\begin{align*}
\sup_{\mathcal{S},\mathbf{z}}\Big|\mathbb{E}_{\mathcal{A}}\big[\ell(\hat{y},y)\big]
  -\mathbb{E}_{\mathcal{A}}\big[\ell(\hat{y}',y)\big]\Big|
\leq \alpha_{\ell}B^{K}\alpha_{\sigma}^{K+1}C_{g}^{K+1}C_{\mathbf{X}}\cdot\mathbb{E}_{\mathcal{A}}\big[\|\triangle\theta_{T}\|_{*}\big].
 \end{align*}
 This completes the proof of Lemma \ref{lemma:Step_1}.

\section*{APPENDIX C: Proof of Lemma \ref{lem:Same_sample} and Lemma \ref{lem:diff_sample}} \label{proof:lem:same_and_diff_sample}

To prove Lemma \ref{lem:Same_sample} and Lemma \ref{lem:diff_sample}, we should first prove the following lemma.
\begin{lemma}\label{lem:bound_variation_df_dW}
  Consider two GCNs with parameters $\theta$ and $\theta'$, respectively. Then, their variation of gradients of $f$ with respect to $\{\mathbf{W}^{(1)},\dots,\mathbf{W}^{(K)},\mathbf{w}\}$ satisfies
  \begin{align}
  \big\|\nabla_{\mathbf{w}}f(\mathbf{x}|\theta)-\nabla_{\mathbf{w}}f(\mathbf{x}|\theta')\big\|_{F}
 \leq \Big(\upsilon_{\sigma}B^{2K}\alpha^{2K}_{\sigma}C^{2K+2}_{g}C^{2}_{\mathbf{X}}
     +B^{K-1}\alpha^{K+1}_{\sigma}C^{K+1}_{g}C_{\mathbf{X}}\Big)\|\triangle\theta\|_{*}, \label{equ:bound_variation_df_dw}
  \end{align}
  and for $k=1,2,\dots,K$,
  \begin{align}
 \big\|\nabla_{\mathbf{W}^{(k)}}f(\mathbf{x}|\theta)-\nabla_{\mathbf{W}^{(k)}}f(\mathbf{x}|\theta')\big\|_{F} 
    \leq 
\Big(\nu_{\sigma}B^{2K}\alpha^{2K}_{\sigma}C^{2K+2}_{g}C^{2}_{\mathbf{X}}
  +B^{K-1}\alpha_{\sigma}^{K+1}C_{g}^{K+1}C_{\mathbf{X}}\Big)\|\triangle\theta\|_{*}
    +\rho_{k}\|\triangle\theta\|_{*},  \label{equ:bound_variation_df_dW}
  \end{align}
  where
  \begin{equation}\label{equ:rho_k}
    \rho_{k}:=\nu_{\sigma}(B\alpha_{\sigma}C_{g})^{K+k-1}C_{g}^{2}C^{2}_{\mathbf{X}}
\Big(\sum_{j=0}^{K-k}(B\alpha_{\sigma}C_{g})^{j}\Big).
  \end{equation}
\end{lemma}

{\emph{Proof}}.
  First, according to the proof of \eqref{equ:bound_variation_X} and \eqref{equ:bound_variation_f}, the following holds true for $k=1,2,\dots,K+1$:
    \begin{align}
\|\mathbf{X}^{(k-1)}\mathbf{W}^{(k)}-\mathbf{X}^{(k-1)'}\mathbf{W}^{(k)'}\|_{F}
\nonumber   \leq & B^{k-1}\alpha^{k-1}_{\sigma}C^{k-1}_{g}C_{\mathbf{X}}\|\triangle\mathbf{W}^{(k)}\|_{2}
     +B\|\triangle\mathbf{X}^{(k-1)}\|_{F}\\
\leq & B^{k-1}\alpha^{k-1}_{\sigma}C^{k-1}_{g}C_{\mathbf{X}}
  \Big(\sum_{j=1}^{k}\|\triangle\mathbf{W}^{(j)}\|_{2}\Big), \label{equ:bound_variation_X_W}
\end{align}
where $\mathbf{W}^{(K+1)}=\mathbf{w}$. 

We now prove \eqref{equ:bound_variation_df_dw}. First, applying $\mathbf{A}_{1}\mathbf{A}_{2}-\mathbf{A}'_{1}\mathbf{A}'_{2}=
       (\mathbf{A}_{1}-\mathbf{A}_{1}')\mathbf{A}_{2}
       +\mathbf{A}'_{1}(\mathbf{A}_{2}-\mathbf{A}_{2}')$, we have
       \begin{align*}
         \|\nabla_{\mathbf{w}}f(\mathbf{x}|\theta)-\nabla_{\mathbf{w}}f(\mathbf{x}|\theta')\|_{F}
      = &\Big\|\nabla\sigma\big(\boldsymbol{\delta}_{\mathbf{x}}^{\top}g(\mathbf{L})\mathbf{X}^{(K)}\mathbf{w}\big)
   [g(\mathbf{L})\mathbf{X}^{(K)}]^{\top}\boldsymbol{\delta}_{\mathbf{x}}
      -\nabla\sigma\big(\boldsymbol{\delta}_{\mathbf{x}}^{\top}g(\mathbf{L})\mathbf{X}^{(K)'}\mathbf{w}'\big)
     [g(\mathbf{L})\mathbf{X}^{(K)'}]^{\top}\boldsymbol{\delta}_{\mathbf{x}}\Big\|_{F} \\
       \leq & \Big\|\Big(\nabla\sigma\big(\boldsymbol{\delta}_{\mathbf{x}}^{\top}g(\mathbf{L})\mathbf{X}^{(K)}\mathbf{w}\big)
           -\nabla\sigma\big(\boldsymbol{\delta}_{\mathbf{x}}^{\top}g(\mathbf{L})\mathbf{X}^{(K)'}\mathbf{w}'\big)\Big)
     [g(\mathbf{L})\mathbf{X}^{(K)}]^{\top}\boldsymbol{\delta}_{\mathbf{x}}\Big\|_{F} \\
       &  +\Big\|\nabla\sigma\big(\boldsymbol{\delta}_{\mathbf{x}}^{\top}g(\mathbf{L})\mathbf{X}^{(K)'}\mathbf{w}'\big)
     [g(\mathbf{L})\triangle\mathbf{X}^{(K)}]^{\top}\boldsymbol{\delta}_{\mathbf{x}}\Big\|_{F}.
     \end{align*}
Using the $\nu_{\sigma}$-smooth property of $\sigma(\cdot)$ and applying $\|\mathbf{A}_{1}\mathbf{A}_{2}\|_{F}\leq\|\mathbf{A}_{1}\|_{2}\|\mathbf{A}_{2}\|_{F}$, we have
\begin{align*}
&\Big\|\Big(\nabla\sigma\big(\boldsymbol{\delta}_{\mathbf{x}}^{\top}g(\mathbf{L})\mathbf{X}^{(K)}\mathbf{w}\big)
           -\nabla\sigma\big(\boldsymbol{\delta}_{\mathbf{x}}^{\top}g(\mathbf{L})\mathbf{X}^{(K)'}\mathbf{w}'\big)\Big)
     [g(\mathbf{L})\mathbf{X}^{(K)}]^{\top}\boldsymbol{\delta}_{\mathbf{x}}\Big\|_{F}\\
\leq & \Big|\nabla\sigma\big(\boldsymbol{\delta}_{\mathbf{x}}^{\top}g(\mathbf{L})\mathbf{X}^{(K)}\mathbf{w}\big)
           -\nabla\sigma\big(\boldsymbol{\delta}_{\mathbf{x}}^{\top}g(\mathbf{L})\mathbf{X}^{(K)'}\mathbf{w}'\big)\Big|\cdot
     \Big\|[g(\mathbf{L})\mathbf{X}^{(K)}]^{\top}\boldsymbol{\delta}_{\mathbf{x}}\Big\|_{F}\\
 \leq & \upsilon_{\sigma}|\boldsymbol{\delta}_{\mathbf{x}}^{\top}g(\mathbf{L})\mathbf{X}^{(K)}\mathbf{w}
      -\boldsymbol{\delta}_{\mathbf{x}}^{\top}g(\mathbf{L})\mathbf{X}^{(K)'}\mathbf{w}'|\cdot
    \|\mathbf{X}^{(K)}\|_{F}\|\boldsymbol{\delta}_{\mathbf{x}}^{\top}g(\mathbf{L})\|_{2} \\
    \leq & \upsilon_{\sigma}C_{g}\|\mathbf{X}^{(K)}\mathbf{w}
      -\mathbf{X}^{(K)'}\mathbf{w}'\|_{F}\cdot
     \|\mathbf{X}^{(K)}\|_{F}\cdot C_{g},
\end{align*}
and since $|\nabla\sigma(\cdot)|\leq\alpha_{\sigma}$, $\Big\|\nabla\sigma\big(\boldsymbol{\delta}_{\mathbf{x}}^{\top}g(\mathbf{L})\mathbf{X}^{(K)'}\mathbf{w}'\big)
     [g(\mathbf{L})\triangle\mathbf{X}^{(K)}]^{\top}\boldsymbol{\delta}_{\mathbf{x}}\Big\|_{F}
     \leq\alpha_{\sigma}C_{g}\|\triangle\mathbf{X}^{(K)}\|_{F}$. Then, combining \eqref{equ:bound_X}, \eqref{equ:bound_variation_X} and \eqref{equ:bound_variation_X_W}, we have
     \begin{align*}
     &\|\nabla_{\mathbf{w}}f(\mathbf{x}|\theta)-\nabla_{\mathbf{w}}f(\mathbf{x}|\theta')\|_{F}
\leq \Big(\upsilon_{\sigma}
   B^{2K}\alpha^{2K}_{\sigma}C^{2K+2}_{g}C^{2}_{\mathbf{X}}
   +B^{K-1}\alpha^{K+1}_{\sigma}C^{K+1}_{g}C_{\mathbf{X}}\Big)\|\triangle\theta\|_{*},
   \end{align*}
   which completes the proof of \eqref{equ:bound_variation_df_dw}.
   
   Next, we turn to prove \eqref{equ:bound_variation_df_dW}. First, for $k=1,2,\dots,K$,
\begin{align}
      \nonumber  &\big\|\nabla_{\mathbf{W}^{(k)}}f(\mathbf{x}|\theta)-\nabla_{\mathbf{W}^{(k)}}f(\mathbf{x}|\theta')\big\|_{F}\\
    \nonumber=&\big\|\big[g(\mathbf{L})\mathbf{X}^{(k-1)}\big]^{\top}\Big(\frac{\partial
f(\mathbf{x}|\theta)}{\partial\mathbf{X}^{(k)}}\odot\mathbf{R}^{(k)}\Big)
     -\big[g(\mathbf{L})\mathbf{X}^{(k-1)'}\big]^{\top}\Big(\frac{\partial
f(\mathbf{x}|\theta')}{\partial\mathbf{X}^{(k)}}\odot\mathbf{R}^{(k)'}\Big)\big\|_{F} \\
 \nonumber\leq&\big\|g(\mathbf{L})\triangle\mathbf{X}^{(k-1)}\big\|_{F} \big\|\frac{\partial
f(\mathbf{x}|\theta)}{\partial\mathbf{X}^{(k)}}\odot\mathbf{R}^{(k)}\big\|_{F}
    +\big\|g(\mathbf{L})\mathbf{X}^{(k-1)'}\big\|_{F} \big\|\frac{\partial
f(\mathbf{x}|\theta)}{\partial\mathbf{X}^{(k)}}\odot\mathbf{R}^{(k)}-\frac{\partial
f(\mathbf{x}|\theta')}{\partial\mathbf{X}^{(k)}}\odot\mathbf{R}^{(k)'}\big\|_{F} \\
\nonumber\leq&C_{g}\|\triangle\mathbf{X}^{(k-1)}\|_{F}\cdot\alpha_{\sigma}\big\|\frac{\partial
f(\mathbf{x}|\theta)}{\partial\mathbf{X}^{(k)}}\big\|_{F}
     +C_{g}\|\mathbf{X}^{(k-1)'}\|_{F}\Big\|\frac{\partial
f(\mathbf{x}|\theta)}{\partial\mathbf{X}^{(k)}}\odot\mathbf{R}^{(k)}-\frac{\partial
f(\mathbf{x}|\theta')}{\partial\mathbf{X}^{(k)}}\odot\mathbf{R}^{(k)'}\big\|_{F}.
\end{align}

Let
\begin{equation}\label{equ:define_gammak}
\gamma_{k}:=\big\|\frac{\partial
f(\mathbf{x}|\theta)}{\partial\mathbf{X}^{(k)}}\odot\mathbf{R}^{(k)}-\frac{\partial
f(\mathbf{x}|\theta')}{\partial\mathbf{X}^{(k)}}\odot\mathbf{R}^{(k)'}\big\|_{F}.
\end{equation}

Then, combining \eqref{equ:bound_X}, \eqref{equ:bound_df_dX}  and \eqref{equ:bound_variation_X}, we have
\begin{equation}
\big\|\nabla_{\mathbf{W}^{(k)}}f(\mathbf{x}|\theta)-\nabla_{\mathbf{W}^{(k)}}f(\mathbf{x}|\theta')\big\|_{F}
 \leq B^{K-1}\alpha_{\sigma}^{K+1}C_{g}^{K+1}C_{\mathbf{X}}\Big(\sum_{j=1}^{k-1}\|\triangle\mathbf{W}^{(j)}\|_{2}\Big)+
 B^{k-1}\alpha_{\sigma}^{k-1}C_{g}^{k}C_{\mathbf{X}}\cdot\gamma_{k}, \label{equ:bound_of_var_df_dWk}
 \end{equation}
Next, we need to bound $\gamma_{k}$.
\begin{align}
\nonumber \gamma_{k}\leq&\Big\|\frac{\partial
f(\mathbf{x}|\theta)}{\partial\mathbf{X}^{(k)}}\odot\big(\mathbf{R}^{(k)}-\mathbf{R}^{(k)'}\big)\Big\|_{F}
   +\Big\|\big(\frac{\partial
f(\mathbf{x}|\theta)}{\partial\mathbf{X}^{(k)}}-\frac{\partial
f(\mathbf{x}|\theta')}{\partial\mathbf{X}^{(k)}}\big)\odot\mathbf{R}^{(k)'}\Big\|_{F} \\
\nonumber \leq&h_{k}+\alpha_{\sigma}\Big\|\frac{\partial
f(\mathbf{x}|\theta)}{\partial\mathbf{X}^{(k)}}-\frac{\partial
f(\mathbf{x}|\theta')}{\partial\mathbf{X}^{(k)}}\Big\|_{F} \\
\nonumber \leq&h_{k}+\alpha_{\sigma}\Big\|g(\mathbf{L})^{\top}\Big(\frac{\partial
f(\mathbf{x}|\theta)}{\partial\mathbf{X}^{(k+1)}}\odot\mathbf{R}^{(k+1)}\Big)\big[\mathbf{W}^{(k+1)}\big]^{\top}
     -g(\mathbf{L})^{\top}\Big(\frac{\partial
f(\mathbf{x}|\theta')}{\partial\mathbf{X}^{(k)}}\odot\mathbf{R}^{(k+1)'}\Big)\big[\mathbf{W}^{(k+1)'}\big]^{\top}\Big\|_{F}\\
\nonumber \leq&h_{k}+\alpha_{\sigma}\|g(\mathbf{L})\|_{2}\Big\|\frac{\partial
f(\mathbf{x}|\theta)}{\partial\mathbf{X}^{(k+1)}}\odot\mathbf{R}^{(k+1)}\Big\|_{F}\|\triangle\mathbf{W}^{(k+1)}\|_{2}
     +\alpha_{\sigma}\|g(\mathbf{L})\|_{2}\|\mathbf{W}^{(k+1)'}\|_{2}\gamma_{k+1}\\
\nonumber   \leq &h_{k}+\alpha^{2}_{\sigma}C_{g}(B\alpha_{\sigma}C_{g})^{K-k}\|\triangle\mathbf{W}^{(k+1)}\|_{2}
+B\alpha_{\sigma}C_{g}\gamma_{k+1},
\end{align}
where $h_{k}:=\big\|\frac{\partial f(\mathbf{x}|\theta)}{\partial\mathbf{X}^{(k)}}\odot\big(\mathbf{R}^{(k)}-\mathbf{R}^{(k)'}\big)\big\|_{F}$.
By \eqref{equ:bound_variation_X_W},
\begin{align}\label{equ:bound_var_Rk}
\nonumber   \|\mathbf{R}^{(k)}-\mathbf{R}^{(k)'}\|_{F}
  =&\big\|\nabla\sigma\big(g(\mathbf{L})\mathbf{X}^{(k-1)}\mathbf{W}^{(k)}\big)
-\nabla\sigma\big(g(\mathbf{L})\mathbf{X}^{(k-1)'}\mathbf{W}^{(k)'}\big)\big\|_{F}\\
\nonumber   \leq&\nu_{\sigma}C_{g}\big\|\mathbf{X}^{(k-1)}\mathbf{W}^{(k)}-\mathbf{X}^{(k-1)'}\mathbf{W}^{(k)'}\big\|_{F}\\
   \leq &\nu_{\sigma}B^{k-1}\alpha^{k-1}_{\sigma}C^{k}_{g}C_{\mathbf{X}}\Big(\sum_{j=1}^{k}\|\triangle\mathbf{W}^{(j)}\|_{2}\Big).
\end{align}
Combining \eqref{equ:bound_df_dX}, we have
\begin{align} \label{equ:bound_hk}
\nonumber  h_{k}&=\big\|\frac{\partial f(\mathbf{x}|\theta)}{\partial\mathbf{X}^{(k)}}\odot\big(\mathbf{R}^{(k)}-\mathbf{R}^{(k)'}\big)\big\|_{F}
       \leq \big\|\frac{\partial f(\mathbf{x}|\theta)}{\partial\mathbf{X}^{(k)}}\big\|_{F}\cdot\|\mathbf{R}^{(k)}-\mathbf{R}^{(k)'}\|_{F}\\
       &\leq \nu_{\sigma}B^{K}\alpha^{K}_{\sigma}C^{K+1}_{g}C_{\mathbf{X}}\Big(\sum_{j=1}^{k}\|\triangle\mathbf{W}^{(j)}\|_{2}\Big).
\end{align}
Let $h_{\max}=\nu_{\sigma}B^{K}\alpha^{K}_{\sigma}C^{K+1}_{g}C_{\mathbf{X}}\|\triangle\theta\|_{*}$. Then, it is easy to see that
\begin{equation}\label{equ:bound_hk_by_hmax}
h_{k}\leq h_{\max}~\textrm{holds for all}~k=1,2,\dots,K.
\end{equation}
Therefore, $$\gamma_{k}\leq h_{\max}+\alpha^{2}_{\sigma}C_{g}(B\alpha_{\sigma}C_{g})^{K-k}\|\triangle\mathbf{W}^{(k+1)}\|_{2}
+B\alpha_{\sigma}C_{g}\cdot\gamma_{k+1}.$$
Furthermore,
since
\begin{align*}
   &\|\frac{\partial
f(\mathbf{x}|\theta)}{\partial\mathbf{X}^{(K)}}-\frac{\partial
f(\mathbf{x}|\theta')}{\partial\mathbf{X}^{(K)}}\|_{F} \\
   =&\|\nabla\sigma\big(\boldsymbol{\delta}_{\mathbf{x}}^{\top}g(\mathbf{L})\mathbf{X}^{(K)}\mathbf{w}\big)
\big[\boldsymbol{\delta}^{\top}_{\mathbf{x}}g(\mathbf{L})\big]^{\top}\mathbf{w}^{\top}
   -\nabla\sigma\big(\boldsymbol{\delta}_{\mathbf{x}}^{\top}g(\mathbf{L})\mathbf{X}^{(K)'}\mathbf{w}'\big)
\big[\boldsymbol{\delta}^{\top}_{\mathbf{x}}g(\mathbf{L})\big]^{\top}\mathbf{w}'^{\top}\|_{F} \\
   \leq &
BC_{g}\|\nabla\sigma\big(\boldsymbol{\delta}_{\mathbf{x}}^{\top}g(\mathbf{L})\mathbf{X}^{(K)}\mathbf{w} \big)-\nabla\sigma\big(\boldsymbol{\delta}_{\mathbf{x}}^{\top}g(\mathbf{L})\mathbf{X}^{(K)'}\mathbf{w}'\big)\|_{F}
   +\|\nabla\sigma\big(\boldsymbol{\delta}_{\mathbf{x}}^{\top}g(\mathbf{L})\mathbf{X}^{(K)'}\mathbf{w}'\big)
\big[\boldsymbol{\delta}^{\top}_{\mathbf{x}}g(\mathbf{L})\big]^{\top}\triangle\mathbf{w}^{\top}\|_{F}
\\
\leq &  \alpha_{\sigma}C_{g}\|\triangle\mathbf{w}\|_{F}
+\nu_{\sigma}BC^{2}_{g}\big\|\mathbf{X}^{(K)}\mathbf{w}
-\mathbf{X}^{(K)'}\mathbf{w}'\big\|_{F} \\
    \leq&  \alpha_{\sigma}C_{g}\|\triangle\mathbf{w}\|_{2}
+\nu_{\sigma}B^{K+1}\alpha^{K}_{\sigma}C^{K+2}_{g}C_{\mathbf{X}}\|\triangle\theta\|_{*},
\end{align*}
we have
\begin{align*}
    \gamma_{K} =& \|\frac{\partial
f(\mathbf{x}|\theta)}{\partial\mathbf{X}^{(K)}}\odot\mathbf{R}^{(K)}-\frac{\partial
f(\mathbf{x}|\theta')}{\partial\mathbf{X}^{(K)}}\odot\mathbf{R}^{(K)'}\|_{F} \\
    \leq &  \|\frac{\partial
f(\mathbf{x}|\theta)}{\partial\mathbf{X}^{(K)}}\odot(\mathbf{R}^{(K)}-\mathbf{R}^{(K)'})\|_{F}+
    \|(\frac{\partial
f(\mathbf{x}|\theta)}{\partial\mathbf{X}^{(K)}}-\frac{\partial
f(\mathbf{x}|\theta')}{\partial\mathbf{X}^{(K)}})\odot\mathbf{R}^{(K)'}\|_{F}\\
    \leq &  h_{K}+\alpha_{\sigma}\|\frac{\partial
f(\mathbf{x}|\theta)}{\partial\mathbf{X}^{(K)}}-\frac{\partial
f(\mathbf{x}|\theta')}{\partial\mathbf{X}^{(K)}}\|_{F} \\
    \leq&  h_{\max}+\alpha^{2}_{\sigma}C_{g}\|\triangle\mathbf{w}\|_{2}
+\nu_{\sigma}B^{K+1}\alpha^{K+1}_{\sigma}C^{K+2}_{g}C_{\mathbf{X}}\|\triangle\theta\|_{*}.
\end{align*}
Finally, based on the above recursive formula of $\gamma_{k}$, we have
\begin{align}
\nonumber  \gamma_{k} \leq & h_{\max}\Big(\sum_{j=0}^{K-k}(B\alpha_{\sigma}C_{g})^{j}\Big)
       +\alpha^{2}_{\sigma}C_{g}(B\alpha_{\sigma}C_{g})^{K-k}\Big(\sum_{j=k+1}^{K+1}\|\triangle\mathbf{W}^{(j)}\|_{2}\Big)\\
\nonumber  & +\nu_{\sigma}B^{K+1}\alpha^{K+1}_{\sigma}C^{K+2}_{g}C_{\mathbf{X}}(B\alpha_{\sigma}C_{g})^{K-k}\|\triangle\theta\|_{*}\\
\nonumber  \leq & h_{\max}\Big(\sum_{j=0}^{K-k}(B\alpha_{\sigma}C_{g})^{j}\Big)
     +\alpha^{2}_{\sigma}C_{g}(B\alpha_{\sigma}C_{g})^{K-k}\Big(\sum_{j=k+1}^{K+1}\|\triangle\mathbf{W}^{(j)}\|_{2}\Big)\\
   &+\nu_{\sigma}B^{2K+1-k}\alpha^{2K+1-k}_{\sigma}C^{2K+2-k}_{g}C_{\mathbf{X}}\|\triangle\theta\|_{*}, \label{equ:bound_gamma_k}
\end{align}
where $\triangle\mathbf{W}^{(K+1)}=\triangle\mathbf{w}$.
Finally, substituting \eqref{equ:bound_gamma_k} into \eqref{equ:bound_of_var_df_dWk},
\begin{align*}
\|\nabla_{\mathbf{W}^{(k)}}f(\mathbf{x}|\theta)-\nabla_{\mathbf{W}^{(k)}}f(\mathbf{x}|\theta')\|_{F}
   \leq&B^{K-1}\alpha_{\sigma}^{K+1}C_{g}^{K+1}C_{\mathbf{X}}\Big(\sum_{j=1}^{k-1}\|\triangle\mathbf{W}^{(j)}\|_{2}\Big)
 +B^{k-1}\alpha_{\sigma}^{k-1}C_{g}^{k}C_{\mathbf{X}}\cdot\gamma_{k}\\
  \nonumber\leq&\Big(\nu_{\sigma}B^{2K}\alpha^{2K}_{\sigma}C^{2K+2}_{g}C^{2}_{\mathbf{X}}
  +B^{K-1}\alpha_{\sigma}^{K+1}C_{g}^{K+1}C_{\mathbf{X}}\Big)\|\triangle\theta\|_{*}\\
  &   +\nu_{\sigma}B^{K+k-1}\alpha_{\sigma}^{K+k-1}C_{g}^{K+k+1}C^{2}_{\mathbf{X}}
\Big(\sum_{j=0}^{K-k}(B\alpha_{\sigma}C_{g})^{j}\Big)\|\triangle\theta\|_{*}\\
  \leq& (\kappa_{1}+\rho_{k})\|\triangle\theta\|_{*},
      \end{align*}
      which completes the proof of \eqref{equ:bound_variation_df_dW}.

      Up to now, the proof of Lemma \ref{lem:bound_variation_df_dW} is complete. Then, we prepare to prove Lemma \textcolor{red}{3} and Lemma \textcolor{red}{4}.

\subsection*{Proof of Lemma \ref{lem:Same_sample}.}
Now, we are ready to prove Eq. (\ref{equ:bound_variation_dl_dw_same_sample}). Firstly, note that
\begin{align*}
      & \|\nabla_{\mathbf{w}}\ell(f(\mathbf{x}_{t}|\theta_{t-1}),y_{t})-\nabla_{\mathbf{w}}\ell(f(\mathbf{x}_{t}|\theta'_{t-1}),y_{t})\|_{F}
     =\big\|\frac{\partial\ell(\hat{y},y_{t})}{\partial\hat{y}}\nabla_{\mathbf{w}}f(\mathbf{x}_{t}|\theta_{t-1})
-\frac{\partial\ell(\hat{y}',y_{t})}{\partial\hat{y}}\nabla_{\mathbf{w}}f(\mathbf{x}_{t}|\theta'_{t-1})\big\|_{F}\\
   \leq&\Big\|\Big(\frac{\partial\ell(\hat{y},y_{t})}{\partial\hat{y}}-\frac{\partial\ell(\hat{y}',y_{t})}{\partial\hat{y}}\Big)
\nabla_{\mathbf{w}}f(\mathbf{x}_{t}|\theta_{t-1})
      +\frac{\partial\ell(\hat{y}',y_{t})}{\partial\hat{y}}\Big(\nabla_{\mathbf{w}}f(\mathbf{x}_{t}|\theta_{t-1})-
 \nabla_{\mathbf{w}}f(\mathbf{x}_{t}|\theta'_{t-1})\Big)\Big\|_{F}\\
   \leq&\big|\frac{\partial\ell(\hat{y},y_{t})}{\partial\hat{y}}-\frac{\partial\ell(\hat{y}',y_{t})}{\partial\hat{y}}\big|\cdot
 \|\nabla_{\mathbf{w}}f(\mathbf{x}_{t}|\theta_{t-1})\|_{F}
      +\big|\frac{\partial\ell(\hat{y}',y_{t})}{\partial\hat{y}}\big|\cdot\|\nabla_{\mathbf{w}}f(\mathbf{x}_{t}|\theta_{t-1})-
 \nabla_{\mathbf{w}}f(\mathbf{x}|\theta'_{t-1})\|_{F}\\
    \leq&\upsilon_{\ell}\big|f(\mathbf{x}_{t}|\theta_{t-1})-f(\mathbf{x}|\theta'_{t-1})\big|\cdot
 \|\nabla_{\mathbf{w}}f(\mathbf{x}_{t}|\theta_{t-1})\|_{F}
+\alpha_{\ell}\|\nabla_{\mathbf{w}}f(\mathbf{x}_{t}|\theta_{t-1})-
 \nabla_{\mathbf{w}}f(\mathbf{x}_{t}|\theta'_{t-1})\|_{F},
\end{align*}
where $\hat{y}=f(\mathbf{x}_{t}|\theta_{t-1})$ and $\hat{y}'=f(\mathbf{x}_{t}|\theta'_{t-1})$. Then, according to \eqref{equ:bound_df_dW}, \eqref{equ:bound_variation_f} and \eqref{equ:bound_variation_df_dw}, we have
\begin{align*}
       & \|\nabla_{\mathbf{w}}\ell(f(\mathbf{x}_{t}|\theta_{t-1}),y_{t})-\nabla_{\mathbf{w}}\ell(f(\mathbf{x}_{t}|\theta'_{t-1}),y_{t})\|_{F}\\
    \leq &\bigg\{\upsilon_{\ell}B^{2K}\alpha_{\sigma}^{2K+2}C^{2K+2}_{g}C^{2}_{\mathbf{X}}
+\alpha_{\ell}\Big(\upsilon_{\sigma}B^{2K}\alpha^{2K}_{\sigma}C^{2K+2}_{g}C^{2}_{\mathbf{X}}
    +B^{K-1}\alpha^{K+1}_{\sigma}C^{K+1}_{g}C_{\mathbf{X}}\Big)\bigg\}\|\triangle\theta_{t-1}\|_{*}.
\end{align*}
This proves Eq. (\ref{equ:bound_variation_dl_dw_same_sample}).

Similarly, for $k=1,2,\dots,K$,
\begin{align*}
    &\|\nabla_{\mathbf{W}^{(k)}}\ell(f(\mathbf{x}_{t}|\theta_{t-1}),y_{t})-\nabla_{\mathbf{W}^{(k)}}\ell(f(\mathbf{x}_{t}|\theta'_{t-1}),y_{t})\|_{F}\\
   \leq&\upsilon_{\ell}\big|f(\mathbf{x}_{t}|\theta_{t-1})-f(\mathbf{x}_{t}|\theta'_{t-1})\big|\cdot
\|\nabla_{\mathbf{W}^{(k)}}f(\mathbf{x}_{t}|\theta_{t-1})\|_{F}
+\alpha_{\ell}\|\nabla_{\mathbf{W}^{(k)}}f(\mathbf{x}_{t}|\theta_{t-1})-\nabla_{\mathbf{W}^{(k)}}f(\mathbf{x}_{t}|\theta'_{t-1})\|_{F}.
\end{align*}
Then, according to \eqref{equ:bound_df_dW}, \eqref{equ:bound_variation_f} and \eqref{equ:bound_variation_df_dW},
\begin{align*}
   &\|\nabla_{\mathbf{W}^{(k)}}\ell(f(\mathbf{x}_{t}|\theta_{t-1}),y_{t})-\nabla_{\mathbf{W}^{(k)}}\ell(f(\mathbf{x}_{t}|\theta'_{t-1}),y_{t})\|_{F}\\
   \leq &\bigg\{\upsilon_{\ell}B^{2K}\alpha_{\sigma}^{2K+2}C^{2K+2}_{g}C^{2}_{\mathbf{X}}+
\alpha_{\ell}\Big\{\Big(\nu_{\sigma}B^{2K}\alpha^{2K}_{\sigma}C^{2K+2}_{g}C^{2}_{\mathbf{X}}\\
   &+B^{K-1}\alpha_{\sigma}^{K+1}C_{g}^{K+1}C_{\mathbf{X}}\Big)+ \nu_{\sigma}B^{K+k-1}\alpha_{\sigma}^{K+k-1}C_{g}^{K+k+1}
    C^{2}_{\mathbf{X}}\Big(\sum_{j=0}^{K-k}(B\alpha_{\sigma}C_{g})^{j}\Big)\Big\}\bigg\}\|\triangle\theta_{t-1}\|_{*},
\end{align*}
which competes the proof of Eq. (\ref{equ:bound_variation_dl_dW_same_sample}).

\subsection*{Proof of Lemma \ref{lem:diff_sample}.}
Since $|\frac{\partial\ell(\hat{y},y)}{\partial\hat{y}}|\leq\alpha_{\ell}$ for any $\hat{y}$ and $y$, we first have that for $k=1,2,\dots,K+1$,
\begin{align*}
  \|\nabla_{\mathbf{W}^{(k)}}\ell(f(\mathbf{x}_{t}|\theta_{t-1}),y_{t})-\nabla_{\mathbf{W}^{(k)}}\ell(f(\mathbf{x}'_{t}|\theta'_{t-1}),y'_{t})\|_{F}
 =&\Big\|\frac{\partial\ell(\hat{y},y_{t})}{\partial\hat{y}}\nabla_{\mathbf{W}^{(k)}}f(\mathbf{x}_{t}|\theta_{t-1})
 -\frac{\partial\ell(\hat{y}',y'_{t})}{\partial\hat{y}}\nabla_{\mathbf{W}^{(k)}}f(\mathbf{x}'_{t}|\theta'_{t-1})\Big\|_{F}\\
\leq&\alpha_{\ell}\Big(\|\nabla_{\mathbf{W}^{(k)}}f(\mathbf{x}_{t}|\theta_{t-1})\|_{F}
 +\|\nabla_{\mathbf{W}^{(k)}}f(\mathbf{x}'_{t}|\theta'_{t-1})\|_{F}\Big),
\end{align*}
where $\hat{y}=f(\mathbf{x}_{t}|\theta_{t-1})$ and $\hat{y}'=f(\mathbf{x}'_{t}|\theta'_{t-1})$ and $\mathbf{W}^{(K+1)}=\mathbf{w}$.
Finally, according to \eqref{equ:bound_df_dW}, 
\begin{align*}
  \|\nabla_{\mathbf{W}^{(k)}}\ell(f(\mathbf{x}_{t}|\theta_{t-1}),y_{t})-\nabla_{\mathbf{W}^{(k)}}\ell(f(\mathbf{x}'_{t}|\theta'_{t-1}),y'_{t})\|_{F}
\leq&\alpha_{\ell}\Big(\|\nabla_{\mathbf{W}^{(k)}}f(\mathbf{x}|\theta_{t-1})\|_{F}
 +\|\nabla_{\mathbf{W}^{(k)}}f(\mathbf{x}|\theta'_{t-1})\|_{F}\Big)\\
 \leq & 2\alpha_{\ell}B^{K}\alpha_{\sigma}^{K+1}C^{K+1}_{g}C_{\mathbf{X}},
 \end{align*}
 holds for $k=1,2,\dots,K+1$.

\section*{APPENDIX D:Proof of Theorem \ref{thm:bound_variation_theta_T}} \label{proof:thm:bound_variation_theta_T}
Based on Lemma \ref{lem:Same_sample} and Lemma \ref{lem:diff_sample}, we detail the proof of Theorem \ref{thm:bound_variation_theta_T} as follows.

Note that $(\mathbf{x}_{t},y_{t})=(\mathbf{x}'_{t},y'_{t})$ with probability $1-\frac{1}{m}$ and $(\mathbf{x}_{t},y_{t})\neq(\mathbf{x}'_{t},y'_{t})$ with probability $\frac{1}{m}$. By considering Eq. (\ref{equ:update_W_and_w}) (in Section \ref{sec:3.2})
 and incorporating the probability of the two scenarios presented in Lemma \ref{lem:Same_sample} and Lemma \ref{lem:diff_sample}, using $\mathcal{F}$ and $\mathcal{F}^{'}$ to denote $f(\mathbf{x}_{t}|\theta_{t-1})$ and $f(\mathbf{x}_{t}|\theta'_{t-1})$, respectively, we have:
 \begin{align*}
      \mathbb{E}_{\mathcal{A}}\big[\|\triangle\mathbf{w}_{t}\|_{2}\big]
    =& (1-\frac{1}{m})\mathbb{E}_{\mathcal{A}}\Big[\|\triangle\mathbf{w}_{t-1}-\eta\big(\nabla_{\mathbf{w}}\ell(\mathcal{F},y_{t})
-\nabla_{\mathbf{w}}\ell(\mathcal{F}^{'},y_{t})\big)\|_{2}\Big] \\
     & +\frac{1}{m}\mathbb{E}_{\mathcal{A}}\Big[\|\triangle\mathbf{w}_{t-1} -\eta\big(\nabla_{\mathbf{w}}\ell(\mathcal{F},y_{t})
-\nabla_{\mathbf{w}}\ell(\mathcal{F}^{'},y'_{t})\big)\|_{2}\Big]\\
     \leq&(1-\frac{1}{m})\mathbb{E}_{\mathcal{A}}\Big[\|\triangle\mathbf{w}_{t-1}\|_{2} +\eta\|\nabla_{\mathbf{w}}\ell(\mathcal{F},y_{t})
-\nabla_{\mathbf{w}}\ell(\mathcal{F}^{'},y_{t})\|_{2}\Big]\\
     & +\frac{1}{m}\mathbb{E}_{\mathcal{A}}\Big[\|\triangle\mathbf{w}_{t-1}\|_{2} +\eta\|\nabla_{\mathbf{w}}\ell(\mathcal{F},y_{t})
-\nabla_{\mathbf{w}}\ell(\mathcal{F}^{'},y'_{t})\|_{2}\Big]\\
     \leq&(1-\frac{1}{m})\mathbb{E}_{\mathcal{A}}\Big[\|\triangle\mathbf{w}_{t-1}\|_{2}+\eta\|\nabla_{\mathbf{w}}\ell(\mathcal{F},y_{t})
-\nabla_{\mathbf{w}}\ell(\mathcal{F}^{'},y_{t})\|_{F}\Big]\\
    & +\frac{1}{m}\mathbb{E}_{\mathcal{A}}\big[\|\triangle\mathbf{w}_{t-1}\|_{2}+\eta\|\nabla_{\mathbf{w}}\ell(\mathcal{F},y_{t})
-\nabla_{\mathbf{w}}\ell(\mathcal{F}^{'},y'_{t})\|_{F}\big].
\end{align*}

Based on Lemma \ref{lem:Same_sample} and Lemma \ref{lem:diff_sample},
\begin{align*}
  \mathbb{E}_{\mathcal{A}}\big[\|\triangle\mathbf{w}_{t}\|_{2}\big]\leq&\mathbb{E}_{\mathcal{A}}\big[\|\triangle\mathbf{w}_{t-1}\|_{2}\big]
  +\eta\kappa_{1}\mathbb{E}_{\mathcal{A}}\big[\|\triangle\theta_{t-1}\|_{*}\big]
  +\frac{2\eta\alpha_{\ell}B^{K}\alpha_{\sigma}^{K+1}C^{K+1}_{g}C_{\mathbf{X}}}{m}.
 \end{align*}
 Similarly, for $k=1,2,\dots,K$,
\begin{align*}
     \mathbb{E}_{\mathcal{A}}\big[\|\triangle\mathbf{W}^{(k)}_{t}\|_{2}\big]
  \leq&\mathbb{E}_{\mathcal{A}}\big[\|\triangle\mathbf{W}^{(k)}_{t-1}\|_{2}\big]
+\eta(\kappa_{1}+\rho_{k})\mathbb{E}_{\mathcal{A}}\big[\|\triangle\theta_{t-1}\|_{*}\big]
  +\frac{2\eta\alpha_{\ell}B^{K}\alpha_{\sigma}^{K+1}C^{K+1}_{g}C_{\mathbf{X}}}{m}.
\end{align*}
 Then,
 \begin{align*}
    \mathbb{E}_{\mathcal{A}}\big[\|\triangle\theta_{t}\|_{*}\big]
    = &\mathbb{E}_{\mathcal{A}}\big[\|\triangle\mathbf{w}_{t}\|_{2}\big]+\sum_{k=1}^{K}\mathbb{E}_{\mathcal{A}}\big[\|\triangle\mathbf{W}^{(k)}_{t}\|_{2}\big] \\
   \leq & \mathbb{E}_{\mathcal{A}}\big[\|\triangle\mathbf{w}_{t-1}\|_{2}\big]
+\eta\kappa_{1}\mathbb{E}_{\mathcal{A}}\big[\|\triangle\theta_{t-1}\|_{*}\big]
  +\frac{2\eta\alpha_{\ell}B^{K}\alpha_{\sigma}^{K+1}C^{K+1}_{g}C_{\mathbf{X}}}{m}\\
    & +\sum_{k}^{K}\mathbb{E}_{\mathcal{A}}\big[\|\triangle\mathbf{W}^{(k)}_{t-1}\|_{2}\big]
+\eta(\kappa_{1}+\rho_{k})\mathbb{E}_{\mathcal{A}}\big[\|\triangle\theta_{t-1}\|_{*}\big]
    +\frac{2\eta\alpha_{\ell}B^{K}\alpha_{\sigma}^{K+1}C^{K+1}_{g}C_{\mathbf{X}}}{m}\\
  =& \big(1+(K+1)\eta\kappa_{1}+\eta\kappa_{2}\big)\mathbb{E}_{\mathcal{A}}\big[\|\triangle\theta_{t-1}\|_{*}\big]
  +\frac{2(K+1)\eta\alpha_{\ell}B^{K}\alpha_{\sigma}^{K+1}C^{K+1}_{g}C_{\mathbf{X}}}{m}.
 \end{align*}
  where $\kappa_{2}=\sum_{k=1}^{K}\rho_{k}$. By \eqref{equ:rho_k}, we have $\kappa_{2}=\nu_{\sigma}\big(B\alpha_{\sigma}C_{g}\big)^{K}C_{g}^{2}C^{2}_{\mathbf{X}}
\Big(\sum_{j=0}^{K-1}(j+1)(B\alpha_{\sigma}C_{g})^{j}\Big)$, as defined in (\textcolor{red}{8})).
Finally, since $\|\triangle\theta_{0}\|_{*}=\|\theta_{0}-\theta'_{0}\|_{*}=0$
$$\mathbb{E}_{\mathcal{A}}\big[\|\triangle\theta_{T}\|_{*}\big]
\leq \frac{c}{m}
  \sum_{t=1}^{T}\Big(1+(K+1)\eta\kappa_{1}+\eta\kappa_{2}\Big)^{t-1}.
$$
This completes the proof of Theorem \ref{thm:bound_variation_theta_T}.

\section*{\textcolor{black}{APPENDIX E: Proof for Section \ref{sec:61}}} \label{proof:section6_1}
\textcolor{black}{
Recall to the GCNII,
$$\left\{
  \begin{array}{cc}
    \mathbf{X}^{(k)}=\sigma\Big(\big((1-a_{k})g(\mathbf{L})\mathbf{X}^{(k-1)}
                  +a_{k}\mathbf{X}^{(0)}\big)
                  \big((1-b_{k})\mathbf{I}_{d}+b_{k}\mathbf{W}^{(k)}\big)\Big), & k=1,2,\dots,K; \\
   f(\mathbf{x} | \theta) = \sigma\Big( \boldsymbol{\delta}_\mathbf{x}^\top
                  \big( (1 - a_{K+1}) g(\mathbf{L}) \mathbf{X}^{(K)} + a_{K+1} \mathbf{X}^{(0)} \big) \mathbf{w} \Big)  &  \\
  \end{array}
\right.
$$
}
\newline
\textcolor{black}{
{\bf Proof of Eq. (\ref{equ:GCNII_iterative_deltaXk}) and Eq. (\ref{equ:GCNII_var_boundof_deltaXk}):}}\\ 

\textcolor{black}{
We first bound the output of each layer, i.e., bound $\|\mathbf{X}^{(k)}\|_{F}$. Applying $\|\sigma(\mathbf{Z})\|_{F}\leq\alpha_{\sigma}\|\mathbf{Z}\|_{F}$ holds for any matrix $\mathbf{Z}$ and  $\|\mathbf{A}_{1}\mathbf{A}_{2}\|_{F}\leq\|\mathbf{A}_{1}\|_{F}\|\mathbf{A}_{2}\|_{2}$, we have
\begin{align*}
\|\mathbf{X}^{(k)}\|_{F}=&\sigma\Big(\big((1-a_{k})g(\mathbf{L})\mathbf{X}^{(k-1)}
                  +a_{k}\mathbf{X}^{(0)}\big)
                  \big((1-b_{k})\mathbf{I}_{d}+b_{k}\mathbf{W}^{(k)}\big)\Big)\\
\leq &\alpha_{\sigma}\Big\|\big((1-a_{k})g(\mathbf{L})\mathbf{X}^{(k-1)}
                  +a_{k}\mathbf{X}^{(0)}\big)
                  \big((1-b_{k})\mathbf{I}_{d}+b_{k}\mathbf{W}^{(k)}\big)\Big\|_{F}\\
\leq & \alpha_{\sigma}\big\|(1-a_{k})g(\mathbf{L})\mathbf{X}^{(k-1)}
                  +a_{k}\mathbf{X}^{(0)}\big\|_{F}\cdot\big\|(1-b_{k})\mathbf{I}_{d}+b_{k}\mathbf{W}^{(k)}\big\|_{2}.
\end{align*}
Furthermore, since $C_{\mathbf{X}}=\|\mathbf{X}\|_{F}=\|\mathbf{X}^{(0)}\|_{F}$, $C_{g}=\|g(\mathbf{L})\|_{2}$, and $\|\mathbf{A}_{1}\mathbf{A}_{2}\|_{F}\leq\|\mathbf{A}_{1}\|_{2}\|\mathbf{A}_{2}\|_{F}$,
\begin{align*}
\big\|(1-a_{k})g(\mathbf{L})\mathbf{X}^{(k-1)}
                  +a_{k}\mathbf{X}^{(0)}\big\|_{F}&\leq \big\|(1-a_{k})g(\mathbf{L})\mathbf{X}^{(k-1)}\big\|_{F}
                  +\big\|a_{k}\mathbf{X}^{(0)}\big\|_{F}\\
                  &\leq(1-a_{k})C_g\|\mathbf{X}^{(k-1)}\|_{F}+a_{k}C_{g},
\end{align*}
and since $\|\mathbf{W}^{(k)}\|_{2}\leq B$, $\big\|(1-b_{k})\mathbf{I}_{d}+b_{k}\mathbf{W}^{(k)}\big\|_{2}\leq 1-b_{k}+b_{k}B$. Therefore,
\begin{align*}
\|\mathbf{X}^{(k)}\|_{F}&\leq \alpha_{\sigma}\big((1-a_{k})C_g\|\mathbf{X}^{(k-1)}\|_{F}+a_{k}C_{g}\big)(1-b_{k}+b_{k}B)\\
                        &=(1-a_{k})(1-b_{k}+b_{k}B)\alpha_{\sigma}C_g\|\mathbf{X}^{(k-1)}\|_{F}+(1-b_{k}+b_{k}B)a_{k}\alpha_{\sigma}C_{g}.
\end{align*}
Note that $\|\mathbf{X}^{(0)}\|_{F}=C_{\mathbf{X}}$, we thus have that for $k=1,2,\dots,K$,
\begin{equation}\label{equ:GCNII_bound_Xk}
\|\mathbf{X}^{(k)}\|_{F}\leq\Big(\prod_{i=1}^{k}(1-a_{i})(1-b_{i}+b_{i}B)\alpha_{\sigma}C_g\Big)C_{\mathbf{X}}
+\sum_{j=1}^{k}\Big(\prod_{i=j+1}^{k}(1-a_{i})(1-b_{i}+b_{i}B)C_g\Big)\Big((1-b_{j}+b_{j}B)a_{j}\alpha_{\sigma}C_{g}\Big).
\end{equation}
}
\textcolor{black}{
For convenience, in the following text we denote
\begin{equation}\label{equ:GCNII_define_bound_Xk}
B_{\mathbf{X}}^{(k)}:= \Big(\prod_{i=1}^{k}(1-a_{i})(1-b_{i}+b_{i}B)\alpha_{\sigma}C_g\Big)C_{\mathbf{X}}
+\sum_{j=1}^{k}\Big(\prod_{i=j+1}^{k}(1-a_{i})(1-b_{i}+b_{i}B)C_g\Big)\Big((1-b_{j}+b_{j}B)a_{j}\alpha_{\sigma}C_{g}\Big),
\end{equation}
and thus $\|\mathbf{X}^{(k)}\|_{F}\leq B_{\mathbf{X}}^{(k)}$. When $a_{k}=0,b_{k}=1$ for all $k$, GCNII degenerates into the traditional GCN, and $B_{\mathbf{X}}^{(k)}=B^{k}\alpha_{\sigma}^{k}C_{g}^{k}C_{\mathbf{X}}$, which is the same as shown in \eqref{equ:bound_X}. The bound of $\|\mathbf{X}^{(k)}\|_{F}$ implies
$$\|\mathbf{H}^{(k)}\|_{F}=\|(1-a_{k+1})g(\mathbf{L})\mathbf{X}^{(k)}
                  +a_{k+1}\mathbf{X}^{(0)}\|_{F}\leq(1-a_{k+1})C_{g}B_{\mathbf{X}}^{(k)}+a_{k+1}C_{\mathbf{X}}.$$
Then, we bound the perturbation of the output of each layer, i.e., bound $\|\triangle\mathbf{X}^{(k)}\|_{F}$.
Note that $$\triangle\mathbf{X}^{(k)}=\mathbf{X}^{(k)}-\mathbf{X}^{(k)'}=
\sigma\Big(\mathbf{H}^{(k-1)}\big((1-b_{k})\mathbf{I}_{d}+b_{k}\mathbf{W}^{(k)}\big)\Big)
        -\sigma\Big(\mathbf{H}^{(k-1)'}\big((1-b_{k})\mathbf{I}_{d}+b_{k}\mathbf{W}^{(k)'}\big)\Big).$$ Thus, following a calculation similar to Lemma \ref{lem:bound_GCN}, we have
       \begin{align*}
\|\triangle\mathbf{X}^{(k)}\|_{F}=&\Big\|\sigma\Big(\mathbf{H}^{(k-1)}\big((1-b_{k})\mathbf{I}_{d}+b_{k}\mathbf{W}^{(k)}\big)\Big)
        -\sigma\Big(\mathbf{H}^{(k-1)'}\big((1-b_{k})\mathbf{I}_{d}+b_{k}\mathbf{W}^{(k)'}\big)\Big)\Big\|_{F}\\
    \leq&\alpha_{\sigma}\big\|\mathbf{H}^{(k-1)}\big((1-b_{k})\mathbf{I}_{d}+b_{k}\mathbf{W}^{(k)}\big)
       -\mathbf{H}^{(k-1)'}\big((1-b_{k})\mathbf{I}_{d}+b_{k}\mathbf{W}^{(k)'}\big)\big\|_{F}\\
       =&\alpha_{\sigma}\bigg(\big\|\mathbf{H}^{(k-1)}-\mathbf{H}^{(k-1)'}\big\|_{F}\cdot
       \big\|(1-b_{k})\mathbf{I}_{d}+b_{k}\mathbf{W}^{(k)}\big\|_{2}
        +\big\|\mathbf{H}^{(k-1)'}\big\|_{F}\cdot\big\|b_{k}(\mathbf{W}^{(k)}-\mathbf{W}^{(k)'})\big\|_{2}\bigg).
                  \end{align*}
Since $\mathbf{H}^{(k-1)}-\mathbf{H}^{(k-1)'}=(1-a_{k})g(\mathbf{L})\triangle\mathbf{X}^{(k-1)}$, $\|\mathbf{H}^{(k-1)}-\mathbf{H}^{(k-1)'}\|_{F}=(1-a_{k})C_g\|\triangle\mathbf{X}^{(k-1)}\|_{F}$. Combining
$\|\mathbf{H}^{(k-1)'}\|_{F}\leq (1-a_{k})C_{g}B_{\mathbf{X}}^{(k-1)}+a_{k}C_{\mathbf{X}}$, $\triangle\mathbf{W}^{(k)}=\mathbf{W}^{(k)}-\mathbf{W}^{(k)'}$, and $\|(1-b_{k})\mathbf{I}_{d}+b_{k}\mathbf{W}^{(k)}\|_{2}\leq(1-b_{k})+b_{k}B$, we have
\begin{equation}\label{equ:GCNII_iterative_bound_var_Xk}
  \|\triangle\mathbf{X}^{(k)}\|_{F}\leq c_{1}^{(k)}\|\triangle\mathbf{X}^{(k-1)}\|_{F}+c_{2}^{(k)}\|\triangle\mathbf{W}^{(k)}\|_{2},
\end{equation}
where $c_{1}^{(k)}=(1-a_{k})(1-b_{k}+b_{k}B)\alpha_{\sigma}C_{g}$ and $c_{2}^{(k)}=\alpha_{\sigma}b_{k}\big((1-a_{k})C_{g}B_{\mathbf{X}}^{(k-1)}+a_{k}C_{\mathbf{X}}\big)$. This completes the proof of Eq. (\ref{equ:GCNII_iterative_deltaXk}).\newline
Furthermore, since
 $\|\triangle\mathbf{X}^{(1)}\|_{F}\leq c_{2}^{(1)}\|\triangle\mathbf{W}^{(1)}\|_{2}$, we further have
 \begin{equation}\label{equ:GCNII_boundof_deltaXk}
  \|\triangle\mathbf{X}^{(k)}\|_{F}\leq e^{(k)}\cdot\big(\sum_{j=1}^{k}\|\triangle\mathbf{W}^{(j)}\|_{2}\big),
 \end{equation}
 where $e^{(k)}=\max\{c_{1}^{(k)}e^{(k-1)},c_{2}^{(k)}\}$ with $e^{(0)}=0$. When $a_{k}=0,b_{k}=1$ for all $k$, GCNII degenerates into the traditional GCN, we have $c_{1}^{(k)}=B\alpha_{\sigma}C_{g},c_{2}^{(k)}=B^{k-1}\alpha_{\sigma}^{k}C_{g}^{k}C_{\mathbf{X}}$, and thus $e^{(k)}=B^{k-1}\alpha_{\sigma}^{k}C_{g}^{k}C_{\mathbf{X}}$, which is the same as shown in \eqref{equ:bound_variation_X}. This conclusively proves Eq. (\ref{equ:GCNII_var_boundof_deltaXk}).}
\newline

 \textcolor{black}{
{\bf Proof of Eq. (\ref{equ:GCNII_bound_var_finaloutput}):}}\\

\textcolor{black}{
To bound $|f(\mathbf{x}|\theta)-f(\mathbf{x}|\theta')|$, we apply the $\alpha_{\sigma}$-Lipschitz property of $\sigma(\cdot)$ and then have
\begin{align*}
  |f(\mathbf{x}|\theta)-f(\mathbf{x}|\theta')|=&|\sigma(\boldsymbol{\delta}_{\mathbf{x}}^{\top}\mathbf{H}^{(K)}\mathbf{w})-
  \sigma(\boldsymbol{\delta}_{\mathbf{x}}^{\top}\mathbf{H}^{(K)'}\mathbf{w}')|
  \leq\alpha_{\sigma}\cdot|\boldsymbol{\delta}_{\mathbf{x}}^{\top}\mathbf{H}^{(K)}\mathbf{w}-
\boldsymbol{\delta}_{\mathbf{x}}^{\top}\mathbf{H}^{(K)'}\mathbf{w}'|,
 \end{align*}
 that is, we need to bound $|\boldsymbol{\delta}_{\mathbf{x}}^{\top}\mathbf{H}^{(K)}\mathbf{w}-
\boldsymbol{\delta}_{\mathbf{x}}^{\top}\mathbf{H}^{(K)'}\mathbf{w}'|$. }
\newline

\textcolor{black}{
Since $\|\mathbf{A}_{1}\mathbf{A}_{2}-\mathbf{A}'_{1}\mathbf{A}'_{2}\|_{F}
       \leq\|\mathbf{A}_{1}-\mathbf{A}_{1}'\|_{F}\|\mathbf{A}_{2}\|_{2}
       +\|\mathbf{A}'_{1}\|_{F}\|\mathbf{A}_{2}-\mathbf{A}_{2}'\|_{2}$,
 \begin{align*}
 |\boldsymbol{\delta}_{\mathbf{x}}^{\top}\mathbf{H}^{(K)}\mathbf{w}-
\boldsymbol{\delta}_{\mathbf{x}}^{\top}\mathbf{H}^{(K)'}\mathbf{w}'|
 \leq \|\boldsymbol{\delta}_{\mathbf{x}}^{\top}(\mathbf{H}^{(K)}-\mathbf{H}^{(K)'})\|_{F}\cdot\|\mathbf{w}\|_{2}
 +\|\boldsymbol{\delta}_{\mathbf{x}}^{\top}\mathbf{H}^{(K)'}\|_{F}\cdot\|\mathbf{w}-\mathbf{w}'\|_{2}.
 \end{align*}
 Since
 $\|\mathbf{H}^{(K)}-\mathbf{H}^{(K)'}\|_{F}\leq(1-a_{K+1})C_g\|\triangle\mathbf{X}^{(K)}\|_{F}$ and $\|\mathbf{w}\|_{2}\leq B$,
$$
\|\boldsymbol{\delta}_{\mathbf{x}}^{\top}(\mathbf{H}^{(K)}-\mathbf{H}^{(K)'})\|_{F}\cdot\|\mathbf{w}\|_{2}
 \leq(1-a_{K+1})BC_g\|\triangle\mathbf{X}^{(K)}\|_{F},
$$
 and
 \begin{align*}
 \|\boldsymbol{\delta}_{\mathbf{x}}^{\top}\mathbf{H}^{(K)'}\|_{F}\cdot\|(\mathbf{w}-\mathbf{w}')\|_{2}
 \leq \Big((1-a_{K+1})C_{g}B_{\mathbf{X}}^{(K)}+a_{K+1}C_{\mathbf{X}}\Big)\|\triangle\mathbf{w}\|_{2},
 \end{align*}
which holds true because  $\|\boldsymbol{\delta}_{\mathbf{x}}^{\top}\mathbf{H}^{(K)'}\|_{F}\leq\|\mathbf{H}^{(K)'}\|_{F}\leq
 (1-a_{K+1})C_{g}B_{\mathbf{X}}^{(K)}+a_{K+1}C_{\mathbf{X}}$. That is,
  \begin{align*}
 |\boldsymbol{\delta}_{\mathbf{x}}^{\top}\mathbf{H}^{(K)}\mathbf{w}-
\boldsymbol{\delta}_{\mathbf{x}}^{\top}\mathbf{H}^{(K)'}\mathbf{w}'|
 \leq (1-a_{K+1})BC_g\|\triangle\mathbf{X}^{(K)}\|_{F}+\big((1-a_{K+1})C_{g}B_{\mathbf{X}}^{(K)}
 +a_{K+1}C_{\mathbf{X}}\big)\|\triangle\mathbf{w}\|_{2}.
 \end{align*}
 By \eqref{equ:GCNII_boundof_deltaXk}, we further have
 \begin{align}\label{equ:GCNII_bound_var_HK_w}
\nonumber  |\boldsymbol{\delta}_{\mathbf{x}}^{\top}\mathbf{H}^{(K)}\mathbf{w}-
\boldsymbol{\delta}_{\mathbf{x}}^{\top}\mathbf{H}^{(K)'}\mathbf{w}'|
 \leq &(1-a_{K+1})BC_g\sum_{j=1}^{K}\|\triangle\mathbf{W}^{(j)}\|_{2}+\big((1-a_{K+1})C_{g}B_{\mathbf{X}}^{(K)}
 +a_{K+1}C_{\mathbf{X}}\big)\|\triangle\mathbf{w}\|_{2}\\
 \leq & \varrho\cdot\Big(\sum_{j=1}^{K}\|\triangle\mathbf{W}^{(j)}\|_{2}+\|\mathbf{w}\|_{2}\Big)
 =\varrho\cdot\|\triangle\theta\|_{*},
 \end{align}
  where $\varrho=\max\Big\{(1-a_{K+1})BC_g\cdot e^{(K)}, (1-a_{K+1})C_{g}B_{\mathbf{X}}^{(K)}+a_{K+1}C_{\mathbf{X}}\Big\}$. Therefore,
 \begin{align}
 |f(\mathbf{x}|\theta)-f(\mathbf{x}|\theta')|\leq\alpha_{\sigma}\cdot|\boldsymbol{\delta}_{\mathbf{x}}^{\top}\mathbf{H}^{(K)}\mathbf{w}-
\boldsymbol{\delta}_{\mathbf{x}}^{\top}\mathbf{H}^{(K)'}\mathbf{w}'|
 \leq\alpha_{\sigma}\varrho\cdot\|\triangle\theta\|_{*}.
 \label{equ:GCNII_var_boundof_finalOutput}
 \end{align}
 Note that when $a_{k}=0,b_{k}=1$ for all $k$, we have $e^{(K)}=B^{K-1}\alpha_{\sigma}^{K}C_{g}^{K}C_{\mathbf{X}}$ and $B_{\mathbf{X}}^{(K)}=B^{K}\alpha_{\sigma}^{K}C_{g}^{K}C_{\mathbf{X}}$, then at this point, $\varrho=B^{K}\alpha_{\sigma}^{K}C_g^{K+1}C_{\mathbf{X}}$, and $ |f(\mathbf{x}|\theta)-f(\mathbf{x}|\theta')|\leq\alpha_{\sigma}\varrho\cdot\|\triangle\theta\|_{*}=
 B^{K}\alpha_{\sigma}^{K+1}C_g^{K+1}C_{\mathbf{X}}\cdot\|\triangle\theta\|_{*}$, which is consistent with \eqref{equ:bound_variation_f}.}

\textcolor{black}{ Thus, we complete the proof of Eq. (\ref{equ:GCNII_bound_var_finaloutput}).}
\newline

 \textcolor{black}{
 {\bf Proof of Eq. (\ref{equ:GCNII_bound_var_df_dw_sec6_1}):}}\\

\textcolor{black}{
 To bound the perturbation of the gradient, we first follow the calculation technique used in Appendix \textcolor{red}{A} to obtain the gradient of $f(\mathbf{x}|\theta)$ as follow:
 \begin{itemize}
 \item [i)] For the final layer,
  \begin{align*}
  \nabla_{\mathbf{w}}f(\mathbf{x}|\theta)=\nabla\sigma(\boldsymbol{\delta}_{\mathbf{x}}^{\top}\mathbf{H}^{(K)}\mathbf{w})
  [\boldsymbol{\delta}_{\mathbf{x}}^{\top}\mathbf{H}^{(K)}]^{\top}.
  \end{align*}
 \item [ii)] For the hidden layer $k=1,2,\dots,K$,
 \begin{align*}
  \nabla_{\mathbf{W}^{(k)}}f(\mathbf{x}|\theta)
  =\sum_{i,j}\frac{\partial
f(\mathbf{x}|\theta)}{\partial\mathbf{X}^{(k)}_{ij}}\cdot\frac{\partial\mathbf{X}^{(k)}_{ij}}{\partial\mathbf{W}^{(k)}}
  =b_{k}[\mathbf{H}^{(k-1)}]^{\top}(\frac{\partial
f(\mathbf{x}|\theta)}{\partial\mathbf{X}^{(k)}}\odot\mathbf{R}^{(k)}),
  \end{align*}
  where
$
 \mathbf{R}^{(k)}=\nabla\sigma\Big(\mathbf{H}^{(k-1)}\big((1-b_{k})\mathbf{I}_{d}+b_{k}\mathbf{W}^{(k)}\big)\Big).
$
  Furthermore,
   \begin{align*}
  \frac{\partial
f(\mathbf{x}|\theta)}{\partial\mathbf{X}^{(k-1)}}
  =\sum_{i,j}\frac{\partial
f(\mathbf{x}|\theta)}{\partial\mathbf{X}^{(k)}_{ij}}\cdot\frac{\partial\mathbf{X}^{(k)}_{ij}}{\partial\mathbf{X}^{(k-1)}}
  =(1-a_{k})[g(\mathbf{L})]^{\top}(\frac{\partial
f(\mathbf{x}|\theta)}{\partial\mathbf{X}^{(k)}}\odot\mathbf{R}^{(k)})
[(1-b_{k})\mathbf{I}_{d}+b_{k}\mathbf{W}^{(k)}]^{\top},
  \end{align*}
  with
  \begin{align*}
  \frac{\partial
f(\mathbf{x}|\theta)}{\partial\mathbf{X}^{(K)}}
  =(1-a_{K+1})\nabla\sigma(\boldsymbol{\delta}_{\mathbf{x}}^{\top}\mathbf{H}^{(K)}\mathbf{w})
  [\boldsymbol{\delta}_{\mathbf{x}}^{\top}g(\mathbf{L})]^{\top}\mathbf{w}^{\top}.
  \end{align*}
\end{itemize}
}

\textcolor{black}{
We now bound $\|\nabla_{\mathbf{w}}f(\mathbf{x}|\theta)-\nabla_{\mathbf{w}}f(\mathbf{x}|\theta')\|_{F}$. Note that $\nabla_{\mathbf{w}}f(\mathbf{x}|\theta)=\nabla\sigma(\boldsymbol{\delta}_{\mathbf{x}}^{\top}\mathbf{H}^{(K)}\mathbf{w})
  [\boldsymbol{\delta}_{\mathbf{x}}^{\top}\mathbf{H}^{(K)}]^{\top}$, we apply $\|\mathbf{A}_{1}\mathbf{A}_{2}-\mathbf{A}'_{1}\mathbf{A}'_{2}\|_{F}
\leq\|\mathbf{A}_{1}-\mathbf{A}_{1}'\|_{F}\cdot\|\mathbf{A}_{2}\|_{F}+\|\mathbf{A}_{1}'\|_{F}\cdot\|\mathbf{A}_{1}-\mathbf{A}'_{2}\|_{F}$ and  have
  \begin{align*}
  &\|\nabla_{\mathbf{w}}f(\mathbf{x}|\theta)-\nabla_{\mathbf{w}}f(\mathbf{x}|\theta')\|_{F}
=\big\|\nabla\sigma(\boldsymbol{\delta}_{\mathbf{x}}^{\top}\mathbf{H}^{(K)}\mathbf{w})
  [\boldsymbol{\delta}_{\mathbf{x}}^{\top}\mathbf{H}^{(K)}]^{\top}
  -\nabla\sigma(\boldsymbol{\delta}_{\mathbf{x}}^{\top}\mathbf{H}^{(K)'}\mathbf{w}')
  [\boldsymbol{\delta}_{\mathbf{x}}^{\top}\mathbf{H}^{(K)'}]^{\top}\big\|_{F}\\
\leq& \big|\nabla\sigma(\boldsymbol{\delta}_{\mathbf{x}}^{\top}\mathbf{H}^{(K)}\mathbf{w})
-\nabla\sigma(\boldsymbol{\delta}_{\mathbf{x}}^{\top}\mathbf{H}^{(K)'}\mathbf{w}')\big|\cdot
  \big\|[\boldsymbol{\delta}_{\mathbf{x}}^{\top}\mathbf{H}^{(K)}]^{\top}\|_{F}+
  \big|\nabla\sigma(\boldsymbol{\delta}_{\mathbf{x}}^{\top}\mathbf{H}^{(K)'}\mathbf{w}')\big|\cdot
  \big\|[\boldsymbol{\delta}_{\mathbf{x}}^{\top}\mathbf{H}^{(K)}]^{\top}
  -[\boldsymbol{\delta}_{\mathbf{x}}^{\top}\mathbf{H}^{(K)'}]^{\top}\|_{F}.
  \end{align*}
  We further apply the property of $\sigma(\cdot)$ and have
 \begin{align*}
 &\|\nabla_{\mathbf{w}}f(\mathbf{x}|\theta)-\nabla_{\mathbf{w}}f(\mathbf{x}|\theta')\|_{F}
 \leq \nu_{\sigma}\cdot\big|\boldsymbol{\delta}_{\mathbf{x}}^{\top}\mathbf{H}^{(K)}\mathbf{w}
-\boldsymbol{\delta}_{\mathbf{x}}^{\top}\mathbf{H}^{(K)'}\mathbf{w}'\big|\cdot
\big\|[\boldsymbol{\delta}_{\mathbf{x}}^{\top}\mathbf{H}^{(K)}]^{\top}\big\|_{F}+
\alpha_{\sigma}\cdot\big\|[\boldsymbol{\delta}_{\mathbf{x}}^{\top}\mathbf{H}^{(K)}]^{\top}
  -[\boldsymbol{\delta}_{\mathbf{x}}^{\top}\mathbf{H}^{(K)'}]^{\top}\big\|_{F}\\
\leq& \nu_{\sigma}\cdot\big|\boldsymbol{\delta}_{\mathbf{x}}^{\top}\mathbf{H}^{(K)}\mathbf{w}
-\boldsymbol{\delta}_{\mathbf{x}}^{\top}\mathbf{H}^{(K)'}\mathbf{w}'\big|\cdot
\big((1-a_{K+1})C_{g}B_{\mathbf{X}}^{(K)}+a_{K+1}C_{\mathbf{X}}\big)+
\alpha_{\sigma}\cdot(1-a_{K+1})C_g\|\triangle\mathbf{X}^{(K)}\|_{F}.
 \end{align*}
Finally, combining \eqref{equ:GCNII_boundof_deltaXk} and \eqref{equ:GCNII_bound_var_HK_w}, we have
\begin{align*}
 \|\nabla_{\mathbf{w}}f(\mathbf{x}|\theta)-\nabla_{\mathbf{w}}f(\mathbf{x}|\theta')\|_{F}
\leq &\nu_{\sigma}\varrho\cdot\|\triangle\theta\|_{*}\cdot
\big((1-a_{K+1})C_{g}B_{\mathbf{X}}^{(K)}+a_{K+1}C_{\mathbf{X}}\big)+
\alpha_{\sigma}\cdot(1-a_{K+1})C_ge^{(K)}\cdot(\sum_{j=1}^{K}\|\triangle\mathbf{W}^{(j)}\|_{2})\\
\leq & \Big(\nu_{\sigma}\varrho\cdot
\big((1-a_{K+1})C_{g}B_{\mathbf{X}}^{(K)}+a_{K+1}C_{\mathbf{X}}\big)+\alpha_{\sigma}\cdot(1-a_{K+1})C_ge^{(K)}\Big)
\cdot\|\triangle\theta\|_{*}.
 \end{align*}
 Thus, we complete the proof of Eq. (\ref{equ:GCNII_bound_var_df_dw_sec6_1}).
}

\textcolor{black}{
{\bf Proof of bounding $\|\nabla_{\mathbf{W}^{(k)}} f(\mathbf{x} | \theta) - \nabla_{\mathbf{W}^{(k)}} f(\mathbf{x} | \theta')\|_{F}$.}}\\

\textcolor{black}{
Next, we bound $\|\nabla_{\mathbf{W}^{(k)}} f(\mathbf{x} | \theta) - \nabla_{\mathbf{W}^{(k)}} f(\mathbf{x} | \theta')\|_{F}$. First, }\\
\textcolor{black}{
\begin{align*}
 &\|\nabla_{\mathbf{W}^{(k)}}f(\mathbf{x}|\theta)-\nabla_{\mathbf{W}^{(k)}}f(\mathbf{x}|\theta')\|_{F}
  =\Big\|b_{k}[\mathbf{H}^{(k-1)}]^{\top}(\frac{\partial
f(\mathbf{x}|\theta)}{\partial\mathbf{X}^{(k)}}\odot\mathbf{R}^{(k)})-b_{k}[\mathbf{H}^{(k-1)'}]^{\top}(\frac{\partial
f(\mathbf{x}|\theta')}{\partial\mathbf{X}^{(k)}}\odot\mathbf{R}^{(k)'})\Big\|_{F}\\
\leq &b_{k}\Big(\big\|\mathbf{H}^{(k-1)}-\mathbf{H}^{(k-1)'}\|_{F}\cdot\big\|\frac{\partial
f(\mathbf{x}|\theta)}{\partial\mathbf{X}^{(k)}}\odot\mathbf{R}^{(k)}\big\|_{F}+\|\mathbf{H}^{(k-1)'}\|_{F}\cdot\big\|\frac{\partial
f(\mathbf{x}|\theta)}{\partial\mathbf{X}^{(k)}}\odot\mathbf{R}^{(k)}-\frac{\partial
f(\mathbf{x}|\theta')}{\partial\mathbf{X}^{(k)}}\odot\mathbf{R}^{(k)'}\big\|_{F}\Big).
\end{align*}
Since $\big\|\frac{\partial
f(\mathbf{x}|\theta)}{\partial\mathbf{X}^{(k)}}\odot\mathbf{R}^{(k)}\big\|_{F}\leq\alpha_{\sigma}\|\frac{\partial
f(\mathbf{x}|\theta)}{\partial\mathbf{X}^{(k)}}\|_{F}$ and $\|\mathbf{H}^{(k-1)}-\mathbf{H}^{(k-1)'}\|_{F}\leq(1-a_{k})C_g\|\triangle\mathbf{X}^{(k-1)}\|_{F}$,
$$\big\|\mathbf{H}^{(k-1)}-\mathbf{H}^{(k-1)'}\|_{F}\cdot\big\|\frac{\partial
f(\mathbf{x}|\theta)}{\partial\mathbf{X}^{(k)}}\odot\mathbf{R}^{(k)}\big\|_{F}\leq(1-a_{k})C_g\|\triangle\mathbf{X}^{(k-1)}\|_{F}\cdot
\alpha_{\sigma}\|\frac{\partial
f(\mathbf{x}|\theta)}{\partial\mathbf{X}^{(k)}}\|_{F}.$$
Following \eqref{equ:define_gammak}, we denote
\begin{equation}\label{equ:GCNII_define_gammak}
  \gamma_{k}:=\big\|\frac{\partial
f(\mathbf{x}|\theta)}{\partial\mathbf{X}^{(k)}}\odot\mathbf{R}^{(k)}-\frac{\partial
f(\mathbf{x}|\theta')}{\partial\mathbf{X}^{(k)}}\odot\mathbf{R}^{(k)'}\big\|_{F}.
\end{equation}
Since $\|\mathbf{H}^{(k-1)'}\|_{F}\leq (1-a_{k})C_{g}B_{\mathbf{X}}^{(k-1)}+a_{k}C_{\mathbf{X}}$, we further apply \eqref{equ:GCNII_boundof_deltaXk} and \eqref{equ:GCNII_bound_df_dXk} to obtain
\begin{align}
 \nonumber &\|\nabla_{\mathbf{W}^{(k)}}f(\mathbf{x}|\theta)-\nabla_{\mathbf{W}^{(k)}}f(\mathbf{x}|\theta')\|_{F}\\
 \leq &b_{k}\bigg\{(1-a_{k})\alpha_{\sigma}C_ge^{(k-1)}\cdot\big(\sum_{j=1}^{k-1}\|\triangle\mathbf{W}^{(j)}\|_{2}\big)\cdot\|\frac{\partial
f(\mathbf{x}|\theta)}{\partial\mathbf{X}^{(k)}}\|_{F}
+\big((1-a_{k})C_{g}B_{\mathbf{X}}^{(k-1)}+a_{k}C_{\mathbf{X}}\big)\cdot\gamma_{k}\bigg\}. \label{equ:GCNII_bound_var_df_dWk0}
\end{align}
That is, to bound $\|\nabla_{\mathbf{W}^{(k)}}f(\mathbf{x}|\theta)-\nabla_{\mathbf{W}^{(k)}}f(\mathbf{x}|\theta')\|_{F}$, we need to bound
$\|\frac{\partial
f(\mathbf{x}|\theta)}{\partial\mathbf{X}^{(k)}}\|_{F}$ and $\gamma_{k}$. We provide the following steps to the bound of $\|\frac{\partial
f(\mathbf{x}|\theta)}{\partial\mathbf{X}^{(k)}}\|_{F}$ and $\gamma_{k}$. Using these two bounds, we finally obtain the upper bound of $\|\nabla_{\mathbf{W}^{(k)}}f(\mathbf{x}|\theta)-\nabla_{\mathbf{W}^{(k)}}f(\mathbf{x}|\theta')\|_{F}$ by applying them to \eqref{equ:GCNII_bound_var_df_dWk0}.
}\\

\textcolor{black}{
{\bf Step 1: we first bound $\big\|\frac{\partial
f(\mathbf{x}|\theta)}{\partial\mathbf{X}^{(k-1)}}\big\|_{F}$.}}

\textcolor{black}{
 According to the iterative formula of $\frac{\partial
f(\mathbf{x}|\theta)}{\partial\mathbf{X}^{(k-1)}}$, we have
  \begin{align*}
  \|\frac{\partial
f(\mathbf{x}|\theta)}{\partial\mathbf{X}^{(k-1)}}\|_{F}
  =&\|(1-a_{k})[g(\mathbf{L})]^{\top}(\frac{\partial
f(\mathbf{x}|\theta)}{\partial\mathbf{X}^{(k)}}\odot\mathbf{R}^{(k)})
[(1-b_{k})\mathbf{I}_{d}+b_{k}\mathbf{W}^{(k)}]^{\top}\|_{F}\\
 \leq&(1-a_{k})\|g(\mathbf{L})\|_{2}\cdot\|\frac{\partial
f(\mathbf{x}|\theta)}{\partial\mathbf{X}^{(k)}}\odot\mathbf{R}^{(k)}\|_{F}\cdot\|(1-b_{k})\mathbf{I}_{d}+b_{k}\mathbf{W}^{(k)}\|_{2}.
  \end{align*}
  Since the absolute value of the elements in $\mathbf{R}^{(k)}$ is less than $\alpha_{\sigma}$, $\|\frac{\partial
f(\mathbf{x}|\theta)}{\partial\mathbf{X}^{(k)}}\odot\mathbf{R}^{(k)}\|_{F}\leq\alpha_{\sigma}\|\frac{\partial
f(\mathbf{x}|\theta)}{\partial\mathbf{X}^{(k)}}\|_{F}$. Then, combining $\|g(\mathbf{L})\|_{2}=C_{g}$ and $\|(1-b_{k})\mathbf{I}_{d}+b_{k}\mathbf{W}^{(k)}\|_{2}\leq(1-b_{k})+b_{k}B$, we obtain the following iterative formula
$$\|\frac{\partial
f(\mathbf{x}|\theta)}{\partial\mathbf{X}^{(k-1)}}\|_{F}\leq(1-a_{k})(1-b_{k}+b_{k}B)\alpha_{\sigma}C_{g}\cdot\|\frac{\partial
f(\mathbf{x}|\theta)}{\partial\mathbf{X}^{(k)}}\|_{F}.$$
Note that since $|\nabla\sigma(\boldsymbol{\delta}_{\mathbf{x}}^{\top}\mathbf{H}^{(K)}\mathbf{w})|\leq\alpha_{\sigma}$ and $\|\mathbf{w}\|_{2}\leq B$, $$\|\frac{\partial
f(\mathbf{x}|\theta)}{\partial\mathbf{X}^{(K)}}\|_{F}
  =\|(1-a_{K+1})\nabla\sigma(\boldsymbol{\delta}_{\mathbf{x}}^{\top}\mathbf{H}^{(K)}\mathbf{w})
  [\boldsymbol{\delta}_{\mathbf{x}}^{\top}g(\mathbf{L})]^{\top}\mathbf{w}^{\top}\|_{F}\leq
  (1-a_{K+1})B\alpha_{\sigma}C_{g}.$$
  Thus,
$$
  \|\frac{\partial
f(\mathbf{x}|\theta)}{\partial\mathbf{X}^{(k)}}\|_{F}\leq\Big(\prod_{j=k+1}^{K}(1-a_{j})(1-b_{j}+b_{j}B)\Big)
(1-a_{K+1})B\alpha_{\sigma}^{K+1-k}C^{K+1-k}_{g}.
$$
For simplicity, we denote $B^{(k)}_{\partial\mathbf{X}}=\Big(\prod\limits_{j=k+1}^{K}(1-a_{j})(1-b_{j}+b_{j}B)\Big)
(1-a_{K+1})B\alpha_{\sigma}^{K+1-k}C^{K+1-k}_{g}$, and then
\begin{equation}\label{equ:GCNII_bound_df_dXk}
  \|\frac{\partial
f(\mathbf{x}|\theta)}{\partial\mathbf{X}^{(k)}}\|_{F}\leq B^{(k)}_{\partial\mathbf{X}},~k=1,2,\dots,K.
  \end{equation}
When $a_{k}=0,b_{k}=1$ for all $k$, we have $B^{(k)}_{\partial\mathbf{X}}=B^{K+1-k}\alpha_{\sigma}^{K+1-k}C_{g}^{K+1-k}$, which is the same as shown in \eqref{equ:bound_df_dX}.}\\

\textcolor{black}{
{\bf Step 2: We next bound $\gamma_{k}$.}}

\textcolor{black}{
Following the proof of Lemma \ref{lem:bound_variation_df_dW}, we have by \eqref{equ:GCNII_define_gammak} that
\begin{align} \label{equ:GCNII_bound_gammak0}
 \gamma_{k}\leq\big\|\frac{\partial
f(\mathbf{x}|\theta)}{\partial\mathbf{X}^{(k)}}\odot(\mathbf{R}^{(k)}-\mathbf{R}^{(k)'})\big\|_{F}+\alpha_{\sigma}\cdot\big\|\big(\frac{\partial
f(\mathbf{x}|\theta)}{\partial\mathbf{X}^{(k)}}-\frac{\partial
f(\mathbf{x}|\theta')}{\partial\mathbf{X}^{(k)}}\big)\big\|_{F}.
\end{align}
Similarly, let $h_{k}:=\big\|\frac{\partial
f(\mathbf{x}|\theta)}{\partial\mathbf{X}^{(k)}}\odot(\mathbf{R}^{(k)}-\mathbf{R}^{(k)'})\big\|_{F}$. Then,
applying $\|\mathbf{A}_{1}\odot\mathbf{A}_{2}\|_{F}\leq\|\mathbf{A}_{1}\|_{F}\|\mathbf{A}_{2}\|_{F}$, we have
\begin{equation}\label{equ:GCNII_bound_of_hk0}
h_{k}=\big\|\frac{\partial
f(\mathbf{x}|\theta)}{\partial\mathbf{X}^{(k)}}\odot(\mathbf{R}^{(k)}-\mathbf{R}^{(k)'})\big\|_{F}\leq\big\|\frac{\partial
f(\mathbf{x}|\theta)}{\partial\mathbf{X}^{(k)}}\big\|_{F}\cdot\big\|\mathbf{R}^{(k)}-\mathbf{R}^{(k)'}\big\|_{F}.
\end{equation}
Note that $\mathbf{R}^{(k)}=\nabla\sigma\Big(\mathbf{H}^{(k-1)}\big((1-b_{k})\mathbf{I}_{d}+b_{k}\mathbf{W}^{(k)}\big)\Big)$, then the $\nu_{\sigma}$-smooth property of $\sigma(\cdot)$ implies
\begin{align*}
\big\|\mathbf{R}^{(k)}-\mathbf{R}^{(k)'}\big\|_{F}
=&\Big\|\nabla\sigma\Big(\mathbf{H}^{(k-1)}\big((1-b_{k})\mathbf{I}_{d}+b_{k}\mathbf{W}^{(k)}\big)\Big)
-\nabla\sigma\Big(\mathbf{H}^{(k-1)'}\big((1-b_{k})\mathbf{I}_{d}+b_{k}\mathbf{W}^{(k)'}\big)\Big)\Big\|_{F}\\
\leq&\nu_{\sigma}\cdot\big\|\mathbf{H}^{(k-1)}\big((1-b_{k})\mathbf{I}_{d}+b_{k}\mathbf{W}^{(k)}\big)
-\mathbf{H}^{(k-1)'}\big((1-b_{k})\mathbf{I}_{d}+b_{k}\mathbf{W}^{(k)'}\big)\big\|_{F}.
\end{align*}
Applying $\|\mathbf{A}_{1}\mathbf{A}_{2}-\mathbf{A}'_{1}\mathbf{A}'_{2}\|_{F}
\leq\|\mathbf{A}_{1}-\mathbf{A}_{1}'\|_{F}\cdot\|\mathbf{A}_{2}\|_{2}+\|\mathbf{A}_{1}'\|_{F}\cdot\|\mathbf{A}_{2}-\mathbf{A}'_{2}\|_{2}$, we further have
\begin{align*}
\big\|\mathbf{R}^{(k)}-\mathbf{R}^{(k)'}\big\|_{F}
\leq&\nu_{\sigma}\cdot\Big(\big\|(\mathbf{H}^{(k-1)}-\mathbf{H}^{(k-1)'}\big\|_{F}\cdot\big\|(1-b_{k})\mathbf{I}_{d}+b_{k}\mathbf{W}^{(k)}\big\|_{2}
+\big\|\mathbf{H}^{(k-1)'}\big\|_{F}\cdot\big\|b_{k}(\mathbf{W}^{(k)}-\mathbf{W}^{(k)'})\big\|_{2}\Big).
\end{align*}
Note that  $\|\mathbf{H}^{(k-1)}-\mathbf{H}^{(k-1)'}\|_{F}\leq(1-a_{k})C_g\|\triangle\mathbf{X}^{(k-1)}\|_{F}
\leq(1-a_{k})C_g\cdot e^{(k-1)}\big(\sum\limits_{j=1}^{k-1}\|\triangle\mathbf{W}^{(j)}\|_{2}\big)$ (see \eqref{equ:GCNII_boundof_deltaXk}), $\|\mathbf{H}^{(k-1)'}\|_{F}\leq (1-a_{k})C_{g}B_{\mathbf{X}}^{(k-1)}+a_{k}C_{\mathbf{X}}$, and $\big\|(1-b_{k})\mathbf{I}_{d}+b_{k}\mathbf{W}^{(k)}\big\|_{2}\leq(1-b_{k})+b_{k}B$. Thus,
\begin{align}
\nonumber\big\|\mathbf{R}^{(k)}-\mathbf{R}^{(k)'}\big\|_{F}
&\leq\nu_{\sigma}\cdot\Big((1-a_{k})(1-b_{k}+b_{k}B)C_g\cdot e^{(k-1)}\big(\sum\limits_{j=1}^{k-1}\|\triangle\mathbf{W}^{(j)}\|_{2}\big)
+\big((1-a_{k})C_{g}B_{\mathbf{X}}^{(k-1)}+a_{k}C_{\mathbf{X}}\big)b_{k}\|\triangle\mathbf{W}^{(k)}\|_{2}\Big)\\
&\leq\nu_{\sigma}r_{k}\cdot\big(\sum\limits_{j=1}^{k}\|\triangle\mathbf{W}^{(j)}\|_{2}\big), \label{equ:GCNII_bound_var_Rk}
\end{align}
where $r_{k}:=\max\Big\{(1-a_{k})(1-b_{k}+b_{k}B)C_g\cdot e^{(k-1)},\big((1-a_{k})C_{g}B_{\mathbf{X}}^{(k-1)}+a_{k}C_{\mathbf{X}}\big)b_{k}\Big\}$. When $a_{k}=0,b_{k}=1$ for all $k$, $e^{(k-1)}=B^{k-2}\alpha_{\sigma}^{k-1}C_{g}^{k-1}C_{\mathbf{X}}$ and $B_{\mathbf{X}}^{(k-1)}=B^{k-1}\alpha_{\sigma}^{k-1}C_{g}^{k-1}C_{\mathbf{X}}$, then at this point, $r_{k}=B^{k-1}\alpha_{\sigma}^{k-1}C_{g}^{k}C_{\mathbf{X}}$, and thus $\big\|\mathbf{R}^{(k)}-\mathbf{R}^{(k)'}\big\|_{F}
\leq\nu_{\sigma}B^{k-1}\alpha_{\sigma}^{k-1}C_{g}^{k}C_{\mathbf{X}}\big(\sum\limits_{j=1}^{k}\|\triangle\mathbf{W}^{(j)}\|_{2}\big)$, which is the same as shown in \eqref{equ:bound_var_Rk}.}

\textcolor{black}{Combining \eqref{equ:GCNII_bound_df_dXk}, \eqref{equ:GCNII_bound_of_hk0}, and \eqref{equ:GCNII_bound_var_Rk}, we have
\begin{equation}\label{equ:GCNII_bound_of_hk}
h_{k}\leq B_{\partial\mathbf{X}}^{(k)}\cdot\nu_{\sigma}r_{k}\cdot\big(\sum\limits_{j=1}^{k}\|\triangle\mathbf{W}^{(j)}\|_{2}\big).
\end{equation}
Next, we use the same technique as in \eqref{equ:bound_hk_by_hmax} that uses an $h_{\max}$ to bound all $h_k$. Specifically, let
$$h_{\max}=\max_{k=1,2,...,K}\Big\{B_{\partial\mathbf{X}}^{(k)}\cdot\nu_{\sigma}r_{k}\Big\}\cdot\|\triangle\theta\|_{*}.$$
Then,
\begin{equation}\label{equ:GCNII_bound_hk_by_hmax}
h_{k}\leq h_{\max}~\textrm{holds for all}~k=1,2,\dots,K.
\end{equation}
One can prove that when $a_{k}=0,~b_{k}=1$ for all $k$, then $h_{\max}=\nu_{\sigma}B^{K}\alpha^{K}_{\sigma}C^{K+1}_{g}C_{\mathbf{X}}\|\triangle\theta\|_{*}$, which is the same as in the case of traditional GCN.}

\textcolor{black}{
Applying \eqref{equ:GCNII_bound_hk_by_hmax} to \eqref{equ:GCNII_bound_gammak0}, we have
\begin{equation}\label{equ:GCNII_bound_gammak1}
\gamma_{k}\leq h_{\max}+\alpha_{\sigma}\cdot\big\|\big(\frac{\partial
f(\mathbf{x}|\theta)}{\partial\mathbf{X}^{(k)}}-\frac{\partial
f(\mathbf{x}|\theta')}{\partial\mathbf{X}^{(k)}}\big)\big\|_{F},
\end{equation}
and we can derive the iterative formula for the bound of $\gamma_{k}$. To do this,  we utilize the iterative formula of $\frac{\partial
f(\mathbf{x}|\theta)}{\partial\mathbf{X}^{(k-1)}}$ and obtain
\begin{align*}
\big\|\frac{\partial
f(\mathbf{x}|\theta)}{\partial\mathbf{X}^{(k)}}-\frac{\partial
f(\mathbf{x}|\theta')}{\partial\mathbf{X}^{(k)}}\big\|_{F}
\leq &\big\|(1-a_{k+1})[g(\mathbf{L})]^{\top}(\frac{\partial
f(\mathbf{x}|\theta)}{\partial\mathbf{X}^{(k+1)}}\odot\mathbf{R}^{(k+1)})
[(1-b_{k+1})\mathbf{I}_{d}+b_{k+1}\mathbf{W}^{(k+1)}]^{\top}\\
&-(1-a_{k+1})[g(\mathbf{L})]^{\top}(\frac{\partial
f(\mathbf{x}|\theta')}{\partial\mathbf{X}^{(k+1)}}\odot\mathbf{R}^{(k+1)'})
[(1-b_{k+1})\mathbf{I}_{d}+b_{k+1}\mathbf{W}^{(k+1)'}]^{\top}\big\|_{F}\\
\leq &(1-a_{k+1})C_g\cdot\bigg(\Big\|(\frac{\partial
f(\mathbf{x}|\theta)}{\partial\mathbf{X}^{(k+1)}}\odot\mathbf{R}^{(k+1)})
[(1-b_{k+1})\mathbf{I}_{d}+b_{k+1}\mathbf{W}^{(k+1)}]^{\top}\\
&-(\frac{\partial
f(\mathbf{x}|\theta')}{\partial\mathbf{X}^{(k+1)}}\odot\mathbf{R}^{(k+1)'})
[(1-b_{k+1})\mathbf{I}_{d}+b_{k+1}\mathbf{W}^{(k+1)'}]^{\top}\Big\|_{F}\bigg).
\end{align*}
Applying $\|\mathbf{A}_{1}\mathbf{A}_{2}-\mathbf{A}'_{1}\mathbf{A}'_{2}\|_{F}\leq
\|\mathbf{A}_{1}-\mathbf{A}_{1}'\|_{F}\|\mathbf{A}_{2}\|_{2}+\|\mathbf{A}_{1}'\|_{F}\|\mathbf{A}_{2}-\mathbf{A}'_{2}\|_{2}$ and $\|(1-b_{k+1})\mathbf{I}_{d}+b_{k+1}\mathbf{W}^{(k+1)}\|_{2}\leq 1-b_{k+1}+b_{k+1}B$,
we have
\begin{align*}
\big\|\frac{\partial
f(\mathbf{x}|\theta)}{\partial\mathbf{X}^{(k)}}-\frac{\partial
f(\mathbf{x}|\theta')}{\partial\mathbf{X}^{(k)}}\big\|_{F}
\leq(1-a_{k+1})C_g\cdot\bigg(&(1-b_{k+1}+b_{k+1}B)\Big\|\frac{\partial
f(\mathbf{x}|\theta)}{\partial\mathbf{X}^{(k+1)}}\odot\mathbf{R}^{(k+1)}-\frac{\partial
f(\mathbf{x}|\theta')}{\partial\mathbf{X}^{(k+1)}}\odot\mathbf{R}^{(k+1)'}\Big\|_{F}\\
&+\Big\|\frac{\partial
f(\mathbf{x}|\theta')}{\partial\mathbf{X}^{(k+1)}}\odot\mathbf{R}^{(k+1)'}\Big\|_{F}\cdot
\big\|b_{k+1}(\mathbf{W}^{(k+1)}-\mathbf{W}^{(k+1)'})\big\|_{2}\bigg).
\end{align*}
Since $\big\|\frac{\partial
f(\mathbf{x}|\theta')}{\partial\mathbf{X}^{(k+1)}}\odot\mathbf{R}^{(k+1)'}\big\|_{F}
\leq\alpha_{\sigma}\big\|\frac{\partial
f(\mathbf{x}|\theta')}{\partial\mathbf{X}^{(k+1)}}\big\|_{F}\leq\alpha_{\sigma} B_{\partial\mathbf{X}}^{(k+1)}$ and $\gamma_{k+1}=\Big\|\frac{\partial
f(\mathbf{x}|\theta)}{\partial\mathbf{X}^{(k+1)}}\odot\mathbf{R}^{(k+1)}-\frac{\partial
f(\mathbf{x}|\theta')}{\partial\mathbf{X}^{(k+1)}}\odot\mathbf{R}^{(k+1)'}\Big\|_{F}$, thus
\begin{equation}\label{equ:GCNII_bound_var_df_dXk}
\big\|\frac{\partial
f(\mathbf{x}|\theta)}{\partial\mathbf{X}^{(k)}}-\frac{\partial
f(\mathbf{x}|\theta')}{\partial\mathbf{X}^{(k)}}\big\|_{F}
\leq(1-a_{k+1})C_g\cdot\bigg((1-b_{k+1}+b_{k+1}B)\gamma_{k+1}
+b_{k+1}\alpha_{\sigma}B_{\partial\mathbf{X}}^{(k+1)}\cdot
\big\|\triangle\mathbf{W}^{(k+1)}\big\|_{2}\bigg).
\end{equation}
Combining \eqref{equ:GCNII_bound_gammak1} and \eqref{equ:GCNII_bound_var_df_dXk}, we obtain the iterative formula for the bound of $\gamma_{k}$ as
\begin{align}
 \gamma_{k}\leq& h_{\max}+(1-a_{k+1})\alpha_{\sigma}C_{g}\cdot\bigg((1-b_{k+1}+b_{k+1}B)\cdot\gamma_{k+1}
+b_{k+1}\alpha_{\sigma}\cdot B_{\partial\mathbf{X}}^{(k+1)}\cdot
\big\|\triangle\mathbf{W}^{(k+1)}\big\|_{2}\bigg).   \label{equ:GCNII_bound_of_gammak}
\end{align}
}

\textcolor{black}{
Furthermore, since $\frac{\partial
f(\mathbf{x}|\theta)}{\partial\mathbf{X}^{(K)}}
  =(1-a_{K+1})\nabla\sigma(\boldsymbol{\delta}_{\mathbf{x}}^{\top}\mathbf{H}^{(K)}\mathbf{w})
  [\boldsymbol{\delta}_{\mathbf{x}}^{\top}g(\mathbf{L})]^{\top}\mathbf{w}^{\top}$,
\begin{align*}
\big\|\frac{\partial
f(\mathbf{x}|\theta)}{\partial\mathbf{X}^{(K)}}-\frac{\partial
f(\mathbf{x}|\theta')}{\partial\mathbf{X}^{(K)}}\big\|_{F}
= &\big\|(1-a_{K+1})\nabla\sigma(\boldsymbol{\delta}_{\mathbf{x}}^{\top}\mathbf{H}^{(K)}\mathbf{w})
  [\boldsymbol{\delta}_{\mathbf{x}}^{\top}g(\mathbf{L})]^{\top}\mathbf{w}^{\top}
-(1-a_{K+1})\nabla\sigma(\boldsymbol{\delta}_{\mathbf{x}}^{\top}\mathbf{H}^{(K)'}\mathbf{w}')
  [\boldsymbol{\delta}_{\mathbf{x}}^{\top}g(\mathbf{L})]^{\top}\mathbf{w}'^{\top}\big\|_{F}\\
  =&\big(1-a_{K+1})\|\nabla\sigma(\boldsymbol{\delta}_{\mathbf{x}}^{\top}\mathbf{H}^{(K)}\mathbf{w})
  [\boldsymbol{\delta}_{\mathbf{x}}^{\top}g(\mathbf{L})]^{\top}\mathbf{w}^{\top}
-\nabla\sigma(\boldsymbol{\delta}_{\mathbf{x}}^{\top}\mathbf{H}^{(K)'}\mathbf{w}')
  [\boldsymbol{\delta}_{\mathbf{x}}^{\top}g(\mathbf{L})]^{\top}\mathbf{w}'^{\top}\big\|_{F}.
\end{align*}
The inequation $\|\mathbf{A}_{1}\mathbf{A}_{2}-\mathbf{A}'_{1}\mathbf{A}'_{2}\|_{F}\leq
\|\mathbf{A}_{1}-\mathbf{A}_{1}'\|_{F}\|\mathbf{A}_{2}\|_{2}+\|\mathbf{A}_{1}'\|_{F}\|\mathbf{A}_{2}-\mathbf{A}'_{2}\|_{2}$ further derives
\begin{align*}
\big\|\frac{\partial
f(\mathbf{x}|\theta)}{\partial\mathbf{X}^{(K)}}-\frac{\partial
f(\mathbf{x}|\theta')}{\partial\mathbf{X}^{(K)}}\big\|_{F}
\leq&\big(1-a_{K+1})\Big(\big|\nabla\sigma(\boldsymbol{\delta}_{\mathbf{x}}^{\top}\mathbf{H}^{(K)}\mathbf{w})
                         -\nabla\sigma(\boldsymbol{\delta}_{\mathbf{x}}^{\top}\mathbf{H}^{(K)'}\mathbf{w}')\big|\cdot
\big\|[\boldsymbol{\delta}_{\mathbf{x}}^{\top}g(\mathbf{L})]^{\top}\mathbf{w}^{\top}\big\|_{2}\\
&+\big|\nabla\sigma(\boldsymbol{\delta}_{\mathbf{x}}^{\top}\mathbf{H}^{(K)'}\mathbf{w}')\big|\cdot
  \big\|[\boldsymbol{\delta}_{\mathbf{x}}^{\top}g(\mathbf{L})]^{\top}(\mathbf{w}-\mathbf{w}')^{\top}\big\|_{2}\Big)\\
\leq &(1-a_{K+1})\Big(\nu_{\sigma}\cdot\big|\boldsymbol{\delta}_{\mathbf{x}}^{\top}\mathbf{H}^{(K)}\mathbf{w}
                         -\boldsymbol{\delta}_{\mathbf{x}}^{\top}\mathbf{H}^{(K)'}\mathbf{w}'\big|\cdot BC_{g}+\alpha_{\sigma}C_{g}\|\triangle\mathbf{w}\|_{2}\Big),
\end{align*}
where the last inequation holds true because of the $\nu_{\sigma}$-smooth property of $\sigma(\cdot)$ and $|\nabla\sigma|\leq\alpha_{\sigma}$. Then, we apply \eqref{equ:GCNII_bound_var_HK_w} to obtain
\begin{equation}\label{equ:GCNII_bound_of_var_df_dXK}
\big\|\frac{\partial
f(\mathbf{x}|\theta)}{\partial\mathbf{X}^{(K)}}-\frac{\partial
f(\mathbf{x}|\theta')}{\partial\mathbf{X}^{(K)}}\big\|_{F}
\leq(1-a_{K+1})\Big(\nu_{\sigma}BC_{g}\cdot\varrho\|\triangle\theta\|_{*}+\alpha_{\sigma}C_{g}\|\triangle\mathbf{w}\|_{2}\Big),
\end{equation}
where $\varrho=\max\Big\{(1-a_{K+1})BC_g\cdot e^{(K)},~~(1-a_{K+1})C_{g}B_{\mathbf{X}}^{(K)}+a_{K+1}C_{\mathbf{X}}\Big\}$.  Substituting \eqref{equ:GCNII_bound_of_var_df_dXK} into \eqref{equ:GCNII_bound_gammak1},
\begin{align}\label{equ:GCNII_bound_of_gammaK}
 \gamma_{K}\leq
h_{\max}
+(1-a_{K+1})\alpha_{\sigma}\cdot\Big(\nu_{\sigma}BC_{g}\cdot\varrho\|\triangle\theta\|_{*}+\alpha_{\sigma}C_{g}\|\triangle\mathbf{w}\|_{2}\Big).
\end{align}
Combining \eqref{equ:GCNII_bound_of_gammak} and \eqref{equ:GCNII_bound_of_gammaK}, we can further obtain the bound of $\gamma_{k}$.
}

\section*{\textcolor{black}{APPENDIX F:Proof for Section \ref{sec:62}}} \label{proof:section6_2}
\textcolor{black}{
{\bf Proof of Eq. (\ref{softmax-2}):}} 
\textcolor{black}{
For vectors $ \mathbf{z} = (z_1, z_2, \dots, z_p)$ and $ \mathbf{z}' = (z'_1, z'_2, \dots, z'_p) $
(with $\|\mathbf{z}-\mathbf{z}'\|_{\infty}<1$), the softmax function is defined as:
$$
\operatorname{softmax}(\mathbf{z})=(Z_1,Z_2,\dots,Z_p),\quad \operatorname{softmax}(\mathbf{z}')=(Z'_1,Z'_2,\dots,Z'_p),$$
where
$$
  Z_k = \frac{e^{z_k}}{\sum_{j=1}^p e^{z_j}}, \quad Z'_k = \frac{e^{z'_k}}{\sum_{j=1}^p e^{z'_j}} \quad \forall k=1,2,\dots,p.
 $$
For each $ k $, rewrite $ |Z_k - Z'_k| $ using the softmax definition:
$$
|Z_k - Z'_k| = \left| \frac{e^{z_k}}{S} - \frac{e^{z'_k}}{S'} \right| = \left| \frac{e^{z_k} S' - e^{z'_k} S}{S S'} \right|,
$$
where $ S = \sum_{j=1}^p e^{z_j} $ and $ S' = \sum_{j=1}^p e^{z'_j} $. By the triangle inequality in the numerator:
$$
|e^{z_k} S' - e^{z'_k} S| \leq e^{z_k} |S' - S| + S |e^{z_k} - e^{z'_k}|.
$$
Summing $ |Z_k - Z'_k| $ over $ k $ gives the 1-norm:
$$
\|\operatorname{softmax}(\mathbf{z}) - \operatorname{softmax}(\mathbf{z}')\|_1 = \sum_{k=1}^p |Z_k - Z'_k| \leq \sum_{k=1}^p \left( \frac{e^{z_k}}{S S'} |S' - S| + \frac{1}{S'} |e^{z_k} - e^{z'_k}| \right).
$$
Since $ \sum\limits_{k=1}^p \frac{e^{z_k}}{S} = 1 $ , this simplifies to:
\begin{equation}\label{softmax-1}
\|\operatorname{softmax}(\mathbf{z}) - \operatorname{softmax}(\mathbf{z}')\|_1  \leq \frac{|S' - S|}{S'} + \frac{1}{S'} \sum_{k=1}^p |e^{z_k} - e^{z'_k}|.
\end{equation}
Notice that for any $ k $, by the mean value theorem,
$$
|e^{z_k} - e^{z'_k}| \leq e^{z'_k}|e^{z_k-z'_k}-1| \leq e^{z'_k}\cdot e\cdot |z_k - z'_k|,
$$
where the last inequation holds true because $\|\mathbf{z}-\mathbf{z}'\|_{\infty}<1$,
and, by the triangle inequality, $$ |S' - S| = \left| \sum_{j=1}^p (e^{z'_j} - e^{z_j}) \right| \leq \sum_{j=1}^p |e^{z'_j} - e^{z_j}|. $$ Substituting into \eqref{softmax-1} gives:
$$
\|\operatorname{softmax}(\mathbf{z}) - \operatorname{softmax}(\mathbf{z}')\|_1 \leq \frac{2}{S'} \sum_{j=1}^p |e^{z'_j} - e^{z_j}|\leq 2e\cdot\max|z_k-z'_k|=2e\cdot\|\mathbf{z}-\mathbf{z}'\|_{\infty},
$$
where the second inequation holds true because $\sum\limits_{k=1}^p \frac{e^{z'_k}}{S'} = 1 $. Thus, we complete the proof of Eq. (\ref{softmax-2}).
}\\

\textcolor{black}{
 {\bf Proof of Eq. (\ref{attention-1}):}}\\

  \textcolor{black}{
  Recall that we denote  $ B $ a constant which bounds all original and perturbed parameters, i.e, $$ \|\mathbf{W}_K\|_{2}, \|\mathbf{W}'_K\|_{2}, \|\mathbf{W}_Q\|_{2}, \|\mathbf{W}'_Q\|_{2}, \|\mathbf{W}_V\|_{2}, \|\mathbf{W}'_V\|_{2}, \|\mathbf{W}_O\|_{2}, \|\mathbf{W}'_O\|_{2} \leq B, $$ and $ \|\mathbf{a}\|_{2}, \|\mathbf{a}'\|_{2} \leq B $ (output vector norm). And for $ \theta = \{\mathbf{W}_K, \mathbf{W}_Q, \mathbf{W}_V, \mathbf{W}_O, \mathbf{a}\} $,
\begin{equation}\label{equ:GraTrans_Grad}
 \nabla_{\mathbf{W}_K} F(\mathbf{x}_n|\theta)= \mathbf{a}^T \mathbb{I}_{\geq 0}(\mathbf{Z}_\theta) \cdot \mathbf{W}_O \cdot \mathbf{W}_V \cdot \mathbf{M}_\theta \cdot (\mathbf{W}_Q \mathbf{x}_n)^{\top},
\end{equation}
where: $ \mathbf{Z}_\theta = \mathbf{W}_O \cdot \text{Attn}(\mathbf{x}_n; \theta) $, $ \text{Attn}(\mathbf{x}_n; \theta) = \sum_{s \in \mathcal{T}^n} A_{s,n}(\theta) (\mathbf{W}_V \mathbf{x}_s) $,
$ A_{s,n}(\theta) = \operatorname{softmax}(S_{\cdot,n}(\theta))_s $ with $ S_{s,n}(\theta) = (\mathbf{W}_K \mathbf{x}_s)^{\top} (\mathbf{W}_Q \mathbf{x}_n) $,
$ \mathbf{M}_\theta = \sum_{s \in \mathcal{T}^n} A_{s,n}(\theta) (\mathbf{x}_s - \bar{\mathbf{x}}_n) \mathbf{x}_s^{\top} $ (aggregate neighbor term).\newline
Then the gradient perturbation $\nabla_{\mathbf{W}_{K}}F(\mathbf{x}_{n}|\theta)-\nabla_{\mathbf{W}_{K}}F(\mathbf{x}_{n}|\theta')$ arises from differences in $ \theta $ and $ \theta' $. According to \eqref{equ:GraTrans_Grad}, we apply the triangle inequality  and Lipschitz continuity of matrix multiplication/activation functions and then decompose the perturbation into contributions from each parameter:
\begin{equation}\label{equ:GraTrans_per}
\|\nabla_{\mathbf{W}_{K}}F(\mathbf{x}_{n}|\theta)-\nabla_{\mathbf{W}_{K}}F(\mathbf{x}_{n}|\theta')\|_{2}\leq \sum_{\phi \in \{\theta\}} \|\nabla_{\mathbf{W}_{K}}F(\mathbf{x}_{n}|\theta) - \nabla_{\mathbf{W}_{K}}F(\mathbf{x}_{n}|\theta_{\phi \to \phi'})\|_{2}
\end{equation}
where $ \theta_{\phi \to \phi'} $ replaces parameter $ \phi $ with $ \phi' $ while keeping others fixed.}

\begin{itemize}
    \item  \textcolor{black}{Contribution from $ \mathbf{a} - \mathbf{a}' $:
The term $ \mathbf{a}^{\top} $ in \eqref{equ:GraTrans_per} introduces a perturbation bounded by: 
\[
\|\nabla_{\mathbf{W}_{K}}F(\mathbf{x}_{n}|\theta) - \nabla_{\mathbf{W}_{K}}F(\mathbf{x}_{n}|\theta_{\mathbf{a} \to \mathbf{a}'})\|_{2} \leq \|\mathbf{a} - \mathbf{a}'\|_{2} \cdot \|\mathbb{I}_{\geq 0}(\mathbf{Z}_{\theta})\|_{2} \cdot \|\mathbf{W}_O\|_{2} \cdot \|\mathbf{W}_V\|_{2} \cdot \|\mathbf{M}_{\theta}\|_{2} \cdot \|\mathbf{W}_Q \mathbf{x}_n\|_{2},
\]
Using
\begin{equation}\label{equ:GraTrans_bound_assump}
\|\mathbb{I}_{\geq 0}(\mathbf{Z}_{\theta})\|_{2} \leq 1, \quad \|\mathbf{M}_{\theta}\|_{2} \leq K_{\max} C_\mathbf{x}^2, \quad \|\mathbf{W}_Q \mathbf{x}_n\|_{2} \leq B C_\mathbf{x},
\end{equation}
where $K_{\max}$ is the maximum neighborhood size, we have
\[
\|\nabla_{\mathbf{W}_{K}}F(\mathbf{x}_{n}|\theta) - \nabla_{\mathbf{W}_{K}}F(\mathbf{x}_{n}|\theta_{\mathbf{a} \to \mathbf{a}'})\|_{2} \leq \|\mathbf{a} - \mathbf{a}'\|_{2} \cdot 1 \cdot B \cdot B \cdot K_{\max} C_\mathbf{x}^2 \cdot B C_\mathbf{x} = \|\mathbf{a} - \mathbf{a}'\|_{2} \cdot K_{\max} B^3 C_\mathbf{x}^3.
\]
}
\item \textcolor{black}{Contribution from $ \mathbf{W}_O - \mathbf{W}'_O $:
The term $\mathbf{ W}_O $ affects both $ \mathbf{Z}_\theta $ and the gradient product. By Lipschitz continuity of matrix multiplication:
\[
\|\nabla_{\mathbf{W}_{K}}F(\mathbf{x}_{n}|\theta) - \nabla_{\mathbf{W}_{K}}F(\mathbf{x}_{n}|\theta_{\mathbf{W}_O \to \mathbf{W}'_O})\|_{2}
\leq \|\mathbf{W}_O - \mathbf{W}'_O\|_{2} \cdot \|\mathbf{a}\|_{2} \cdot \|\mathbb{I}_{\geq 0}(\mathbf{Z}_{\theta})\|_{2} \cdot \|\mathbf{W}_V\|_{2} \cdot \|\mathbf{M}_{\theta}\|_{2} \cdot \|\mathbf{W}_Q \mathbf{x}_n\|_{2}.
\]
Substituting \eqref{equ:GraTrans_bound_assump}:
\[
\|\nabla_{\mathbf{W}_{K}}F(\mathbf{x}_{n}|\theta) - \nabla_{\mathbf{W}_{K}}F(\mathbf{x}_{n}|\theta_{\mathbf{W}_O \to \mathbf{W}'_O})\|_{2} \leq \|\mathbf{W}_O - \mathbf{W}'_O\| \cdot B \cdot 1 \cdot B \cdot K_{\max} C_\mathbf{x}^2 \cdot B C_\mathbf{x} = \|\mathbf{W}_O - \mathbf{W}'_O\|_{2} \cdot K_{\max} B^3 C_\mathbf{x}^3.
\]
}

\item \textcolor{black}{Contribution from $ \mathbf{W}_V - \mathbf{W}'_V $:
$ \mathbf{W}_V $ is independent of $ \text{Attn}(\mathbf{x}_n) $ and $ \mathbf{M}_\theta$. The perturbation bound is:
\[
\|\nabla_{\mathbf{W}_{K}}F(\mathbf{x}_{n}|\theta) - \nabla_{\mathbf{W}_{K}}F(\mathbf{x}_{n}|\theta_{\mathbf{W}_V \to \mathbf{W}'_V})\|_{2} \leq \|\mathbf{W}_V - \mathbf{W}'_V\|_{2} \cdot \|\mathbf{a}\|_{2} \cdot \|\mathbb{I}_{\geq 0}(\mathbf{Z}_{\theta})\|_{2} \cdot \|\mathbf{W}_O\|_{2} \cdot \|\mathbf{M}_{\theta}\|_{2} \cdot \|\mathbf{W}_Q \mathbf{x}_n\|_{2}.
\]
By symmetry with $ \mathbf{W}_O $:
\[
\|\nabla_{\mathbf{W}_{K}}F(\mathbf{x}_{n}|\theta) - \nabla_{\mathbf{W}_{K}}F(\mathbf{x}_{n}|\theta_{\mathbf{W}_V \to \mathbf{W}'_V})\|_{2} \leq \|\mathbf{W}_V - \mathbf{W}'_V\|_{2} \cdot K_{\max} B^3 C_\mathbf{x}^3.
\]
}

\item \textcolor{black}{Contribution from $ \mathbf{W}_Q - \mathbf{W}'_Q $:
$ \mathbf{W}_Q $ affects attention scores $ S_{s,n} $ and thus $ A_{s,n} $ and $ \mathbf{M}_\theta$. Using Lipschitzness of softmax and matrix multiplication:
\[
\|\nabla_{\mathbf{W}_{K}}F(\mathbf{x}_{n}|\theta) - \nabla_{\mathbf{W}_{K}}F(\mathbf{x}_{n}|\theta_{\mathbf{W}_Q \to \mathbf{W}'_Q})\|_{2} \leq 2e \|\mathbf{W}_Q - \mathbf{W}'_Q\|_{2} \cdot \|\mathbf{a}\|_{2} \cdot \|\mathbb{I}_{\geq 0}(\mathbf{Z}_{\theta})\|_{2} \cdot \|\mathbf{W}_O\|_{2} \cdot \|\mathbf{W}_V\|_{2} \cdot \|\mathbf{M}_{\theta}\|_{2} \cdot C_\mathbf{x}.
\]
Substituting bounds:
\[
\|\nabla_{\mathbf{W}_{K}}F(\mathbf{x}_{n}|\theta) - \nabla_{\mathbf{W}_{K}}F(\mathbf{x}_{n}|\theta_{\mathbf{W}_Q \to \mathbf{W}'_Q})\|_{2} \leq 2e\|\mathbf{W}_Q - \mathbf{W}'_Q\|_{2} \cdot K_{\max} B^3 C_\mathbf{x}^3.
\]
}

\item \textcolor{black}{Contribution from $ \mathbf{W}_K - \mathbf{W}'_K $:
$ \mathbf{W}_K $ directly impacts $ S_{s,n} $, $ A_{s,n} $, and $ \mathbf{M}_\theta$. By analogous reasoning:
\[
\|\nabla_{\mathbf{W}_{K}}F(\mathbf{x}_{n}|\theta) - \nabla_{\mathbf{W}_{K}}F(\mathbf{x}_{n}|\theta_{\mathbf{W}_K \to \mathbf{W}'_K})\|_{2} \leq 2e
\|\mathbf{W}_K - \mathbf{W}'_K\|_{2} \cdot K_{\max} B^3 C_\mathbf{x}^3.
\]
}

\textcolor{black}{
     According to \eqref{equ:GraTrans_per}, total gradient perturbation bound
summing all contributions, and we get:
\[
\|\nabla_{\mathbf{W}_{K}}F(\mathbf{x}_{n}|\theta)-\nabla_{\mathbf{W}_{K}}F(\mathbf{x}_{n}|\theta')\|_{2} \leq 2e K_{\max} B^3 C_\mathbf{x}^3  \|\triangle\theta\|_{*},
\]
where $ \|\triangle\theta\|_{*} = \|\mathbf{W}_K - \mathbf{W}'_K\|_{2} + \|\mathbf{W}_V - \mathbf{W}'_V\|_{2} + \|\mathbf{W}_O - \mathbf{W}'_O\|_{2} + \|\mathbf{W}_Q - \mathbf{W}'_Q\|_{2} + \|\mathbf{a} - \mathbf{a}'\|_{2}$.
}
\end{itemize}

\bibliographystyle{elsarticle-num}
\bibliography{references}

\end{document}